\documentclass[twoside,11pt]{article}

\usepackage{jmlr2e}

\usepackage{todonotes} % \missingfigure[figwidth=50mm]{}
\usepackage{microtype}
\usepackage{subfigure}
\usepackage{booktabs} % for professional tables
\usepackage{amsmath}
\usepackage{paralist} % for inparaenum
\usepackage{comment}

\DeclareMathOperator{\grad}{\nabla}

\usepackage{hyperref}

\usepackage{algorithm}
\usepackage{algorithmic}

\newlength\myindent
\newcommand\bindent{%
  \begingroup
  \setlength{\itemindent}{0.25cm}
}
\newcommand\eindent{\endgroup}

\jmlrheading{X}{XXXX}{X-XX}{X/XX}{XX/XX}{ago109}{Alexander G. Ororbia, Ankur Mali, Daniel Kifer, and C. Lee Giles}
\ShortHeadings{Local Representation Alignment}{Ororbia, Mali, Kifer, and Giles}
\firstpageno{1}

\begin{document}

\title{Conducting Credit Assignment by Aligning Local Distributed Representations}
% Conducting Credit Assignment by Aligning Local Distributed Representations

\author{\name Alexander G. Ororbia \email ago109@psu.edu \\
       \addr College of Information \& Sciences Technology\\
       Penn State University\\
       University Park, PA 16802, USA
       \AND
       \name Ankur Mali \email aam35@psu.edu \\
       \addr College of Information \& Sciences Technology\\
       Penn State University\\
       University Park, PA 16802, USA
       \AND
       \name Daniel Kifer \email dkifer@cse.psu.edu \\
       \addr School of Electrical Engineering and Computer Science\\
       Penn State University\\
       University Park, PA 16802, USA
       \AND
       \name C. Lee Giles \email giles@ist.psu.edu \\
       \addr College of Information \& Sciences Technology\\
       Penn State University\\
       University Park, PA 16802, USA}

\editor{}

\maketitle

\begin{abstract}%   <- trailing '%' for backward compatibility of .sty file
Using back-propagation and its variants to train deep networks is often problematic for new users. Issues such as exploding gradients, vanishing gradients, and high sensitivity to weight initialization strategies often make networks difficult to train, especially when users are experimenting with new architectures.
Here, we present Local Representation Alignment (LRA), a training procedure that is much less sensitive to bad initializations, does not require modifications to the network architecture, and can be adapted to networks with highly nonlinear and discrete-valued activation functions. Furthermore, we show that one variation of LRA can start with a null initialization of network weights and still successfully train networks with a wide variety of nonlinearities, including tanh, ReLU-6, softplus, signum and others that may draw their inspiration from biology.

A comprehensive set of experiments on MNIST and the much harder Fashion MNIST data sets show that LRA can be used to train networks robustly and effectively, succeeding even when back-propagation fails and outperforming other alternative learning algorithms, such as target propagation and feedback alignment.
\end{abstract}


\section{Introduction}
\label{intro}
In artificial neural networks, credit assignment is the task of computing the contribution to the overall error caused by individual units in the network, and then using this information to update the parameters of the entire network. Credit assignment and updates are most often done with the help of the gradient calculated by the well-known back-propagation of errors \citep{rumelhart1988backprop},\footnote{We will refer to back-propagation of errors also as ``backprop'' and ``back-propagation'' throughout.}  which provides a theoretical basis for training deep networks (i.e. gradient descent) but also presents challenges in practice due to the known vanishing/exploding gradient and shrinking variance problems.
%In artificial neural networks, credit assignment is the task of computing the contribution to the overall error caused by individual units in the network, and then using this information to update the parameters of the entire network. Credit assignment and updates are most often done with the help of the gradient calculated by the well-known back-propagation of errors \citep{rumelhart1988backprop},which provides a theoretical basis for training deep networks (i.e. gradient descent) but also presents challenges in practice due to the known vanishing/exploding gradient and shrinking variance problems.

The most common strategies for dealing with such problems include 
\begin{inparaenum}[(1)]
\item a careful initialization of the network parameters, often following from a network-specific analysis of the learning dynamics caused by back-propagation \citep{glorot2010understanding,he2015delving,Sussillo14,MishkinM15}, and
\item modifying the network structure, for example by using ReLU instead of sigmoid activations or adding batch normalization layers \cite{ioffe2015batch}.
\end{inparaenum}
Challenges such as these are often a barrier for new users to training deep networks to accuracies that approach the state of the art.

In this paper, we present a novel training algorithm, called Local Representation Alignment (LRA), that is robust to poor choices of initialization and can train deep networks from initializations that would cause backprop to fail. This allows network designers to choose units, including non-differentiable ones, based on the type of representation they provide rather than on the ideosyncrasies of backprop-based algorithms. 
%In particular, LRA also works on discrete, non-differentiable units or stochastic units.

The idea behind LRA is that every layer, not just the output layer, has a target and each layer's weights are adjusted so that its output moves closer to its target. While this idea is common to prior work such as TargetProp \citep{Miguel14,Bengio14,lee2015targetprop}, one key novelty of LRA is that it chooses targets that are in the possible representation of the associated layers and hence the layer's parameters can be updated more effectively (i.e. layers are not forced to try to match a target that is impossible to achieve). Thus, unlike innovations such as Difference Target Propagation \citep{lee2015targetprop}, Batch Normalization \citep{ioffe2015batch}, etc, LRA does not need to introduce new layers in the architecture. As a result it can be viewed either as an alternative to such approaches, or as a complementary technique because it is compatible with these other methods (i.e. it can be used with batch normalization layers and residual blocks and any other layers that are helpful for the problem-specific representations a deep network needs to acquire).

Our method for setting the targets treats a deep network as a collection of smaller, related subgraphs. This view allows us to flexibly incorporate ideas like feedback alignment \citep{lillicrap2016random} to train deep networks with non-differentiable activations -- specifically, we use the feedback matrix to update the targets rather than the weights. This variant of LRA can also train differentiable networks as fast as back-propagation but more robustly -- it is less sensitive to weight initializations than backprop and even other variants of feedback alignment \citep{lillicrap2016random,nokland2016direct}.

Another interesting feature of LRA is that it dynamically chooses which layers need to be trained - in the beginning, all layers, even the bottom-most layers in the network, receive significant updates; towards the end of training, only the top few layers need to be updated.

In the experiments of this paper, we compare two variations of the LRA  -- one where updates are based on calculus and another where error feedback weights are used (which turns out to be superior in robustness and speed). In our results, we show that LRA is:
\begin{itemize}
	\item robust to initialization when training highly nonlinear, deep networks. This result even holds in the extreme case of zero initialization, which back-propagation and target propagation cannot even handle.
    \item able to avoid the vanishing gradient problem and train deep networks rather independently of the nonlinearity used internally. This means we can train performant models composed of many units such as the classical logistic sigmoid.
    \item able to adapt the amount of computation expended during training. The depth of credit assignment is tied to how well a representation aligns with a target as governed by the local loss function. 
    \item able to learn networks that contain discrete-valued activation functions. % without resorting to biased estimators of the discrete function's gradient.
    \item able to learn networks that contain stochastic units.
    \item easily able to work with biologically inspired mechanisms, such as hard and soft lateral competition among hidden units. Since derivatives are no longer required for these mechanisms, a pathway to integrating other, even non-differentiable neurocognitively-motivated mechanisms is opened up.
\end{itemize}

We first explain how back-propagation can be re-cast in the target propagation framework of \citet{lee2015targetprop} in Section \ref{sec:lra}. This makes it easier to identify weaknesses in the ``backprop as targetprop'' viewpoint and explain the modifications that give rise to LRA. Then we explain how a further modification allows LRA to work with discrete, non-differentiable units, as well as stochastic units. We discuss related work in Section \ref{sec:review} and present experimental results in Section \ref{sec:exp_results} using MNIST and Fashion MNIST, a new and much more challenging benchmark.

\section{Local Representation Alignment}
\label{sec:lra}

\begin{figure*}
\centering     %%% not \center
\subfigure[LR-diff target search.]{\label{fig:lra_target}\includegraphics[width=74mm]{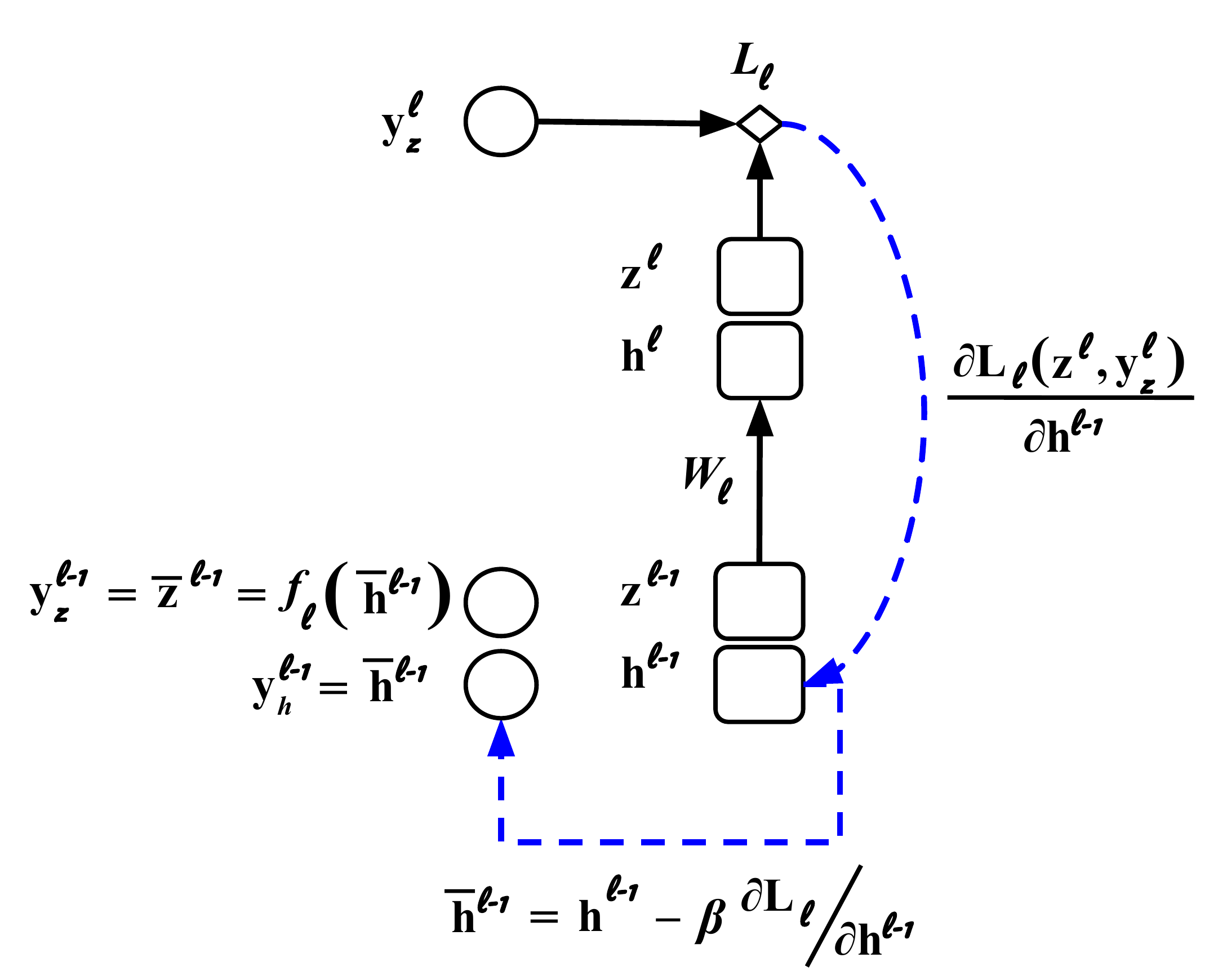}
}
\subfigure[LRA-fdbk target search.]{\label{fig:lra_update}\includegraphics[width=74mm]{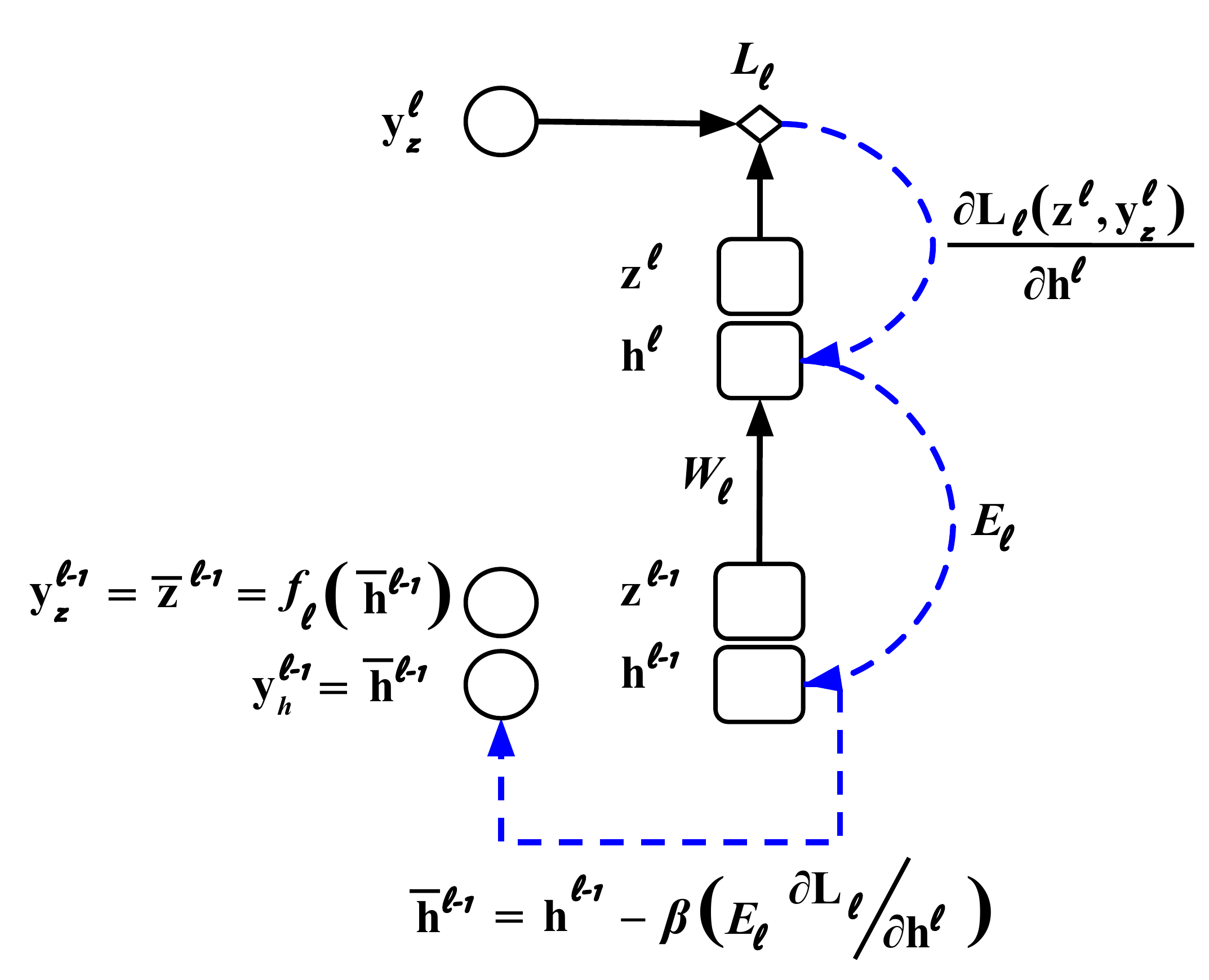}
}
\caption{The target calculation process for two variants of LRA applied to a differentiable multilayer perceptron (MLP). A computational subgraph is formed from two layers. Given the target for layer $\ell$, we compute a target $\mathbf{y}_h^{\ell-1}$ for the pre-activation of layer $\ell-1$ and use it to obtain a target $\mathbf{y}^{\ell-1}_z$ for the output of layer $\ell-1$. % for a given $(\mathbf{y},\mathbf{x})$. In particular, we focus on a computation subgraph composed of all operations and calculations between layers $\ell-1$ (input layer) and $\ell$ (output layer). Assuming a target has been found for the output layer, $\ell$, both processes are shown for $k=3$ iterations (including initial feedforward calculation). 
}
\label{fig:lra_algos}
\end{figure*}

To simplify notation, we describe LRA in the context of a multilayer perceptron architecture. However, LRA naturally applies to any stacked neural architecture, including those that are recurrent. In the supplementary material we provide the general formulation of LRA such that it may be applied recurrent models as well.

\subsection{Notation}
Let $\mathbf{z}^{\ell-1}$ be the inputs, as computed in the feed-forward phase, to the nodes in layer $\ell$. This means that $\mathbf{z}^\ell$ is also the output of layer $\ell$ (hence we call it the \emph{post-activation}). Let $W_\ell$ be the weight matrix at layer $\ell$ that multiplies the inputs $\mathbf{z}^{\ell-1}$. Let $\mathbf{h}^\ell$ be the \emph{pre-activation} of layer $\ell$ (i.e. $\mathbf{h}^\ell=W_\ell \mathbf{z}^{\ell-1}$). Let $\mathbf{f}^\ell$ be the vector of activation functions for layer $\ell$, so that $\mathbf{z}^\ell=\mathbf{f}^\ell(\mathbf{h}^\ell)$.

Following TargetProp \citep{lee2015targetprop}, during training, for each layer $\ell$ (starting from the output layer and working down the network) we set a target $\mathbf{y}_z^\ell$ for the feedforward post-activations $\mathbf{z}^\ell$. Then, for each layer $\ell$,  we take a gradient descent step  on the network parameters $W_{\ell}$ of layer $\ell$ to move $\mathbf{z}^\ell$  closer to $\mathbf{y}_z^\ell$ by minimizing a layer-specific loss function $L_\ell(\mathbf{z}^\ell, \mathbf{y}_z^\ell)$. Thus we iteratively choose new targets, then refine weights, then choose new targets, etc.
\footnote{During the weight update, for the purposes of computing derivatives, the current target $\mathbf{y}_z^\ell$ is not treated as a function of the input weights (so that the feed-forward output moves closer to the desired targets rather than vice versa).} 
However, unlike TargetProp \citep{lee2015targetprop}, LRA will not modify the network architecture. The setting of targets in differentiable networks is illustrated in Figure \ref{fig:lra_algos} (which shows two variants of LRA). Handling non-differentiable and stochastic activations requires a small modification that is discussed in Section \ref{non_diff_fun}.

\subsection{Setting the targets.}
We first explain how backpropagation fits into this framework. One way, noted by \citet{lee2015targetprop}, is to set the target for the output of layer $\ell$ to be its current feed-forward value plus the gradient of the overall loss with respect to the output of layer $\ell$ (and the loss function is squared loss). This is a global view - the target of each layer is set to optimize the global loss.  Instead, we present a more local view in that the target for layer $i$ is specifically set to help layer $\ell+1$ reach its own target. This view is the starting point that will allow us to consider the training of a deep network as training on computation subgraphs (as in Figure \ref{fig:lra_algos}) that can be optimized locally.

In our view of backprop as target prop, we use $L_2$ loss for each layer and set the target $\mathbf{y}_z^\ell$ of layer $\ell$ so that the difference between the layer $\ell$ output $\mathbf{z}^\ell$ and its target $\mathbf{y}_z^\ell$ is the gradient of the next\footnote{i.e., $\ell+1$} layer's loss with respect to $\mathbf{z}_\ell$ (i.e. the direction that $\mathbf{z}^\ell$ should move to achieve the steepest local change in the next layer's output).
\begin{lemma}
Consider a network with $n$ layers, input $\mathbf{x}$, feedforward activations $\mathbf{z}^\ell$ ($\ell=1,\dots, n$), output layer target $\mathbf{t}$, and squared loss $L_\ell$ for layers $1,\dots,\ell-1$ (that is, except the top layer). Recursively set $\mathbf{y}_z^n=\mathbf{t}$ and, for $\ell=n-1,\dots,1$, set $\mathbf{y}_z^\ell=\mathbf{z}^\ell - \nabla_{\mathbf{z}^\ell} L_{\ell+1}(\mathbf{z}^{\ell+1}, \mathbf{y}_z^{\ell+1})$, where $\mathbf{z}^{\ell+1}$ is considered a function of $\mathbf{z}^{\ell}$. A simultaneous one-step gradient descent update to all the weight matrices $W_1, \dots W_n$ with respect to the layer-wise losses is equivalent to one step of back-propagation.
\end{lemma}
\begin{proof}
It suffices to show that if $W_{\ell,j}$ is the  weight vector for node $j$ in layer $\ell$,  then the derivative of the loss for layer $\ell$ with respect to these weights is the same as the derivative of the overall loss with respect to these weights: $\frac{d}{dW_{\ell,j}}L_\ell(\mathbf{z}^{\ell}, \mathbf{y}_z^{\ell}) \equiv\frac{d}{dW_{\ell,j}}\frac{1}{2}||\mathbf{z}^{\ell}-\mathbf{y}_z^{\ell}||_2^2 = \grad_{W_{\ell,j}}L_n(\mathbf{z}^n, \mathbf{t})$ for each $j$ and for $\ell=1,\dots, n-1$ (since the update to the weights $W_n$ at the top layer is always a gradient descent update). We will prove the stronger statement, that for any $k \geq \ell$, then $\frac{d}{dW_{\ell,j}}L_k(\mathbf{z}^{k}, \mathbf{y}_z^{k})  = \grad_{W_{\ell,j}}L_n(\mathbf{z}^n,\mathbf{t})$ for each $j$ and for $\ell=1,\dots, n-1$

In the base case, when $k=n-1$, 
\begin{align*}
\lefteqn{\frac{d}{dW_{\ell},j}L_{n-1}(\mathbf{z}^{n-1},\mathbf{y}_z^{n-1})
=\frac{d}{dW_{\ell},j}\frac{1}{2}||\mathbf{z}^{n-1}-\mathbf{y}_z^{n-1}||_2^2}\\
&=\frac{d\mathbf{z}^{n-1}}{dW_{\ell,j}}(\mathbf{z}^{n-1}-\mathbf{y}_z^{n-1})\Big|_{\mathbf{y}_z^{n-1}=\mathbf{z}^{n-1} - \nabla_{\mathbf{z}^{n-1}} L_n{(\mathbf{z}^n, \mathbf{t})}}\\
&=\frac{d\mathbf{z}^{n-1}}{dW_{\ell,j}}\nabla_{\mathbf{z}^{n-1}} L_n{(\mathbf{z}^{n}, \mathbf{t})}
=\frac{d}{dW_{\ell,j}}L_n{(\mathbf{z}^{n}, \mathbf{t})}
\end{align*}
by the chain rule. For the inductive step, we show that if the result is true for $k$ (i.e. $\frac{d}{dW_{\ell,j}}L_k(\mathbf{z}^{k}, \mathbf{y}_z^{k})  = \grad_{W_{\ell,j}}L_n( \mathbf{z}^n, \mathbf{t})$) then it is also true for $k-1$ (i.e. $\frac{d}{dW_{\ell,j}}L_{k-1}(\mathbf{z}^{k-1}, \mathbf{y}_z^{k-1})  = \grad_{W_{\ell,j}}L_n( \mathbf{z}^n, \mathbf{t})$). So
\begin{align*}
\lefteqn{\frac{d}{dW_{\ell},j}L_{k-1}(\mathbf{z}^{k-1},\mathbf{y}_z^{k-1})
=\frac{d}{dW_{\ell},j}\frac{1}{2}||\mathbf{z}^{k-1}-\mathbf{y}_z^{k-1}||_2^2}\\
&=\frac{d\mathbf{z}^{k-1}}{dW_{\ell,j}}(\mathbf{z}^{k-1}-\mathbf{y}_z^{k-1})\Big|_{\mathbf{y}_z^{k-1}=\mathbf{z}^{k-1} - \nabla_{\mathbf{z}^{k-1}} L_{k}(\mathbf{z}^{k}, \mathbf{y}_z^{k})}\\
&=\frac{d\mathbf{z}^{k-1}}{dW_{\ell,j}}\nabla_{\mathbf{z}^{k-1}} L_{k}{(\mathbf{z}^{k}, \mathbf{y}_z^{k})}\\
&=\frac{d}{dW_{\ell,j}}L_{k}{(\mathbf{z}^{k}, \mathbf{y}_z^{k})}
=\frac{d}{dW_{\ell}}L_{n}{(\mathbf{z}^{n}, \mathbf{t})}
\end{align*}
by the chain rule and inductive hypothesis.
\end{proof}

Hence mini-batch gradient descent can be viewed as, for every iteration, selecting a mini-batch, choosing targets for each layer to obtain a per-layer optimization problem, and partially optimizing the layer-wise loss (with one gradient step). 

This view naturally leads to three ways to improve training:
\begin{itemize}
\item The target $\mathbf{y}_z^\ell$, intuitively, is a desired value for the output of layer $\ell$ that will help layer $\ell+1$ lower its loss. Thus it is important to ensure that the target is actually representable by layer $\ell$. Therefore we should look at the pre-activation $\mathbf{h}^\ell$ (i.e. inputs to the nodes at layer $\ell$) and determine what values of $\mathbf{h}^\ell$, when fed through the activation function of layer $\ell$, will help layer $\ell+1$ reduce its loss. We set $\mathbf{y}_h^\ell$ to be the target for the pre-activation $\mathbf{h}^\ell$ and feed these pre-activation targets through the activation function to obtain the targets $\mathbf{y}_z^\ell$ for layer $\ell$, as shown in Figures \ref{fig:lra_algos}.  One way to set the pre-activation target is through one step of gradient descent on the local loss function: $\mathbf{y}_h^\ell = \mathbf{h}^\ell - \eta \grad_{\mathbf{h}^{\ell}} L_{\ell+1}(\mathbf{z}^{\ell+1},\mathbf{y}^{\ell+1}_z)$ and then $\mathbf{y}_z^\ell = \mathbf{f}^\ell(\mathbf{y}_h^\ell)= \mathbf{f}^\ell\left(\mathbf{h}^\ell - \eta \grad_{\mathbf{h}^{\ell}} L( \mathbf{z}^{\ell+1}-\mathbf{y}^{\ell+1}_z)\right)$.
\item To choose better pre-activation targets $\mathbf{y}_h^\ell$, we can perform multiple gradient descent steps for $\mathbf{y}_h^\ell$ on the loss of the next layer $L_{\ell+1}(\mathbf{z}^{\ell+1}, \mathbf{y}_z^{\ell+1})$.
%to find a more helpful value of the pre-activation target. 
Such a procedure can be viewed as walking along the manifold of $\mathbf{z}^\ell$ parametrized by $\mathbf{h}^\ell$, to find a pre-activation target $\mathbf{y}_h^\ell$ (and hence $\mathbf{y}_z^\ell$) that would help the next layer reduce its loss. An alternative to gradient steps is to use feedback alignment to update the targets (instead of the standard approach of updating the weights \cite{nokland2016direct,lillicrap2016random}). 
\item The per-layer loss can be customized for each layer. For example, the least squares loss can be replaced by the $L_1$ norm or the log-penalty (Cauchy).
\end{itemize}
After the targets are set, the weights are updated with one step of gradient descent: $\mathbf{W}_\ell\gets \mathbf{W}_\ell - \grad_{\mathbf{W}_\ell} \mathcal{L}_\ell(\mathbf{z}^\ell, \mathbf{y}^\ell_z)$ -- possibly using an adaptive learning rate rule, e.g., Adam or RMSprop. To handle discrete-valued and/or stochastic activations, we employ the notion of a ``short-circuit'' connection (only used during training) to form the error pathway around the activation. This detail is discussed in Section \ref{non_diff_fun}.  %, instead of resorting to full reverse-mode differentiation of the subgraph's local loss with respect to the input pre-activities. 

\subsection{Divide and Conquer: The Computation Subgraph}
\label{comp_subgraph}
LRA aims to decompose the larger credit assignment problem in neural architectures into smaller, easier-to-solve sub-problems. With this in mind, we can view any stacked neural architecture, or rather, its full, underlying operation graph, as a composition of ``computation subgraphs''. A directed, acyclic computation graph may be decomposed into a set of smaller direct subgraphs, where a subgraph's boundaries are defined as the set of input node variables and the set of output node variables. 

To be more specific, the input to a subgraph of a multilayer neural architecture is the vector of pre-activation values, $\mathbf{h}_{\ell-1}$ computed from the subgraph below (unless this is the bottommost subgraph which means that the input is simply the raw feature vector). The output of the subgraph is simply the post-activation it computes as a function of its input, or simply, $\mathbf{z}_{\ell}$. Note that while we present our notation to imply that a computation subgraph only encapsulates two layers of actual processing elements, layers $\ell-1$ and $\ell$, the subgraph itself could be deep and include processing elements one decides are internal to the subgraph itself. This would allow us to house self-connected nodes inside the graph as well, assuming that the subgraph is temporal and is intended to compute given sequences of inputs/outputs. In choosing boundaries, one could distinguish which node layers are to be representations (or latent variables) and which node layers are simply inner computation elements, meant to support the representation nodes. 
%In the appendix, we depict how one might handle a graph that contains self-connected nodes and with nodes that play the different roles of representation versus supporting/inner computation.
%
Figure \ref{fig:lra_algos} show one such  subgraph 
%emphasized against the general computation structure inside the red dashed boundary lines (in the case of a feedforward model, the graph is simply a linked chain of operations). 
Note that this subgraph also includes a loss function $\mathcal{L}_\ell$ and the targets $\mathbf{y}^\ell_z$.

% do 3-layer network
% 1) do feedforward pass
% 2) compute displacements and targets (with k steps)
% 3) compute parameter updates

\begin{algorithm}%[tb]
   \caption{The LRA Algorithm applied to differentiable, feedforward neural architectures. 
   %For architectures with non-differentiable components, see supplementary material.
   }
   \label{algo:lra}
\begin{algorithmic}
   \STATE {\bfseries Input:} data $(\mathbf{x}, \mathbf{t})$, number steps \emph{K}, halting criterion $\epsilon$, step-size $\beta$, model parameters $\Theta = \{W_1, \mathbf{b}_1, \cdots, W_\ell, \mathbf{b}_\ell, \cdots, W_n, \mathbf{b}_n \}$, and norm constraints $\{ c_1, c_2 \}$
   \STATE {\bfseries Specify:} $\{f_1(\mathbf{h}), \cdots, f_\ell(\mathbf{h}), \cdots, f_n(\mathbf{h}) \}$, $\{ \mathcal{L}_1(\mathbf{z}, \mathbf{y}), \cdots, \mathcal{L}_\ell(\mathbf{z}, \mathbf{y}), \cdots, \mathcal{L}_n(\mathbf{z}, \mathbf{y}) \}$, and, optionally, error weights $\{E_1, \cdots, E_\ell, \cdots, E_n \}$
   \STATE \COMMENT $\mathbf{y}^\ell_h$: what we would like input to the activation function at layer $\ell$ to be
   \STATE \COMMENT $\mathbf{y}^\ell_z$: what we would like output of activation function $f_\ell$ at layer $\ell$ to be
   \STATE \COMMENT $\mathbf{h}^\ell$: input to activation function at layer $\ell$ resulting from feed forward phase
      \STATE \COMMENT $\mathbf{z}^\ell$: output of activation function at layer $\ell$ resulting from feed forward phase
      \STATE \COMMENT $\mathbf{\bar{h}}^\ell$ indicates temporary variable for $\mathbf{h}^\ell$ and $\mathbf{\bar{z}}^\ell$ indicates temporary variable for $\mathbf{z}^\ell$
    \STATE \textbf{Function}\{ComputeUpdateDirection\}\{$(\mathbf{x}, \mathbf{t}),\Theta, K, c_1, c_2, \beta, \epsilon $\}
    \bindent
    \STATE $\mathbf{z}^0 = \mathbf{x} $ \COMMENT Run feedforward pass to get initial layer-wise statistics
	\FOR{$\ell=1$ {\bfseries to} $n$}
    	\bindent
        \STATE $W_\ell, \mathbf{b}_\ell, \leftarrow \Theta$ %// Extract relevant layerwise parameters
    	\STATE $\mathbf{h}^\ell \leftarrow W_\ell \mathbf{z}^{\ell-1} + \mathbf{b}_\ell $, $\mathbf{y}^\ell_h \leftarrow \mathbf{h}^\ell$, $\mathbf{\bar{h}}^\ell \leftarrow \mathbf{h}^\ell$
    	\STATE $\mathbf{z}^\ell \leftarrow f_\ell( \mathbf{h}^\ell )$, $\mathbf{y}^\ell_z \leftarrow \mathbf{z}^\ell$, $\mathbf{\bar{z}}^\ell \leftarrow \mathbf{z}^\ell$
        \eindent
    \ENDFOR
    \STATE $\mathbf{y}^n_z \leftarrow \mathbf{t}$ \COMMENT Override top-level target with correct target from training data
    \STATE $\ell = n$
    \WHILE{$\ell \geq 1$ and $\mathcal{L}_\ell( \mathbf{z}^{\ell} , \mathbf{y}^\ell_z) \geq \epsilon$}
    	\bindent
          \STATE $W_\ell, \mathbf{b}_\ell, \leftarrow \Theta$
    	%\IF{$\mathcal{L}_\ell( \mathbf{z}^{\ell} , \mathbf{y}^\ell_z) \geq \epsilon$} 
         % \bindent
           \STATE \COMMENT Calculate parameter update direction for layer $\ell$, comparing initial guess to target
           \STATE \COMMENT Normalize$(\cdot,\cdot)$ is defined in Eq \ref{eqn:normalize}.
          \STATE $\nabla_{W_\ell} \leftarrow Normalize( \frac{\partial \mathcal{L}_\ell( \mathbf{z}^{\ell}, \mathbf{y}^\ell_z)}{\partial W_\ell }, c_1 )$, $\nabla_{\mathbf{b}_\ell} \leftarrow Normalize( \frac{\partial \mathcal{L}_\ell( \mathbf{z}^\ell , \mathbf{y}^\ell_z)}{\partial \mathbf{b}_\ell }, c_1 ) $
          \STATE \COMMENT Find target for layer $\ell-1$ (Note: could add early stopping criterion)
          %\STATE \COMMENT Note: skip $\ell = 1$ target computation if credit assignment goes this deep
          \FOR{$k = 1$ {\bfseries to} $K$} 
          \bindent
          		\STATE \COMMENT Calculate pre-activation displacement
          		\IF{LRA-diff} 
          		\bindent
          		\STATE $\Delta_{h^{\ell-1}} \leftarrow \frac{\partial \mathcal{L}_\ell(\mathbf{z}^\ell, \mathbf{y}^\ell_z) }{\partial \mathbf{h}^{\ell-1}} = (W_\ell)^T \Big( \frac{\partial \mathcal{L}_\ell(\mathbf{z}^\ell, \mathbf{y}^\ell_z) }{\partial \mathbf{z}^\ell} \otimes \frac{\partial f_\ell(\mathbf{h}_\ell) }{\partial \mathbf{h}^\ell} \Big) \otimes \frac{\partial f_{\ell-1}(\mathbf{h}^{\ell-1}) }{\partial \mathbf{h}^{\ell-1}}$
          		\eindent
          		\ELSIF{LRA-fdbk}
          		\bindent
          		\STATE $\Delta_{h^{\ell-1}} \leftarrow E_\ell \frac{\partial \mathcal{L}_\ell(\mathbf{z}^\ell, \mathbf{y}^\ell_z) }{\partial \mathbf{h}^{\ell}} = E_\ell \Big( \frac{\partial \mathcal{L}_\ell(\mathbf{z}^\ell, \mathbf{y}^\ell_z) }{\partial \mathbf{z}^\ell} \otimes \frac{\partial f_\ell(\mathbf{h}^\ell) }{\partial \mathbf{h}^\ell} \Big)$
          		\eindent
          		\ENDIF
          		\STATE $\Delta_{h^{\ell-1}} \leftarrow Normalize( \Delta_{h^{\ell-1}}, c_2 ) $ %\COMMENT See text for definition of \emph{Normalize}$(\cdot)$
          		\STATE \COMMENT Recalculate neural activities of subgraph
          		\STATE $\mathbf{\bar{h}}^{\ell-1} \leftarrow \mathbf{\bar{h}}^{\ell-1} - \beta \Delta_{h^{\ell-1}}, \quad \mathbf{\bar{z}}^{\ell-1} \leftarrow f_{\ell-1}( \mathbf{\bar{h}}^{\ell-1} )$
          		\STATE $\mathbf{\bar{h}}^\ell \leftarrow  W_\ell \mathbf{\bar{z}}^{\ell-1} + \mathbf{b}_\ell, \quad $ $\mathbf{\bar{z}}^\ell \leftarrow f_\ell( \mathbf{\bar{h}}^\ell )$ 
          \eindent
          \ENDFOR
          \STATE $\mathbf{y}^{\ell-1}_z \leftarrow \mathbf{\bar{z}}^{\ell-1}$ \COMMENT Update variable holding target for subgraph below
       %   \eindent
       % \ENDIF
        \STATE $\ell = \ell - 1$
        \eindent
        \ENDWHILE
	\STATE \textbf{Return} $\nabla_{\Theta} = \{ \nabla_{W_1}, \nabla_{\mathbf{b}_1}, \cdots, \nabla_{W_\ell}, \nabla_{\mathbf{b}_\ell}, \cdots, \nabla_{W_n}, \nabla_{\mathbf{b}_n} \}$
    \eindent
    \STATE \textbf{EndFunction}
    %\STATE % empty space to return
\end{algorithmic}
\end{algorithm}

\subsubsection{Instantiation: The Differentiable Multilayer Perceptron}
\label{lra:mlp_case}
Although LRA can be extended to recurrent networks (an extension is presented in the supplementary materials), in this paper, for the sake of explanation, we specialized it to 
feedforward neural architectures. 

Algorithm \ref{algo:lra} presents the pseudocode for LRA for differentiable networks (for non-differentiable networks, additional notation is needed, so this modification is deferred to Section \ref{non_diff_fun}). The pseudocode presents two versions of LRA: LRA-diff (the immediate generalization of backprop we have been discussing) and LRA-fdbk which incorporates ideas from feedback alignment \cite{lillicrap2016random}. The difference between them is that LRA-diff sets the target for layer $\ell$ using the full derivative $\frac{\delta L_{\ell+1}(\mathbf{z}^{\ell+1},\mathbf{y}^{\ell+1}_z)}{\delta \mathbf{h}^{\ell}}$. Using the chain rule, this derivative is expressed as $\frac{\delta \mathbf{h}^{\ell+1}}{\delta \mathbf{h}^\ell}\frac{\delta L_{\ell+1}(\mathbf{z}^{\ell+1},\mathbf{y}^{\ell+1}_z)}{\delta \mathbf{h}^{\ell+1}}$. LRA-fdbk, during training, ``short-circuits'' the connection between $\mathbf{h}^{\ell+1}$ and $\mathbf{h}_\ell$ (as in Figure \ref{fig:lra_algos}) by replacing 
$\frac{\delta \mathbf{h}^{\ell+1}}{\delta \mathbf{h}^\ell}$
in the chain rule with a fixed matrix $\mathbf{E}_{\ell+1}$ that is randomly chosen before the start of training (e.g., sample its weights from a standard Gaussian).

Such short-circuit operations have been shown to be useful empirically \citep{lillicrap2016random,nokland2016direct} although they are not very well understood theoretically. Their use in prior work even allowed networks consisting of tanh nonlinearities (but not ReLU) to be trained from initial weights equal to 0 \cite{nokland2016direct}. In contrast, our experiments show that LRA-fdbk is even more robust and can train a much broader set of networks from 0. 

%The general formulation of the procedure, which can be applied to any operation-graph including those that define recurrent networks, can be found in the appendix. % AO: add in the appendix later...
%$\mathcal{L}^{vec}_\ell( \mathbf{z}_{\ell} , \mathbf{y}^z_\ell)$ simply means a vector of computed losses, one per sample (so in a mini-batch of $m$ samples, we would obtain a vector of $(1 \times m)$ scalar loss values, assuming column-major orientation), as opposed to the full, summed loss $\mathcal{L}_\ell( \mathbf{z}_{\ell} , \mathbf{y}^z_\ell)$. In the event that mini-batches of samples are used, we create a so-called ``depth mask'' $\mathbf{m}$ to allow LRA to decide how deep the credit assignment should go on a per-sample basis. 
In Algorithm \ref{algo:lra}, $\otimes$ is used to denote the Hadamard product (or elementwise multiplication). 
%When using the first displacement function of Algorithm \ref{algo:lra}, we refer to the resulting LRA procedure as \emph{LRA-diff}, to make clear that the full subgraph must be fully differentiable. If the second displacement function is used, we will refer to the algorithm as \emph{LRA-fdbk} in order to indicate that error feedback connections are used to short-circuit the subgraph and avoid the need for full differentiability.

The normalization function $Normalize(\cdot)$ depicted in Algorithm \ref{algo:lra} is defined formally as:
\begin{align}
Normalize(\Delta, c) = \Bigg \{\frac{c}{||\Delta||}\Delta \mbox{, if } ||\Delta|| \geq c \mbox{, and } \Delta \mbox{, if } ||\Delta|| < c \Bigg \}\label{eqn:normalize}
\end{align}
where $\Delta$ is any vector  
%(either for a neural activity pattern or a weight matrix/vector) given as an input argument 
%and $||\cdot||$ is chosen to be the Euclidean, or L2, norm. $c$ is the norm constraint or threshold. This implementation is based on the gradient re-projection proposed in \cite{pascanu2013difficulty} for back-propagation of errors.

%A multilayer perceptron, or feedforward network, is defined by the set of parameters $\Theta = \{W_1, \mathbf{b}_1, \cdots, W_\ell, \mathbf{b}_\ell, \cdots, W_L, \mathbf{b}_\ell \}$. Note that $\mathbf{b}_\ell$ is the bias vector applied to units at layer $\ell$. To measure the discrepancy between a hidden layer's current representation, conditioned on data, and a computed target one may employ many possible loss functions depending on the distributions assumed over hidden units. If we choose to assume a Gaussian distribution (with an identity covariance matrix), the (instantaneous) local loss to measure the agreement between the representation and target would be 

There are many possible choices for the local losses that measure the discrepancy between a layer's output and its target. One possibility is
the L2-norm, defined as
\begin{align}
\mathcal{L}_\ell(\mathbf{z}, \mathbf{y}) = \frac{1}{2} \sum^{|\mathbf{z}|}_{i=1} (\mathbf{y}_i - \mathbf{z}_i)^2 \mbox{.} \label{gaussian_loss}
\end{align}
Another choice is the $L_1$ norm, which is defined as:
\begin{align}
\mathcal{L}_\ell(\mathbf{z}, \mathbf{y}) = \sum^{|\mathbf{z}|}_{i=1} |(\mathbf{y}_i - \mathbf{z}_i)|\mbox{.} \label{laplacian_loss}
\end{align}
After preliminary experimentation, we actually found a different loss, the log-penalty, to work much better with LRA for a wide variety of networks. %, including non-zero mean functions like the logistic sigmoid. 
The log-penalty function is derived from the log likelihood of the Cauchy distribution. In this paper, we implemented the log-penalty loss, for a single vector, as:
\begin{align}
\mathcal{L}_\ell(\mathbf{z}, \mathbf{y}) = \sum^{|\mathbf{z}|}_{i=1} \log(1 + (\mathbf{y}_i - \mathbf{z}_i)^2) \label{cauchy_loss}
\end{align}
where the loss is computed over all dimensions $|\mathbf{z}|$ of the vector $\mathbf{z}$ (where a dimension is indexed/accessed by integer $i$). 
%This means we assume a Cauchy distribution, over the representation space at layer $\ell$, $P(\mathbf{y_i}) \propto 1 / (1 + (\mathbf{y}_i - \mathbf{z}_i))$, with a location parameter $\mathbf{z}$ that is learned from data. Future research should develop alternative metrics to the ones we propose in order to learn neural architectures with hidden activation patterns that align better with calculated target representations of input data.

%Note that the while loop is able to terminate the backwards pass early when the lower level feed-forward activations are already close to their targets, thus providing additional savings of computation.

LRA also performs variable depth credit assignment because of the condition in the while loop that stops the backward pass early when the feedforward activation of a layer is close to its target (i.e. the local loss  is at most $\epsilon$).
This feature is not strictly necessary, but it is a nice addition that allows it to save computation by eventually only modifying the top layers of a network.

Finally, for the inner loop which successively refines the target, we found that with LRA-fdbk, the best performance is achieved when $K=1$. This setting also makes its computational requirements comparable to backprop.

\subsubsection{Handling Non-Differentiable/Stochastic Activations}
\label{non_diff_fun}
The only thing that prevents the LRA-fdbk version of Algorithm \ref{algo:lra} from handling non-differentiable and stochastic units is the computation of $\frac{\partial \mathcal{L}_\ell(\mathbf{z}^\ell, \mathbf{y}^\ell_z) }{\partial \mathbf{h}^{\ell}}$ as it involves, via chain rule, the derivative of the activation function $f_\ell$ of layer $\ell$.

This difficulty is easily circumvented with a trick inspired by \citet{lee2015targetprop}. We illustrate it with the following two activations (1)  $sign(h)$ (also known as the Heavyside step function) which returns $-1$, 0, or 1, depending on whether $h$ is negative, $0$, positive, respectively, and (2) $bernoulli(h)$ which outputs 1 with probability $\sigma(h)$ and $0$ with probability $1-\sigma(h)$ (where $\sigma$ is the sigmoid function). 

This means that we can rewrite these activations as a composition of two functions $g(f(h))$, where $f$ is a differentiable approximation of the activation function. For instance:
\begin{align*}
sign(h) &= sign(tanh(h))\\
bernoulli(h) &= bernoulli^*(\sigma(h))
\end{align*}
where bernoulli$^*(x)$ returns 1 with probability $x$ and $0$ otherwise.

Splitting activation functions this way allows us to extend our notation so that:
\begin{align*}
\mathbf{z}^\ell &= f(\mathbf{h}^\ell)\\
\mathbf{z}^\ell_* &= g(\mathbf{h}^\ell)
\end{align*}
Now, $\mathbf{z}^\ell$ is an intermediate output and $\mathbf{z}^\ell_*$ is the output of layer $\ell$.
The modification to LRA is almost trivial: we use $\mathbf{z}^\ell_*$ in the feed-forward phase\footnote{i.e. in Algorithm \ref{algo:lra} we set $\mathbf{h}^\ell = W_\ell \mathbf{z}^{\ell-1}_* + \mathbf{b}_\ell$ then   $\mathbf{z}^\ell = f_\ell(\mathbf{h}^\ell)$, and then  $\mathbf{z}^\ell_* = g(\mathbf{z}^\ell)$} but set the targets for $\mathbf{z}^\ell$ instead of $\mathbf{z}^\ell_*$. For reference, the complete algorithm is shown in the supplementary materials.

\subsection{Overcoming Poor Initializations}
\label{poor_inits}
Poor initializations affect networks with various activations differently, as we shall observe in the experiments later in this chapter. LRA, in both forms, can be seen as correcting for poor settings--something back-propagation of errors cannot do. LRA-diff can, if given a large enough local computation budget $K$, ``walk away'' from poor settings quickly. That is, even if the penultimate layer $n-1$ provides bad features for the final classification, the target for that layer will be a set of features that help the final layer make a better prediction. Then, recursively, if layer $n-2$ provides bad features, the target for that layer will be a set of features that will help layer $n-1$ achieve its target which, in turn, will help the final layer.
%(as also observed in Experiment 1)
%since its goal is to first find better layerwise representations that explain the relationship between the input $\mathbf{x}$ and output $\mathbf{y}$.

In the experiments, we investigate how robust LRA is to various settings of the initialization scheme and how other algorithms, especially target propagation and feedback alignment, compare.

\subsection{Overcoming Exploding and Vanishing Gradients}
\label{gradient_problems}
When neural networks are made deeper, backprop error gradients must pass backward through many layers using the global feedback pathway that involves a series of multiplications. As a result of these extra multiplications, these gradients tend to either explode or vanish \citep{bengio1994learning,pascanu2013difficulty}. In order to keep the values of the gradients within reasonable magnitudes and prevent zero gradients, it is common to impose constraints to ensure that layers are sufficiently linear in order to prevent post-activations from reaching their saturation regimes. However, this required linearity can create less than desirable side-effects, e.g., adversarial samples \citep{szegedy2013intriguing,ororbia2017unifying}.

LRA handles the vanishing gradient problem by tackling the global credit assignment in a local fashion, using the perspective of computation subgraphs as described earlier in Section \ref{comp_subgraph}.  In other words, LRA treats the underlying graph of the neural graphical model as a series of subgraphs and then proceeds to optimize each subproblem. This is the essence of local learning in LRA. 
To overcome exploding gradients, LRA makes use of gradient re-projection, which is often used to introduce stability in recurrent neural networks \citep{pascanu2013difficulty}. Re-projection, embodied in the $Normalize(\cdot)$ function call, is used in two places within the LRA procedure--1) rescale parameter updates $\nabla_{W_\ell}$ to have Frobenius norm $c_1$ and $\nabla_{b_\ell}$ to have $L_2$ norm equal to $c_1$, and 2) rescale the calculated representation displacement to have $L_2$ norm equal to $c_2$. 
%The exact implementation of $Normalize(\cdot)$ is identical to that of Algorithm 1 of \cite{pascanu2013difficulty}.

\section{Experimental Results}
\label{sec:exp_results}
The goal of these experiments is to test how easy it is to train deep networks with different algorithms (compared to LRA) and how robust these algorithms are to various settings, such as choice of initialization weights. Our claim is that LRA makes training of algorithms very easy and does not require much tuning to achieve high levels of accuracy -- our goal is not necessarily to reach the state-of-the-art on any particular task (which typically requires expensive hyperparameter tuning). One of the use-cases of LRA is for researchers outside of deep learning who want to experiment with many different (and possibly novel) architectures designed for their data.

For all experiments in this paper, we keep the parameter optimization setting the same for all scenarios so that we may tease out the effects of individual learning algorithms instead. Specifically, updates calculated by each algorithm are used in a simple first-order gradient descent with a fixed learning rate of $0.01$ and mini-batches of 50 samples. 
%Our goal is not to reach state-of-the-art on any particular task but to instead to investigate the optimization ability and robustness of our algorithm in comparison to back-propagation and alternatives.Since finding out optimal hyperparemeters for any neural network is non-trivial task and poor initialization can lead to unstable or poor performance.
%In our experiments we found that LRA helps neural network learn in settings that cause alternative algorithms to fail
%in worst settings and overcome limitations of alternative algorithm. 

We briefly describe the datasets used to investigate the ability of each learning algorithm in training deep, nonlinear networks.

\begin{figure*}
\centering     %%% not \center
{\includegraphics[width=100mm]{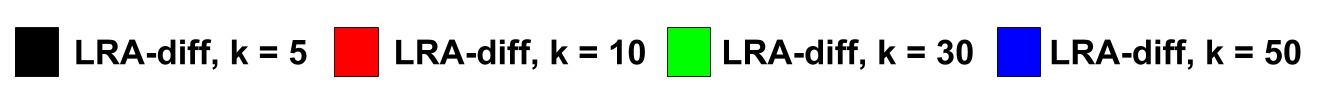}}
\subfigure[10 epochs of \emph{LRA-diff} (MNIST).]{\includegraphics[width=63mm]{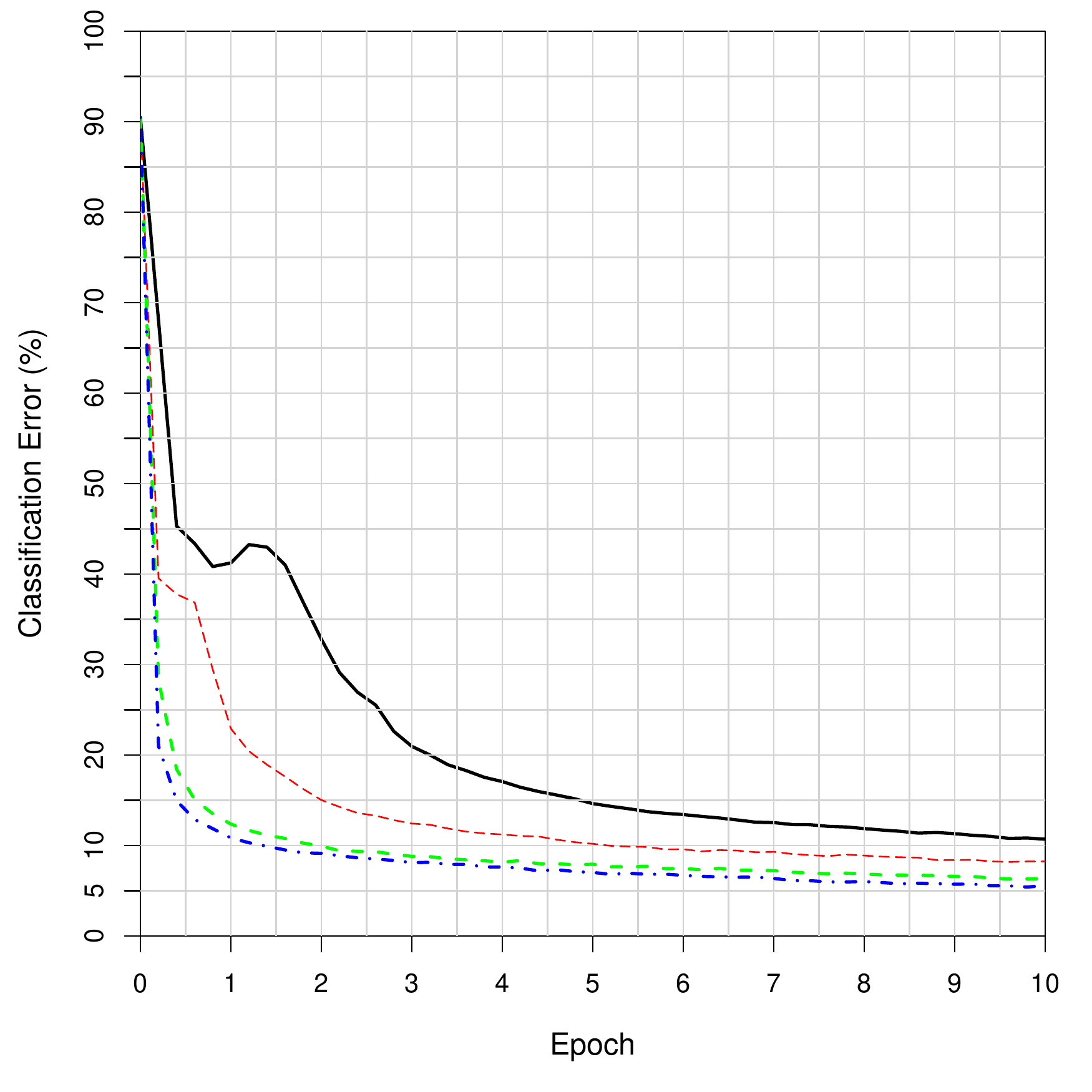}}
\subfigure[10 epochs of \emph{LRA-diff}  (Fashion MNIST).]{\label{fig:lra1_fmnist}\includegraphics[width=63mm]{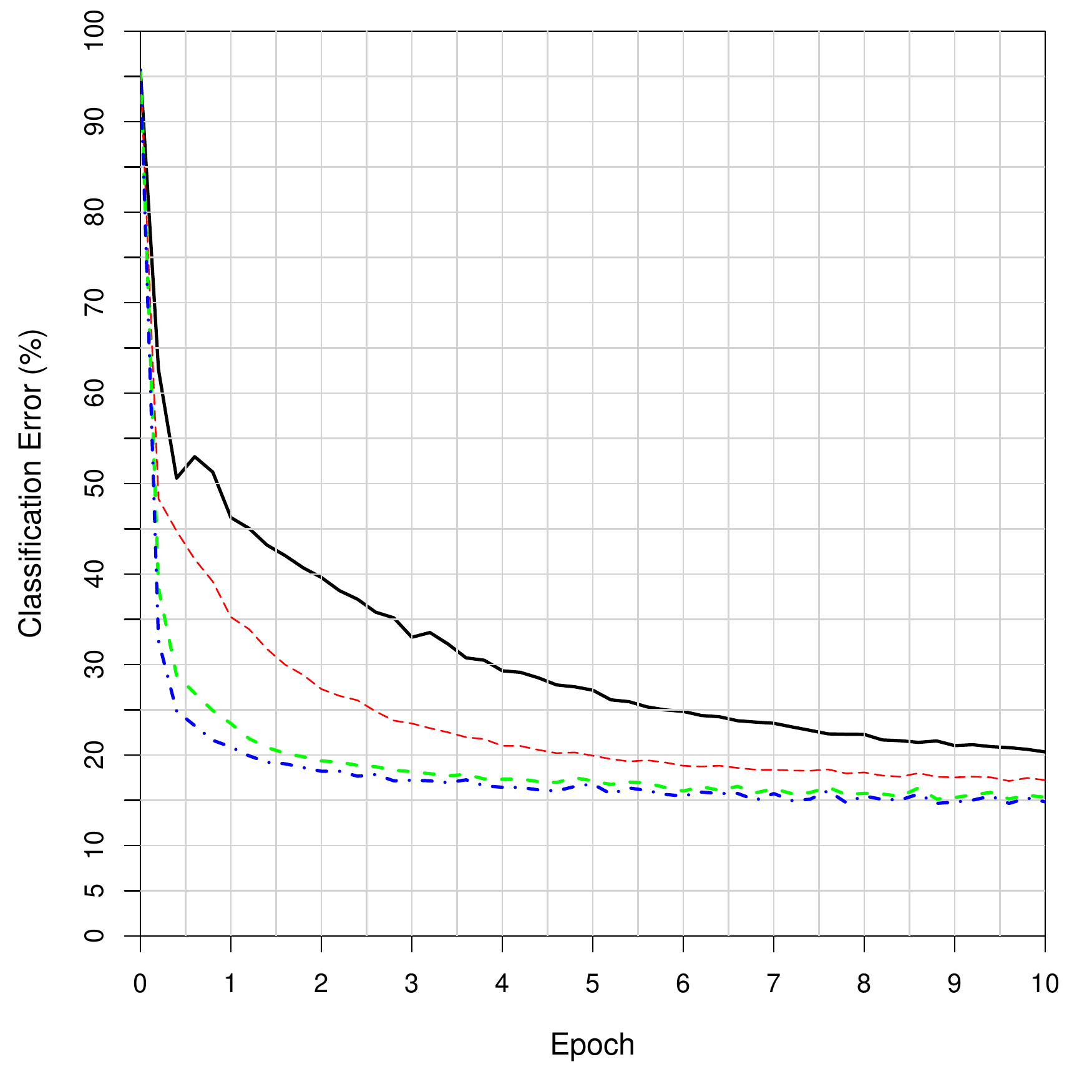}}
\subfigure{\includegraphics[width=110mm]{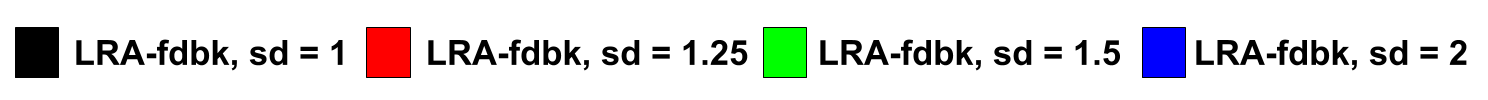}}
\subfigure[10 epochs of \emph{LRA-fdbk} (MNIST).]{\includegraphics[width=63mm]{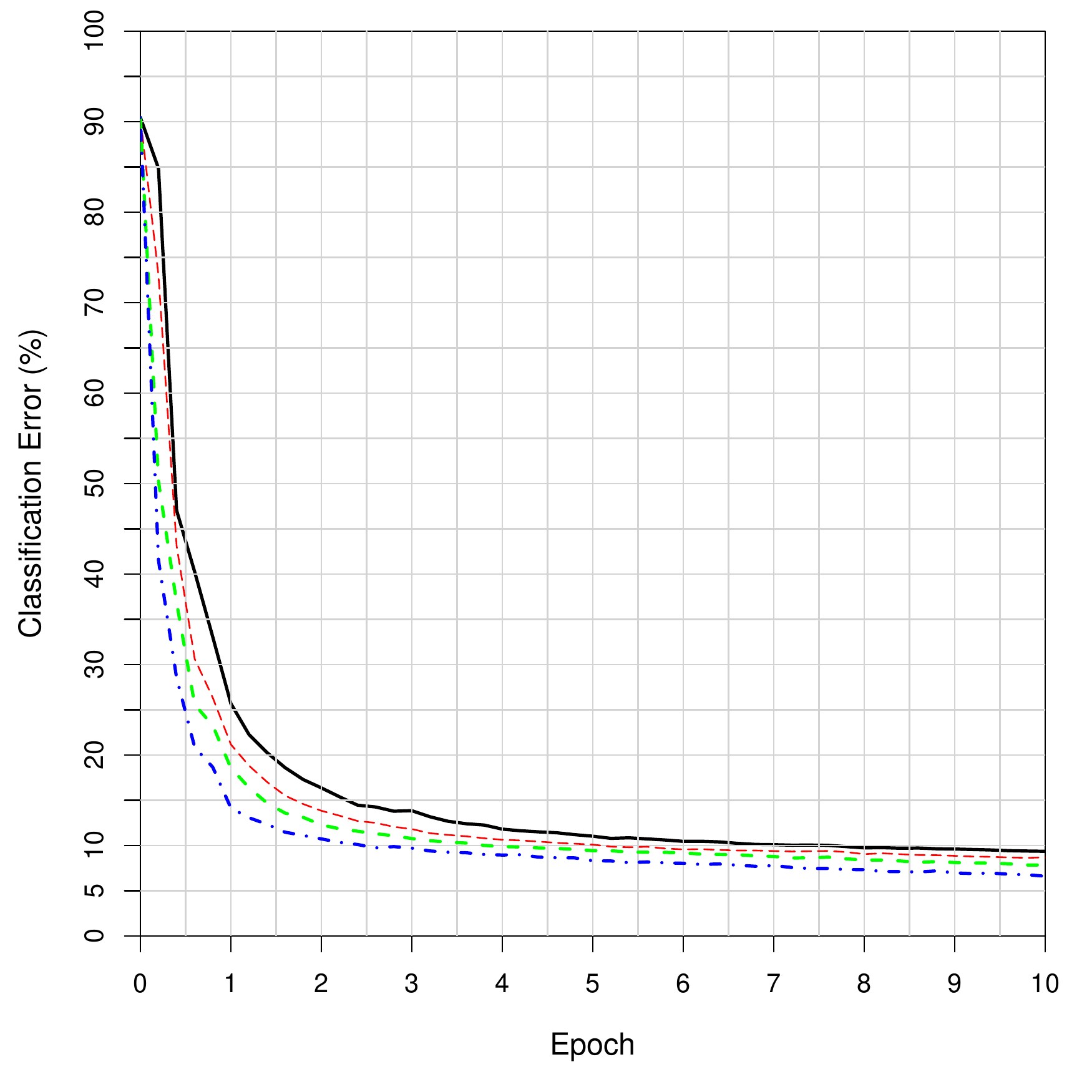}}
\subfigure[10 epochs of \emph{LRA-fdbk}  (Fashion MNIST).]{\label{fig:lra2_fmnist}\includegraphics[width=63mm]{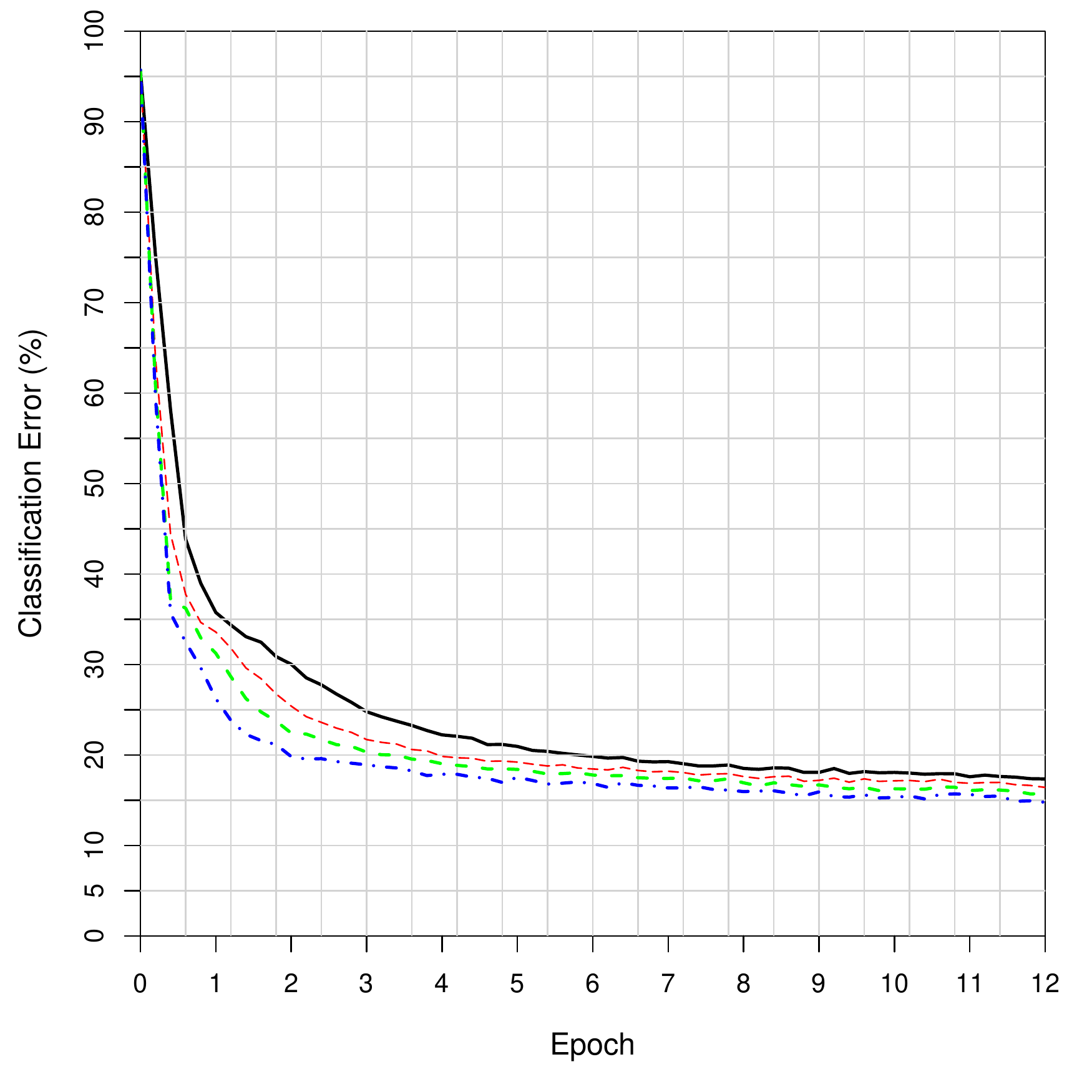}}
\caption{Drop in test-set error during first few epochs of training tanh networks on the MNIST and Fashion MNIST image datasets.}
\label{results:lra_diff_v_fdbk}
\end{figure*}

\begin{figure*}
\centering     %%% not \center
\subfigure[Backprop acquired filters.]{\label{fig:filter_bp}\includegraphics[width=48mm]{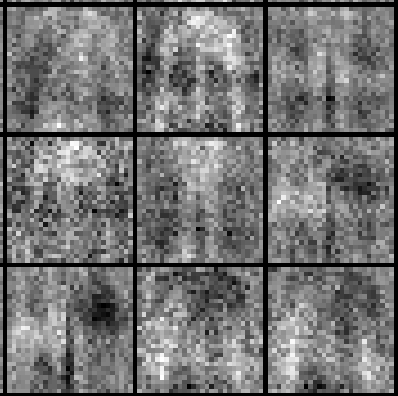}}
\subfigure[LRA-diff acquired filters.]{\label{fig:filter_lra1}\includegraphics[width=48mm]{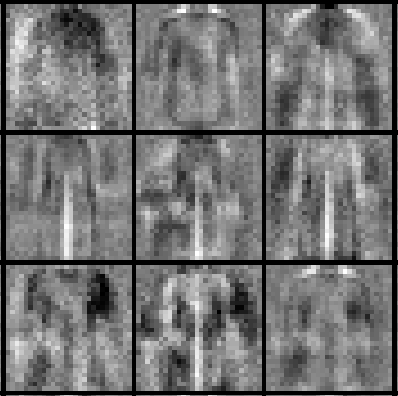}}
\subfigure[LRA-fdbk acquired filters.]{\label{fig:filter_lra2}\includegraphics[width=48mm]{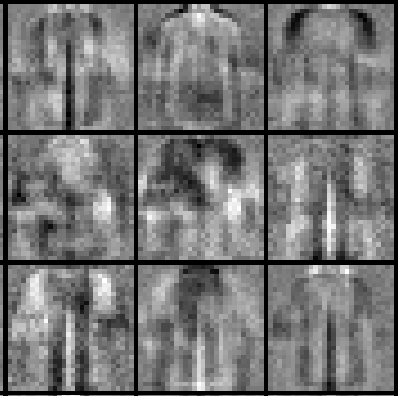}}
\caption{Third-level filters acquired after 1 epoch under each learning algorithm: (a) backprop, (b) LRA-diff, (c) LRA-fdbk.}
\label{results:lra_filters}
\end{figure*}

\paragraph{MNIST:}
\label{data:mnist} 
The popular MNIST dataset \citep{lecun1998gradient} contains $28x28$ grey-scale pixel images, each belonging to one of 10 digit categories. There are 60,000 training images, from which we create a validation subset of 10,000 images, and 10,000 test images. 

\paragraph{Fashion MNIST:}
\label{data:fmnist}
Fashion MNIST \citep{xiao2017fashion} is a dataset composed of $28x28$ grey-scale images of clothing items, meant to serve as a much more difficult drop-in replacement for MNIST itself. The size and structure of the training and testing splits are the same as in MNIST and each image is associated with one of 10 classes. We create a validation set of 10,000 samples from the training split via random sampling without replacement.

\subsection{Effect of Manifold-Walking}
\label{exp1:manifold_walk}
We first investigate how \emph{LRA-diff}'s computation budget $K$, specifically the number of sub-optimization steps allocated per subgraph, affects its ability to train a highly nonlinear network. Furthermore, we contrast this procedure against the vastly simpler error-feedback variant, \emph{LRA-fdbk} (which uses $K=1$ so is also much faster). To do so, we construct networks of three hidden layers of 64 hyperbolic tangent units with biases initialized from zero and weights according to the following classical heuristic:
\begin{align}
W_{ij} \sim U \bigg [ -\frac{1}{\sqrt[]{n_{in}}}, \frac{1}{\sqrt[]{n_{in}}} \bigg ], \label{eq:uniform_fan_in}
\end{align}
noting that $n_{in}$ is the size of previous/incoming layer of post-activities or the number of columns of the weight matrix $W$ (if working in column-major form). $U[-a,a]$ is the uniform distribution in the interval $(-a,a)$.  

For \emph{LRA-diff}, we varied the computation budget $K = \{5, 10, 30, 50\}$, and for \emph{LRA-fdbk}, we fixed $K=1$. The entries of the feedback matrix for LRA-fdbk were generated independently from a Gaussian with standard deviation $\sigma_{E}$. We tried the following settings of $\sigma_{E} = \{1.0, 1.25, 1.5, 2.0\}$ (note in the plots this is denoted \emph{sd}). Both variants of LRA employed the Cauchy loss (see Equation \ref{cauchy_loss}) as the metric for measuring discrepancy between representation and target. Networks were trained over 100 epochs but we only show the first 5 epochs, since roughly after this point, the differences in generalization rates were too similar to warrant visualization. % replace "sd" with "sigma_E" in the plots later...

In Figure \ref{results:lra_diff_v_fdbk}, for \emph{LRA-diff}, we see that increasing $K$ leads to ultimately better generalization and sooner on both MNIST and Fashion MNIST. However, there is a diminishing return as one dramatically increases the number of steps from $K = 30$ to $K = 50$. As $K$ is the number of iterations of the inner loop in Algorithm \ref{algo:lra}, increasing $K$ leads to significant slowdown. However, more importantly, \emph{LRA-fdbk}, which uses $K=1$, is therefore a far faster variation of LRA and it
%only requires that the output post and pre-activation functions of any subgraph are differentiable, 
reaches the same level of generalization. This means \emph{LRA-fdbk} is able to use the short-circuit feedback connections to create a useful displacement for the model's current input representation to help lower its local loss. 
%While, like \emph{LRA-diff}, multiple steps can also be taken using \emph{LRA-fdbk} within any subgraph, we found that this was largely unnecessary in preliminary experiments, and taking a single step along the layerwise manifold worked well across all settings we tried. 
We found that the initialization of the error feedback weights affects \emph{LRA-fdbk}'s performance, though as one raises the standard deviation, the impact is far less severe than varying the number of steps in \emph{LRA-diff}.\footnote{Choosing a value for the standard deviation that is too low, especially below one, however, can slow down the learning process. We found that naively using a standard deviation of one worked quite well for our preliminary experiments and thus did no further tuning after Experiment \ref{exp1:manifold_walk}.}

In Figure \ref{results:lra_filters}, we show the filters acquired by the feedforward network on Fashion MNIST after a single epoch. To create these filter visualizations, we employ the feature activation maximization approach as presented by \cite{erhan2010understanding}. Furthermore, while this approach generally only applies to the first hidden layer of units, which sit closest to the input pixel nodes, we can apply the same technique to the upper hidden layers of the network, such as the third layer, by simply ignoring the nonlinearity at each level of the model. Thus, we approximately ``linearize'' the nonlinear network which allows us to collapse successive weight matrices back into a single matrix (taking advantage of this natural property of deep linear networks). This will incur some minor approximation error, since the network is not truly linear, but we found that this approximation gave us a very fast and reasonably good picture of what knowledge might be captured in the synaptic connections that form the memories of the upper layer nodes, i.e., those closest to the output layer.
Observe that both versions of LRA (note that for \emph{LRA-diff} we use the network trained with $K = 30$) learn reasonably good and clear filters after just one pass through the data. Back-propagation, however, learns far noisier filters. Again, it is surprising to see that \emph{LRA-fdbk} learn so well with only one pass through the data, given that it is far cheaper computationally than \emph{LRA-diff}. Encouraged by this positive behavior, its low computation requirements, and the fact that \emph{LRA-fdbk} can also handle a far greater range of activations, such as discrete-valued and stochastic ones, we focus on \emph{LRA-fdbk} (with $K=1$, so it only makes once pass through its inner loop, making its computational requirements comparable to back-prop) for the rest of the paper.
 
\begin{figure*}
\centering     %%% not \center
\subfigure{\includegraphics[width=50mm]{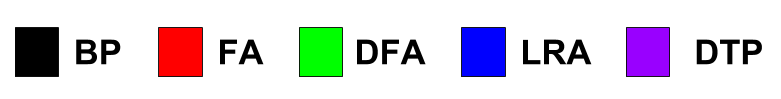}} \\
\subfigure[MNIST, sigmoid network.]{\includegraphics[width=75mm]{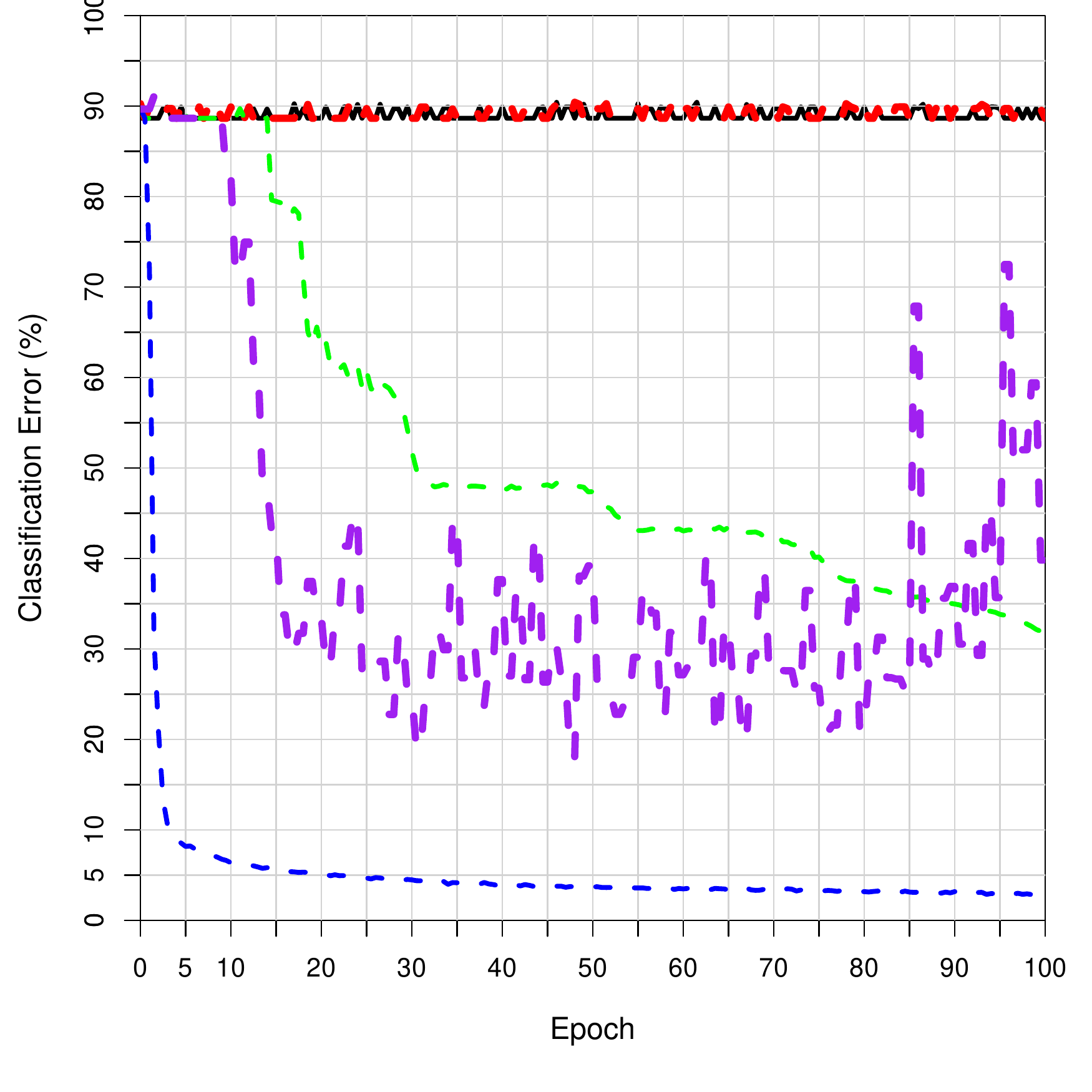}}
\subfigure[Fashion MNIST, sigmoid network.]{\includegraphics[width=75mm]{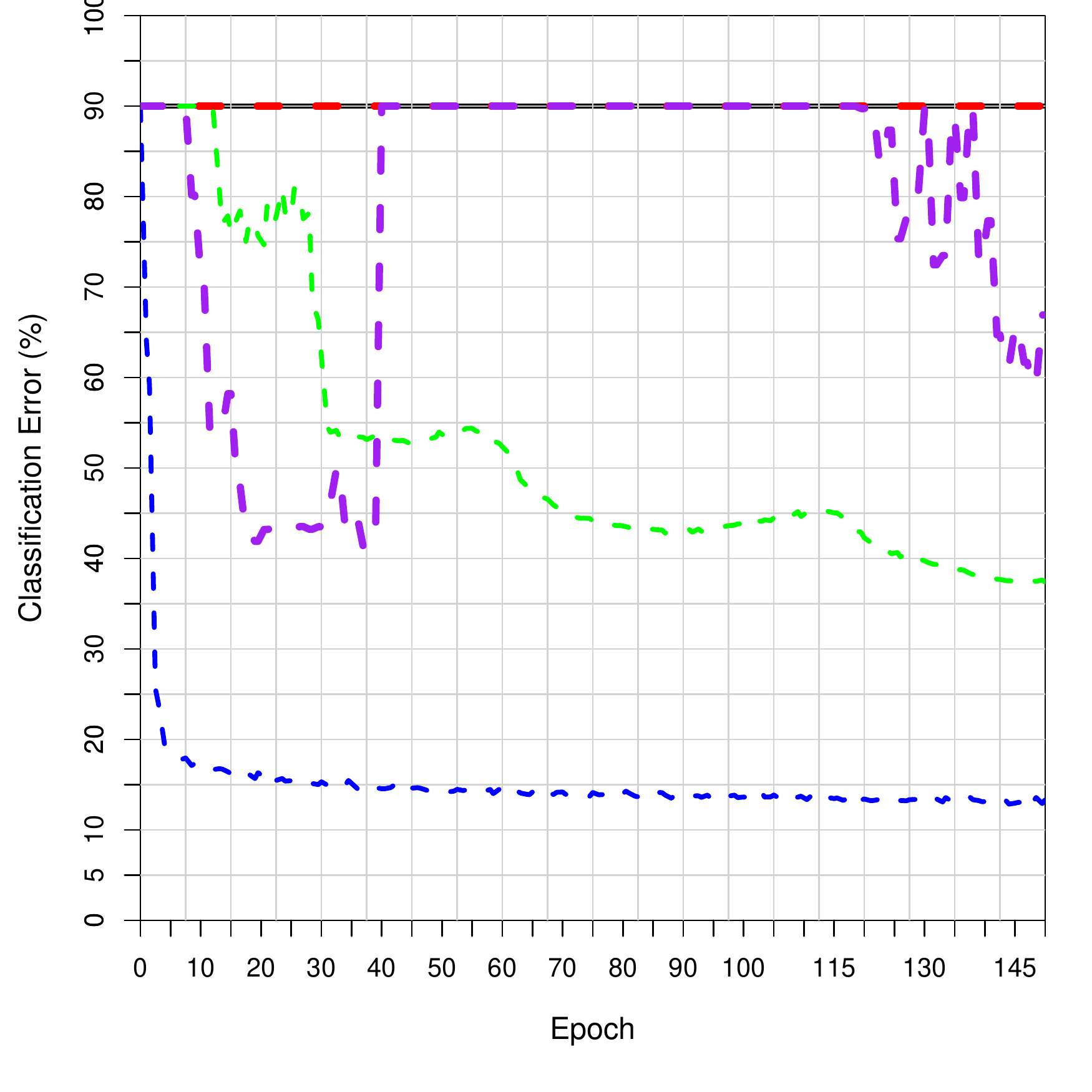}}
\subfigure[MNIST, tanh network.]{\includegraphics[width=75mm]{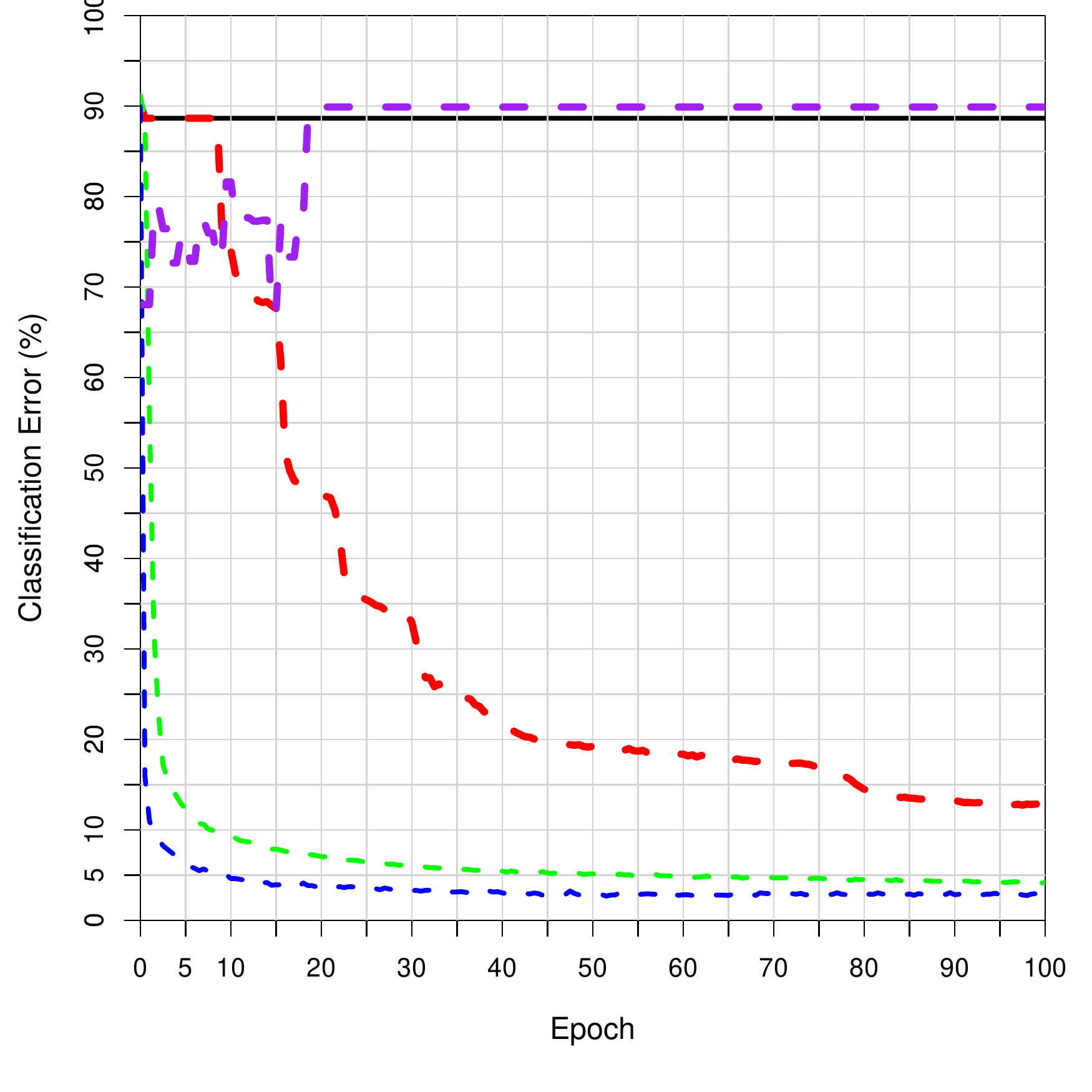}}
\subfigure[Fashion MNIST, tanh network.]{\includegraphics[width=75mm]{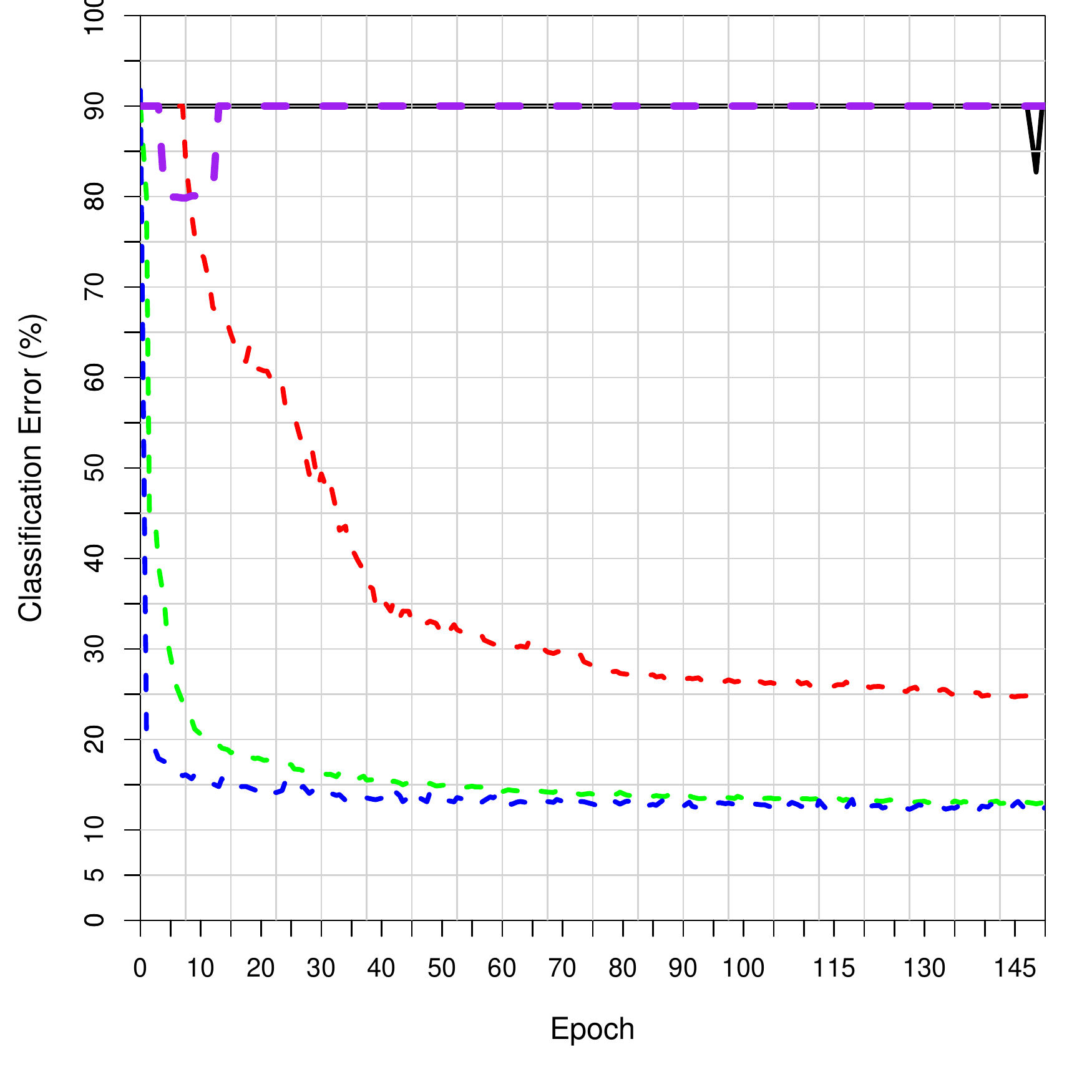}}
\caption{Generalization of various learning algorithms used to train deep and thin networks of either sigmoid (top) or tanh (bottom) units on MNIST (left) and Fashion MNIST (right). Initial weights were sampled: $w_{i,j} \sim \mathcal{N}(\mu = 0,\,\sigma^{2} = 0.025)$. Note that DTP, FA, and BP often fail or are unstable with this initialization, especially when training deep sigmoidal networks on Fashion MNIST.}
\label{results:lra_generalization}
\end{figure*}

\begin{figure}[t]
\centering     %%% not \center
\subfigure{\includegraphics[width=60mm]{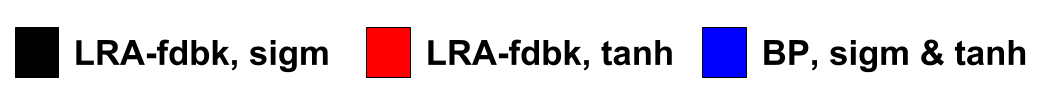}}\\
\subfigure[MNIST mean depth.]{\includegraphics[width=60mm]{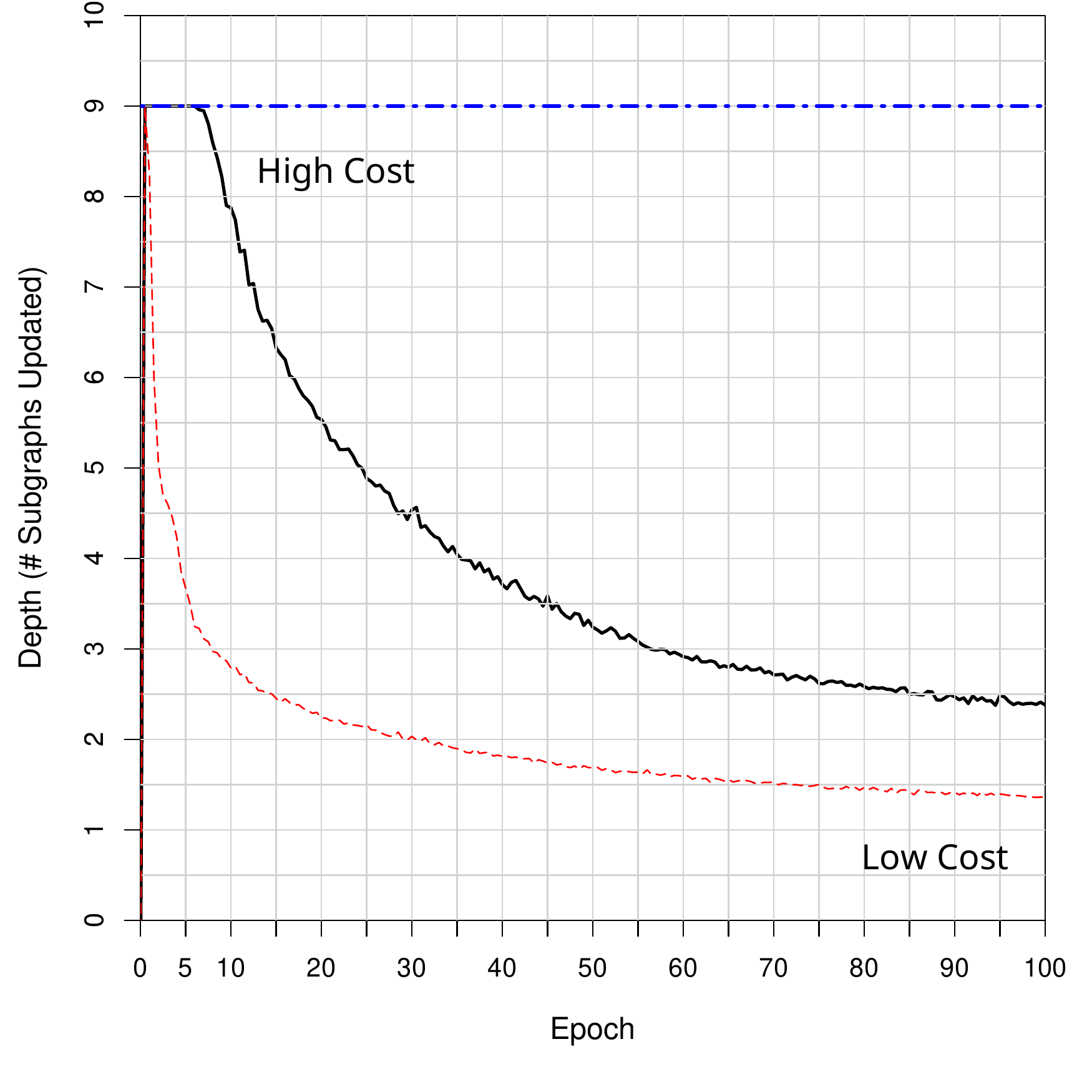}}
\subfigure[Fashion MNIST mean depth.]
{\includegraphics[width=60mm]{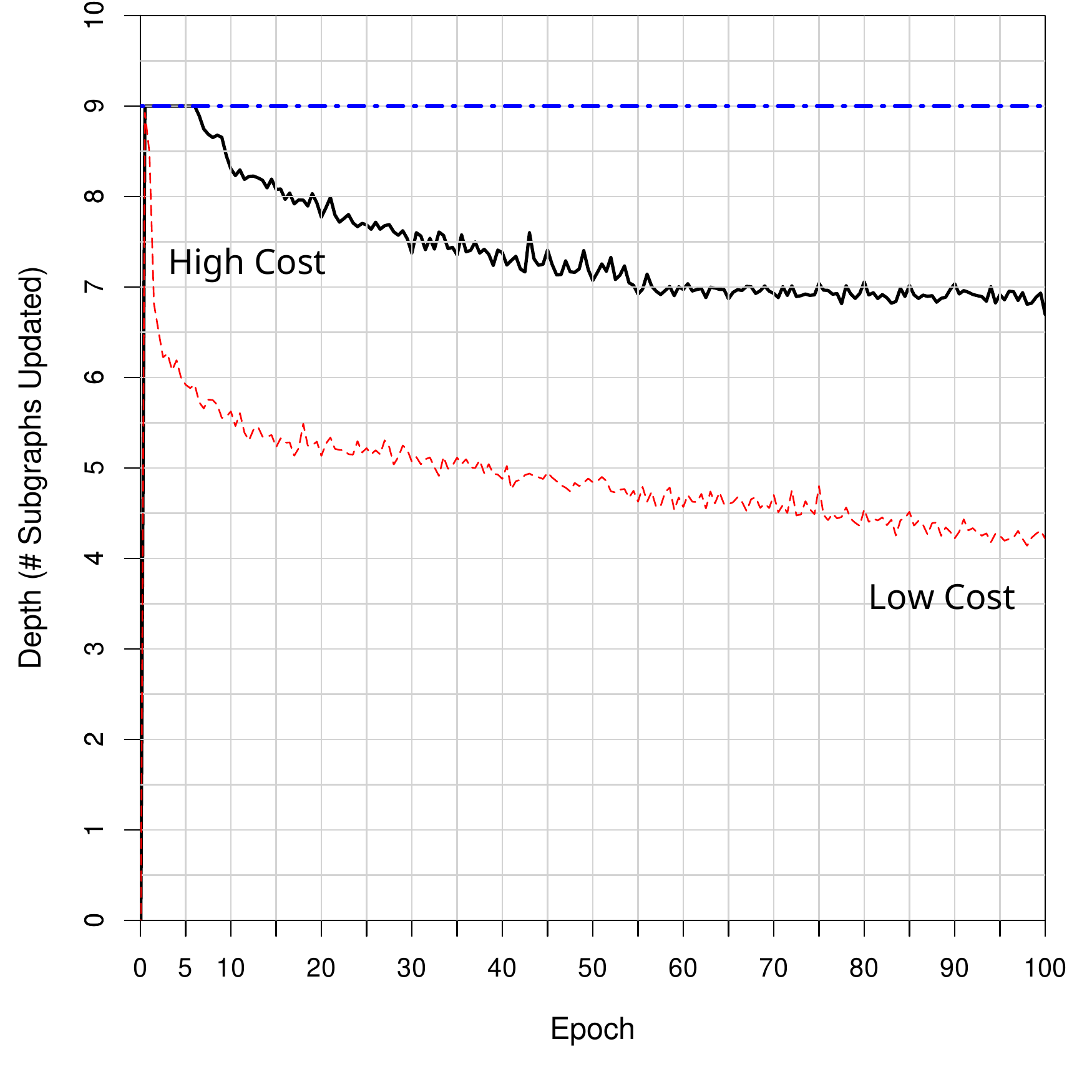}}
\caption{Average depth of the credit assignment carried out by LRA on w.r.t. deep sigmoidal and tanh networks (if the backwards pass is stopped whenever the feed-forward activations are similar to the targets), compared against backprop's fixed depth.}
\label{results:lra_depth}
\vspace{-0.5in}
\end{figure}

\begin{table}[t]
\begin{center}
\caption{Robustness to poor initialization. Training and generalization error (\%) of deep sigmoid or tanh networks, initialized with 0-mean Gaussians with std of $0$, $0.025$, $0.05$, and $0.1$}
\label{results:lra_deep_nets}
\footnotesize
\begin{tabular}{|c||c|c|c|c|c|c|c|c|}
\hline
  & \multicolumn{8}{c|}{\textbf{MNIST}}  \\
  & \multicolumn{4}{c|}{$\phi(\cdot) =$ \textbf{Sigmoid}} & \multicolumn{4}{c|}{$\phi(\cdot) =$ \textbf{Tanh}}  \\
  \textbf{Algorithm} & \textbf{std}: $0$ & $0.025$ & $0.05$ & $0.1$ & $0$ & $0.025$ & $0.05$ & $0.1$ \\
  \hline
  Backprop & X & 88.66 & 88.66 & 88.66 & X & 88.65 & 2.67 & 2.48 \\
  FA &  88.65 & 88.65 & 88.65 &  88.65 & 16.1 & 12.73 & 11.86 & 11.81 \\
  DFA & 30.84 & 31.78 & 31.77 & 26.81 & 5.24 & 4.11 & 5.12 & 4.9 \\
  DTP & X & 18.06 & 16.83 & 17.06 & X & 67.51 & 5.38 & 3.19  \\
  LRA-fdbk & 2.75 & 2.85 & 2.89 & 2.96 & 2.69 & 2.87 & 2.82 & 3.72 \\
  \hline
\end{tabular}

\begin{tabular}{|c||c|c|c|c|c|c|c|c|}
\hline
  & \multicolumn{8}{c|}{\textbf{Fashion MNIST}}  \\
  & \multicolumn{4}{c|}{$\phi(\cdot) =$ \textbf{Sigmoid}} & \multicolumn{4}{c|}{$\phi(\cdot) =$ \textbf{Tanh}}  \\
  \textbf{Algorithm} & \textbf{std}: $0$ & $0.025$ & $0.05$ & $0.1$ & $0$ & $0.025$ & $0.05$ & $0.1$ \\
  \hline
  Backprop & X & 90.00 & 90.00 & 90.00 & X & 29.04 & 11.48 & 11.38 \\
  FA & 90.00 & 90.00 & 90.00 & 89.97 & 24.82 & 24.57 & 21.67 & 20.01  \\
  DFA & 35.91 & 37.4 & 41.62 & 37.7 & 13.92 & 12.82 & 13.95 & 14.13  \\
  DTP & X & 41.1 & 26.26 & 24.69 & X & 79.84 & 60.19 & 13.59 \\
  LRA-fdbk & 13.27 & 12.85 & 12.94 & 12.96 & 11.85 & 12.14 & 12.62 & 12.65 \\
  \hline
\end{tabular}
\end{center}
\end{table}

\subsection{Robustness to Initialization}
\label{exp2a:init}
It is very difficult to train deep (and thin) networks from simplistic initialization schemes \citep{romero2014fitnets}. Furthermore, \cite{lecun1998efficient} showed that using the logistic sigmoid as an activation function can slow down learning considerably, largely due to its non-zero mean, which was further investigated by \cite{glorot2010understanding}. Given the problems that come with unit saturation and vanishing gradients \citep{bengio1994learning}, training a very deep and thin network, especially composed of logistic sigmoid units, with only back-propagation, can be very difficult.

To investigate LRA's robustness to poor initialization, we use it and competing methods to train deep nonlinear networks consisting of either logistic sigmoid or hyperbolic tangent activation functions with model parameters (i.e. network weights) initialized from a parametrized, zero-mean Gaussian distribution. The Gaussian distribution is a very simple, common way to initialize the parameters of a neural model, and controlling its standard deviation, $\sigma$, allows us to probe different cases when back-propagation fails. In this experiment, we investigate the settings $\sigma = \{ 0.025, 0.05, 0.1 \}$, and compare LRA and backprop (\emph{BP}) to algorithms such as Difference Target Propagation (\emph{DTP}) from \cite{lee2015targetprop}, Feedback Alignment (\emph{FA}) from \cite{lillicrap2016random}, and Direct Feedback Alignment (\emph{DFA}) from \cite{nokland2016direct}. Furthermore, we also show the situation where the weights are simply initialized to zero. Whenever it was not possible to learn from zero with a given algorithm, such as BP and DTP, we simply marked the appropriate slot with an $X$. To initialize the feedback weights of DFA and FA, we follow the protocol prescribed by \cite{nokland2016direct}.

The network architecture each algorithm is responsible for training is the same: a multilayer perceptron containing eight hidden layers of 128 processing elements. We examine the ability of each algorithm to train the same architecture employing logistic sigmoid (a non-zero mean activation function) and the hyperbolic tangent (a zero-mean activation function).

In Table \ref{results:lra_deep_nets}, we present the best found generalization error rate for the deep architecture learned with each algorithm under each initialization setting. Observe that \emph{LRA-fdbk} is rather robust to the initializations scheme, and more importantly, is able to train to good generalization regardless of which unit type is used. Furthermore, even at initializations close to or at zero, LRA-fdbk is able to train deep networks of both logistic sigmoid and hyperbolic tangent networks.  In Figure \ref{results:lra_generalization}, we focus on the setting with $\sigma = 0.025$ for the Gaussian distribution used to initialize the networks (i.e. the lowest non-zero setting). Observe that despite poor initialization, even within 10 epochs, \emph{LRA-fdbk} is able to consistently reach good classification error ($\approx 5\%$ on MNIST and $< 20\%$ on Fashion MNIST), while the other methods struggle to reach those numbers even after 100 epochs. DFA is competitive with \emph{LRA-fdbk} when using hyperbolic tangent units but trains the same network composed of sigmoid units poorly.

According to our results, for DTP, the inverse mapping used to reconstruct the underlying layerwise targets does not work all that well when weights are initialized from a purely random Gaussian, especially with a low standard deviation. As observed in Table \ref{results:lra_deep_nets}, DTP struggles to train these deep networks, even when given the advantage and allowed to use an adaptive learning rate unlike the other algorithms. DTP's struggle might be the result of losing too much of its layerwise target information too soon--the inverse mapping (or decoder of the layerwise auto-associative structure) requires a strong signal at each iteration to learn and if the signal is too weak or lost, the target produced for reconstruction becomes rather useless. However, DTP does far better than FA across the scenarios, although its lacks the ability to train from zero initialization. Our preliminary experimentation with DTP also uncovered that, in addition to requiring a more complex outer optimization procedure (like RMSprop) to achieve decent results, the learning procedure is highly dependent on its conditions and internal hyper-parameter settings (and there exist few heuristics on good starting points). To make DTP work well, significant tweaking of its settings would be required on a per-dataset/per-architecture basis in order to improve targets for the inverse mapping. Since the error of DTP (or target propagation algorithms in general) is represented as the change in activities of the same set of neurons, if any neural activity is unstable, the overall algorithm will fail to train the underlying model effectively. Furthermore, because of the extra computation involved in DTP, it is also far slower than LRA.

With respect to DFA, the layers are no longer related through a sequential backward pathway. This means that the lower-level neurons are disconnected from the forward propagation pathway when errors are calculated using the feedback projection weights. In contrast, we find that in FA the error signal is still created by a backward pass as in BP, but this time with the final per-neuron derivatives approximated by the feedback weights that replace the transpose of the forward weights in the BP global feedback pathway. Hence FA fails in cases where we have instability or few gradients are acting on participating neurons. DFA actually works fairly well compared to the other baselines if the activation function is the hyperbolic tangent, and does outperform FA when the logistic sigmoid is utilized. Furthermore, unlike DTP and BP, DFA and FA can train networks from zero, although LRA does a much better job.

Another interesting and important property of \emph{LRA-fdbk} is that, unlike all of the other approaches investigated, it can \textbf{automatically decide its depth of credit assignment}. Specifically, in the case of the MLP, \emph{LRA-fdbk} can decide how many subgraphs it needs to update. This is possible because of the condition in the while loop that stops the backwards pass if the feedforward activations are already similar to the targets (i.e. have a local loss at most $\epsilon$). This condition can be used at the mini-batch level or at the per-sample level. %criterion/gate we proposed earlier in Algorithm \ref{algo:lra}. 
In Figure \ref{results:lra_depth}, we observe this dynamic behavior when recording the mean depth (or average number of subgraphs updated over the full training set in one pass), seeing that the network starts, within the first several epochs, by updating all of the subgraphs of the network. However, as learning continues, usually past five epochs, we see the number of subgraphs updated decrease, and, in the case of the tanh networks, approach one or zero. While not presented in the depth plots, we also recorded the average number of updates made at each layer. These logs revealed that, around the same time the average depth approaches one, even updates at the very top subgraph become less frequent. This aligns with our intuition that once latent representations of a lower layer are ``good enough'', LRA can quit expending computation on that layer, on a per-sample basis. No other algorithm, including BP, has the property to adapt its computation when calculating parameter displacements (which is also why BP is depicted as horizontal line in Figure \ref{results:lra_depth}--its cost is the same over each epoch).

\subsection{Training from Null Initialization}
\label{exp2b:null_init}
Next, we further experiment with LRA's ability to train networks from null, or pure zero, initialization. While BP and DTP will fail in this setting,
%the proposed LRA will not, since, given its short-circuit pathway, it does not require the statistics of the internal subgraph to even exist. 
DFA and FA will not. However, in order for DFA and FA to operate in this special setting certain restrictions are needed, e.g. the activation function must be specific like the hyperbolic tangent \cite{nokland2016direct} and non-zero initialization must be used for certain activations including the linear rectifier. LRA does not impose these restrictions, and furthermore, can easily handle non-differentiable operations, e.g., the signum function, as we shall observe shortly. % or stochastic binary sampling

\begin{table*}[t]
\begin{center}
\caption{Generalization error (\%) of various networks trained with LRA-fdbk, from null initialization. We also report, next to each ``best error'', the end of which epoch the model reached this level of generalization. Note that, for the output of the internal subgraphs, \emph{LWTA}-1 means tanh was used and \emph{LWTA}-2 means hard-tanh was used.}
\label{results:lra_nets}
\footnotesize
\begin{tabular}{|c||c|c|c|c||c|c|c|c|}
\hline
  & \multicolumn{4}{c||}{\textbf{MNIST}} & \multicolumn{4}{|c|}{\textbf{Fashion MNIST}} \\
  $\phi(\cdot)$ & \textbf{Ep. 1} & \textbf{Ep. 50} & \textbf{Ep. 100} & \textbf{Best} & \textbf{Ep. 1} & \textbf{Ep. 50} & \textbf{Ep. 100} & \textbf{Best} \\
  \hline
  LRA-\emph{Softsign} & 5.94 & 2.13 & 2.32 & 1.86 (25) & 17.22 & 13.47 & 13.2 & 12.74 (30) \\
  LRA-\emph{Relu6} & 8.6 & 3.14 & 2.58 & 2.18 (69) & 19.7 & 11.52 & 12.46 & 11.18 (98) \\
  LRA-\emph{Softplus} & 14.94 & 4.1 & 3.53 & 2.96 (94) & 31.35 & 15.94 & 15.46 & 13.37 (79) \\
   %LRA-\emph{Lat-ReLU} & -- & -- & -- & -- & -- & -- & -- & -- \\
   LRA-\emph{LWTA}-1 & 7.02 & 2.96 & 2.68 & 2.37 (67) & 18.25 & 12.72 & 12.52 & 12.33 (72) \\
   LRA-\emph{LWTA}-2 & 6.96 & 2.57 & 2.88 & 2.48 (44) & 17.89 & 12.85 & 12.76 & 12.05 (54) \\
   LRA-\emph{SLWTA}-1 & 9.15 & 2.1 & 2.24 & 1.78 (33) & 20.15 & 11.78 & 11.47 & 11.03 (96) \\
   LRA-\emph{SLWTA}-2 & 8.99 & 2.21 & 2.1 & 1.91 (33) & 20.00 & 11.74 & 11.58 & 11.29 (88) \\
   LRA-\emph{Signum} & 6.88 & 2.09 & 2.07 & 1.84 (45) & 19.23 & 12.48 & 13.06 & 12.06 (65) \\
  \hline
\end{tabular}
\end{center}
\end{table*}

% LRA can train a variety of units, including those are better at modeling properties found in real neurobiological systems, such as stochastic binary units (in spiking nets) or lateral competition units (as in sparse/predictive coding frameworks)
To demonstrate LRA's ability to handle a wide variety of functions, we train models of 3 layers of 800 hidden units with updates estimated over mini-batches of 20 samples. Parameters were updated using the Adam  adaptive learning rate \citep{kingma2014adam}. The activation functions we experimented with included the softsign \citep{glorot2010understanding}, the softplus \citep{glorot2011deep}, the linear rectifier \citep{glorot2011deep}, local-winner-take-all (LWTA) lateral competition \citep{srivastava2013compete}, and the signum (or sign).

For the networks that used LWTA and signum units, the architecture for any particular subgraph of the MLP, except the bottommost and topmost subgraphs, is defined as follows (essentially decomposing the activation into two parts, as discussed in Section \ref{non_diff_fun}):
\begin{align}
\mathbf{z}^{\ell-1} = f_{\ell-1}( \mathbf{h}^{\ell-1} ), \quad  
\mathbf{h}^\ell = W_\ell f^d_{\ell-1}( \mathbf{z}^{\ell-1} ) + \mathbf{b}_\ell, \quad \mbox{and} \quad \mathbf{z}^\ell = f_\ell( \mathbf{h}^\ell )
\end{align}
where $f^d_{\ell}$ is a discretization function, e.g., signum or Heaviside step, or lateral competition activation, e.g., LWTA\footnote{Note that the signum is specifically defined as: $sign(v) = \{1 \mbox{ if } v \geq 0 \mbox{ and } 0 \mbox{ if } v < 0\}$.} and $f_\ell$ depends on the network: When using the subgraph above for signum, we set $f_{\ell}$ to be the hyperbolic tangent and in the case of LWTA units, we experimented with both the hyperbolic tangent -- \emph{LWTA-1} -- and the hard hyperbolic tangent \citep{gulcehre2016noisy} -- \emph{LWTA-2}. For the LWTA units, we follow and design lateral competition blocks following the same convention discussed by \cite{srivastava2013compete}. Specifically, any processing element $z_j$, in block $i$ of $n$ units (in layer $\ell$), is defined by the hard interaction function below:
\begin{align}
  z^i_j = 
    \begin{cases} 
      z^i_j & \mbox{if} \ z^i_j \geq z^i_k \ \forall k = 1..n \\
      0 & \mbox{otherwise.}
    \end{cases}
\end{align}
In this experiment, the LWTA blocks we employed grouped four neurons together, with no overlap, yielding 200 blocks of laterally competitive neurons. Index precedence is used to break any ties. This form of structured sparsity through competition has also been observed in a biological neural circuits when modeling brain processes. Specifically, areas of the brain exhibit are structured with neurons providing excitatory feedback to nearby neurons, as evidenced in studies of cortical and sub-cortical regions of the brain \citep{stefanis1969interneuronal,andersen1969participation,eccles2013cerebellum}. The concept of inter-neuronal competition also plays a key role in Bayesian theories of the brain, specifically those that build on (sparse) predictive coding \citep{olshausen1997sparse,rao1999predictive}, which argue that lateral competition allow the underlying system to uncover the few causal factors, out of the many possible, that explain a given input stimulus at any time step. From a practical machine learning perspective, sparsity is highly desirable for a wide variety of reasons \citep{glorot2011deep}, and has recently been shown to be useful in training directed neural generative models of sequential data \citep{ororbia2017learning}.

\begin{figure*}
\centering     %%% not \center
\subfigure[First hidden layer (BP).]{\label{fig:bp_tsne_z1}\includegraphics[width=48mm]{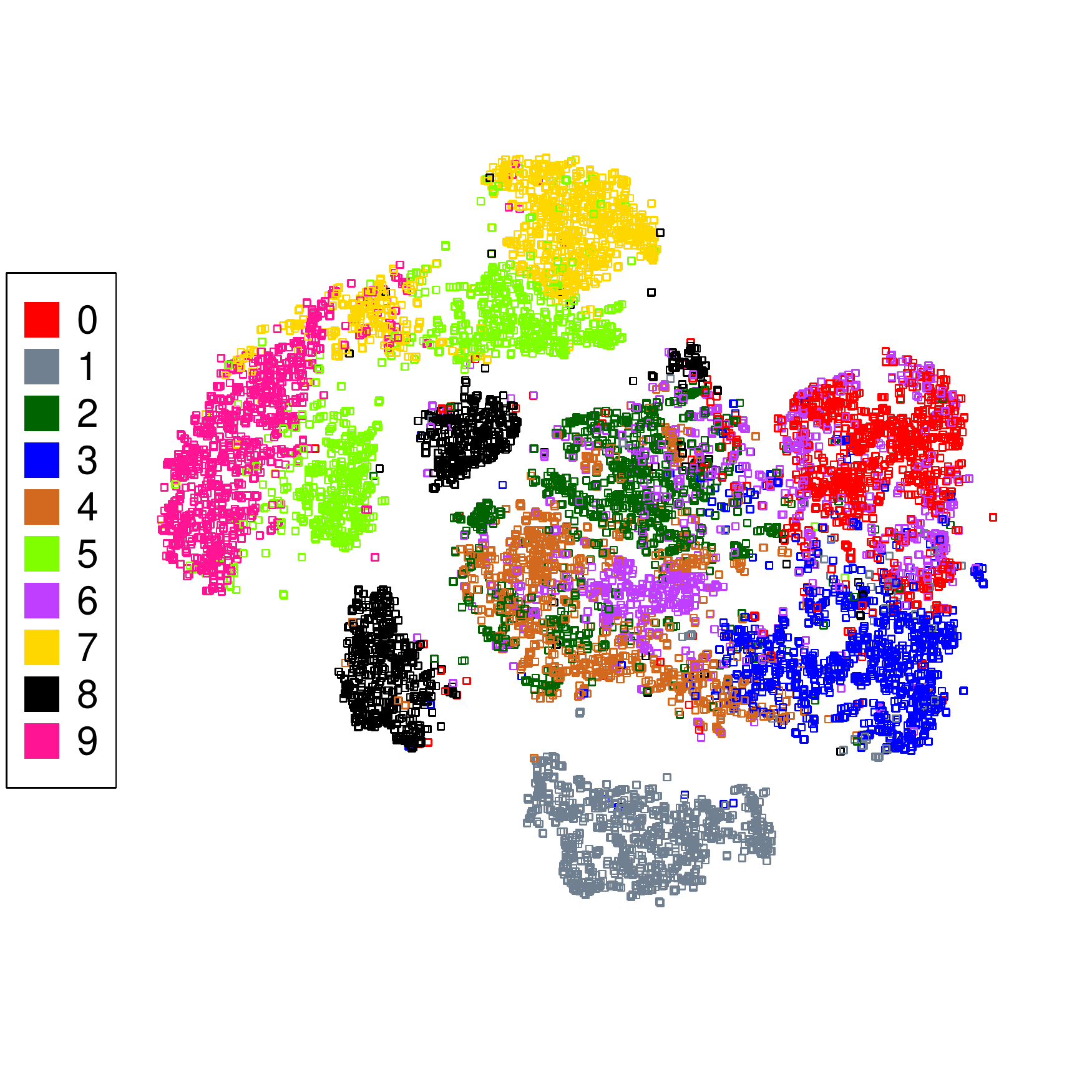}}
\subfigure[Second hidden layer (BP).]{\label{fig:bp_tsne_z2}\includegraphics[width=48mm]{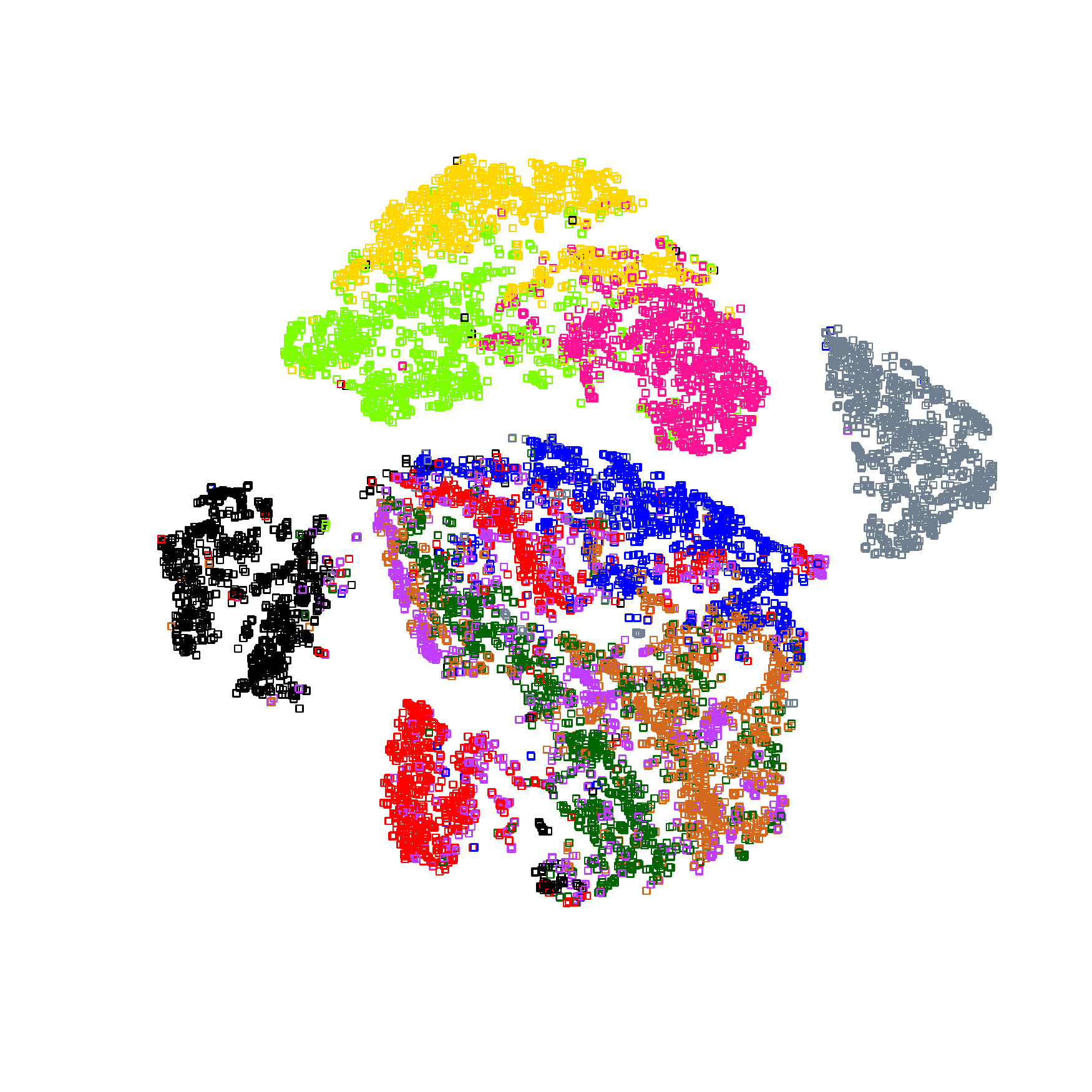}}
\subfigure[Third hidden layer (BP).]{\label{fig:bp_tsne_z3}\includegraphics[width=48mm]{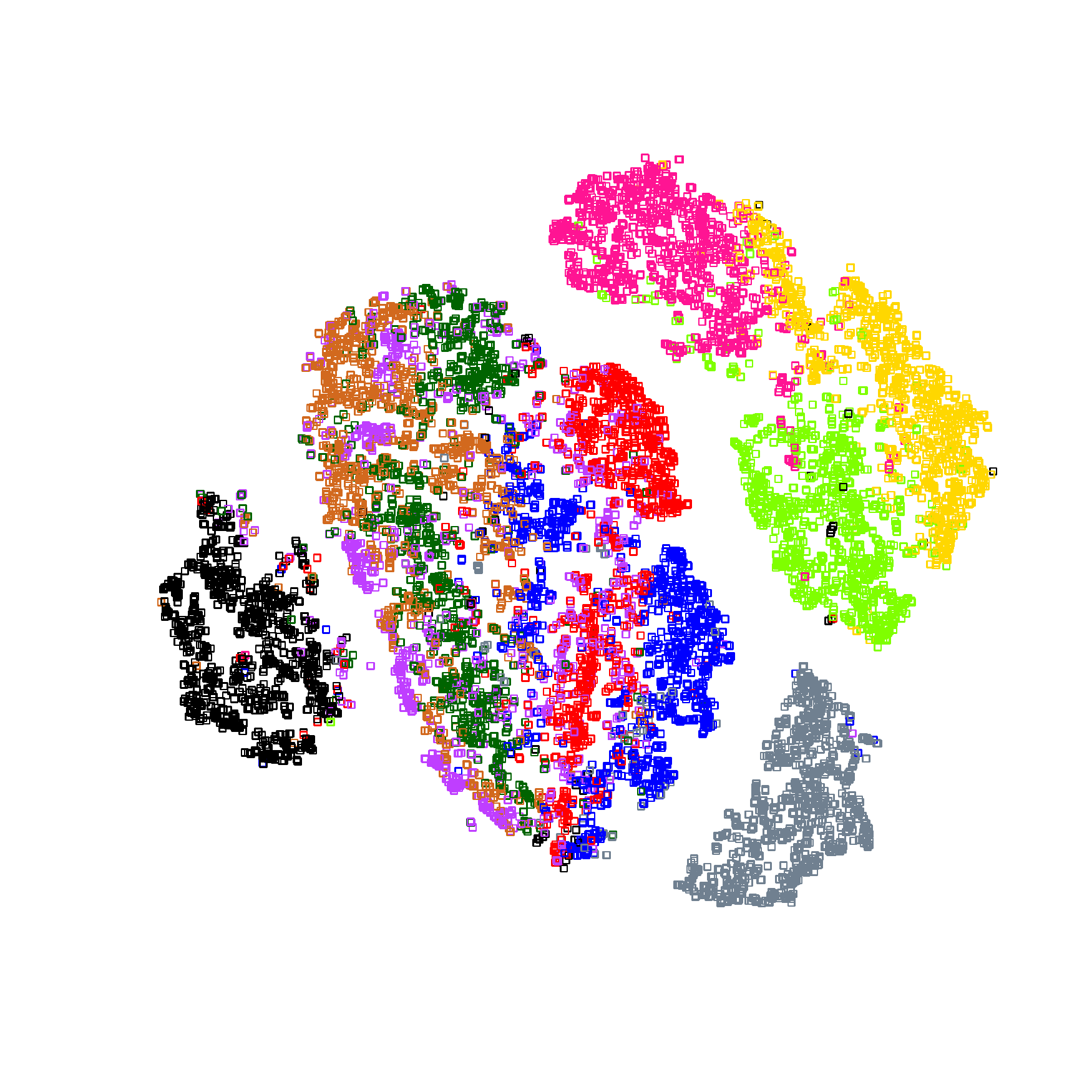}}

\subfigure[First hidden layer (DFA).]{\label{fig:dfa_tsne_z1}\includegraphics[width=48mm]{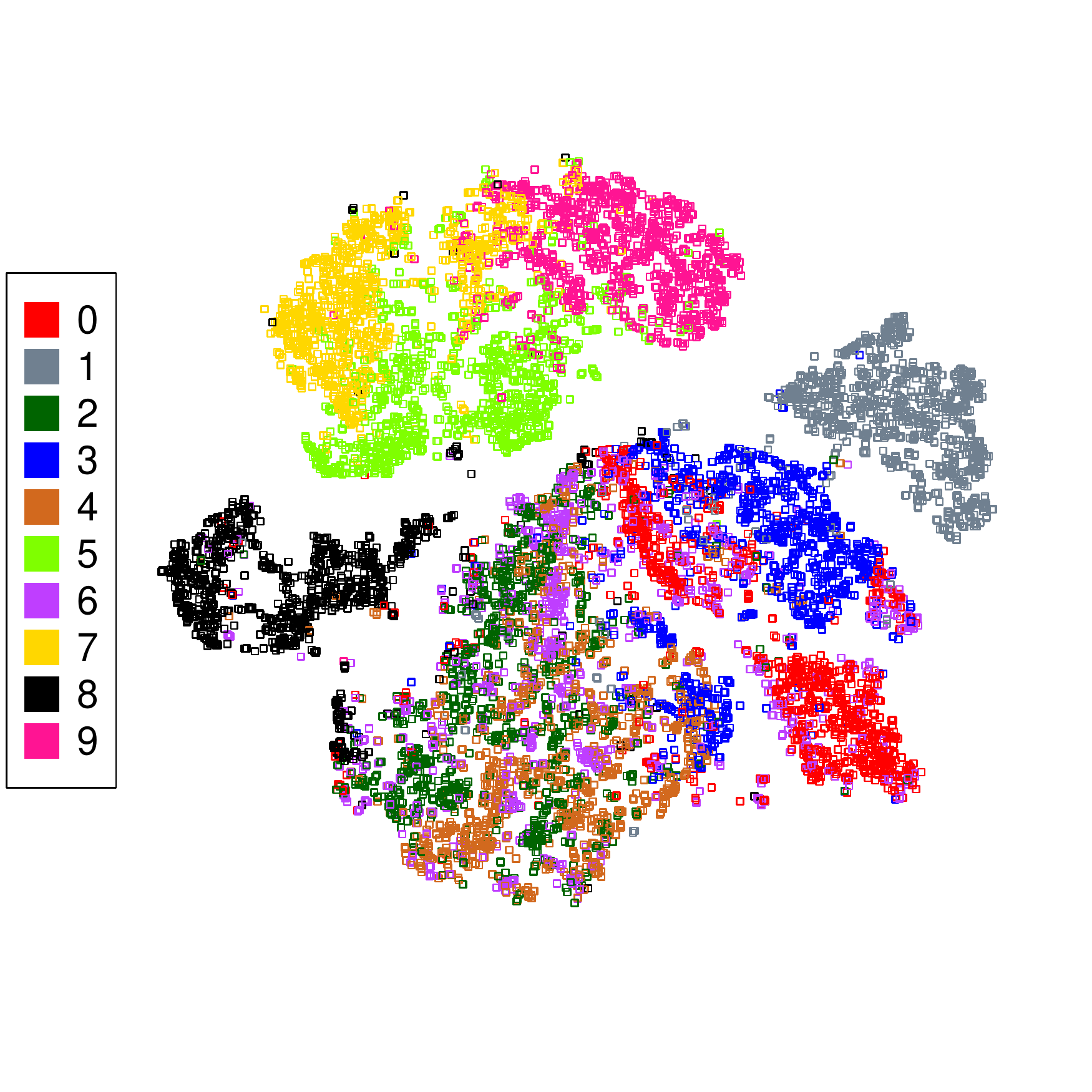}}
\subfigure[Second hidden layer (DFA).]{\label{fig:dfa_tsne_z2}\includegraphics[width=48mm]{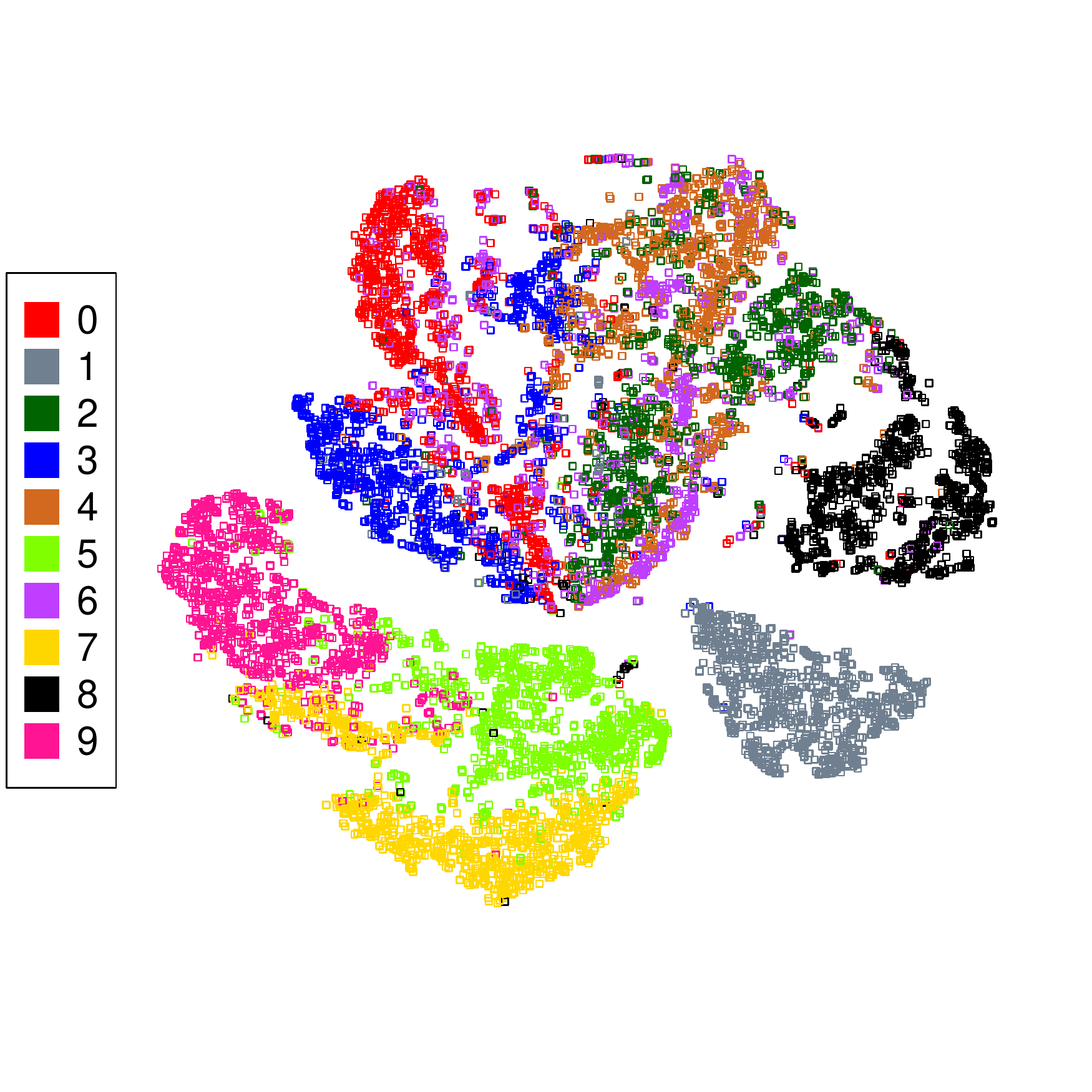}}
\subfigure[Third hidden layer (DFA).]{\label{fig:dfa_tsne_z3}\includegraphics[width=48mm]{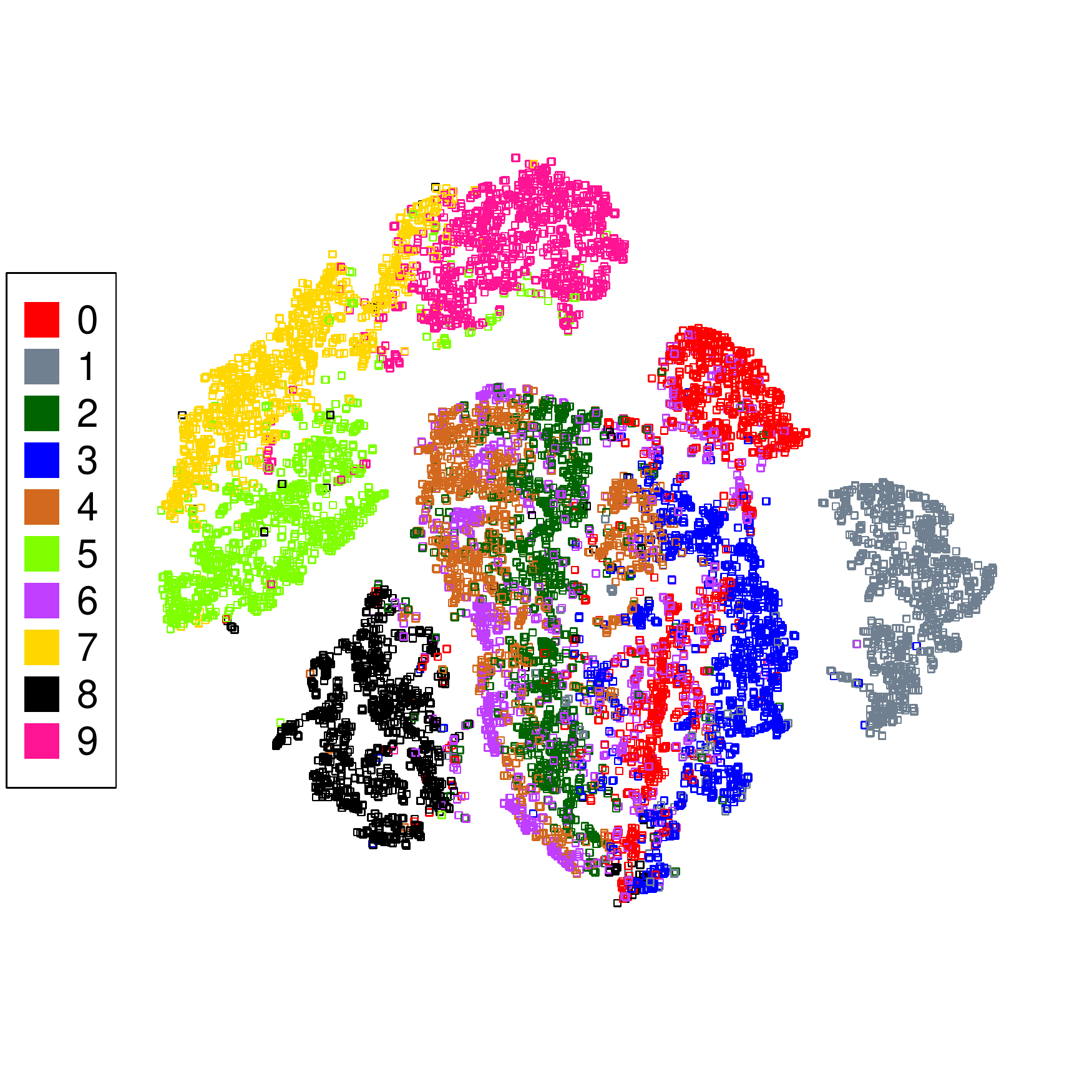}}

\subfigure[First hidden layer (LRA).]{\label{fig:lra_tsne_z1}\includegraphics[width=48mm]{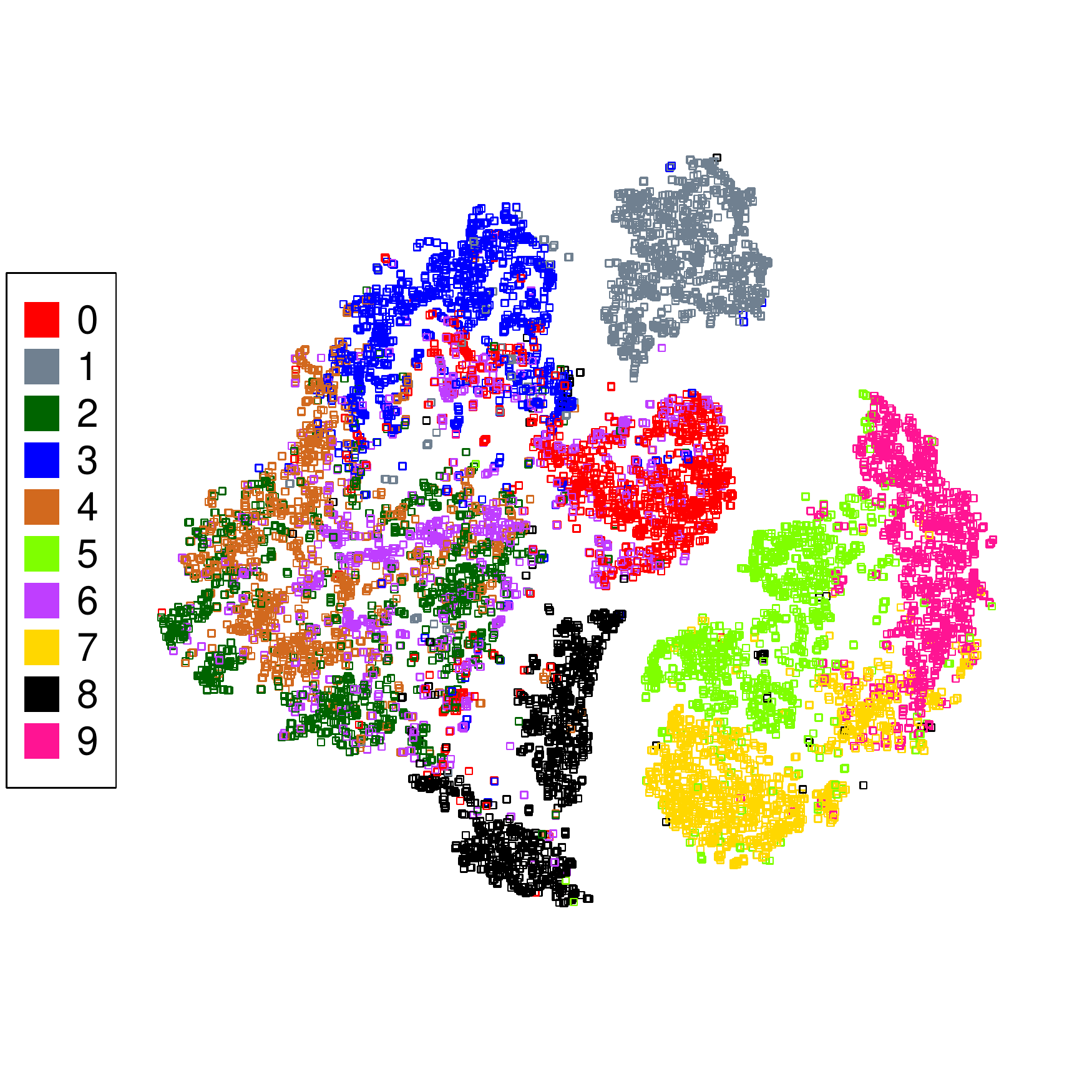}}
\subfigure[Second hidden layer (LRA).]{\label{fig:lra_tsne_z2}\includegraphics[width=48mm]{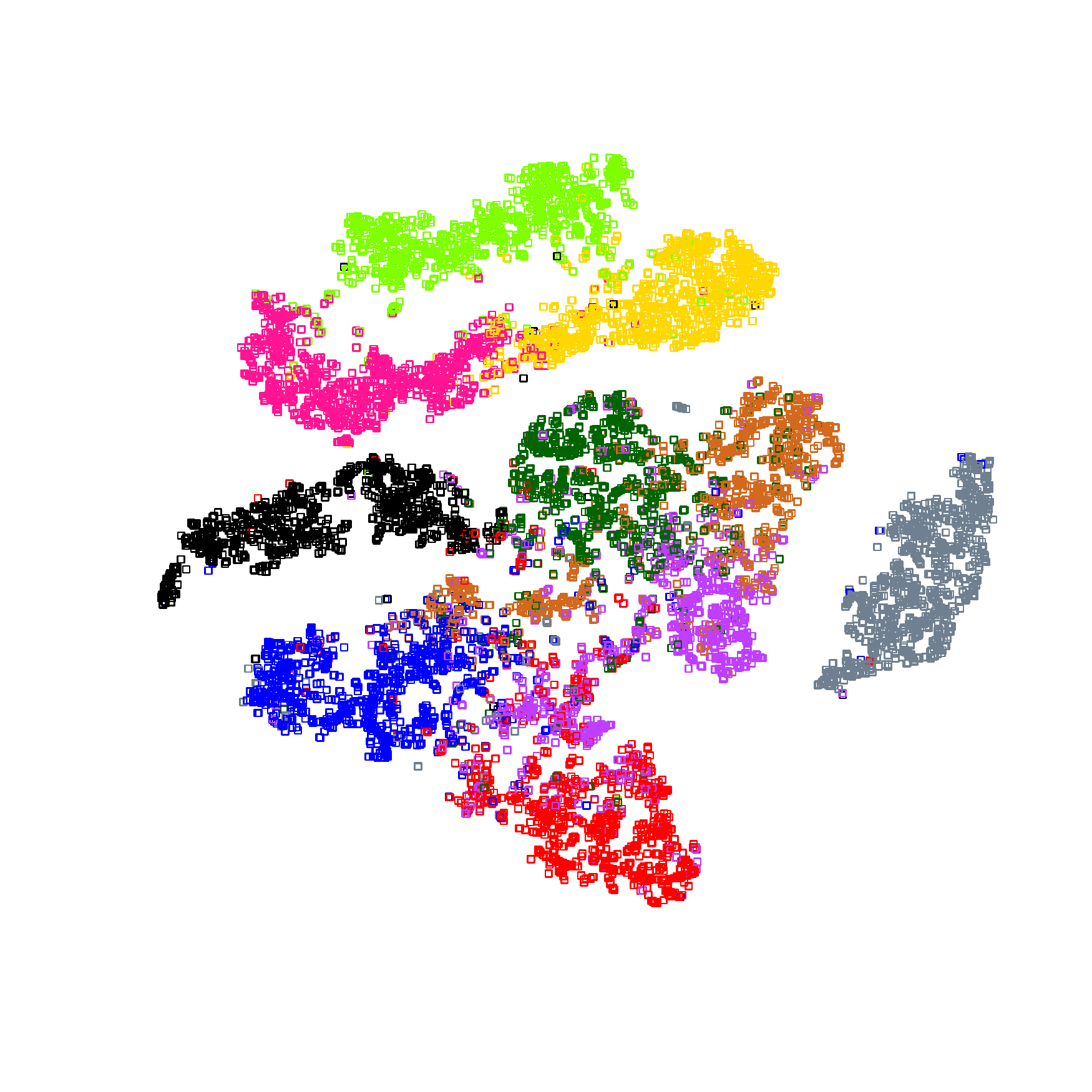}}
\subfigure[Third hidden layer (LRA).]{\label{fig:lra_tsne_z3}\includegraphics[width=48mm]{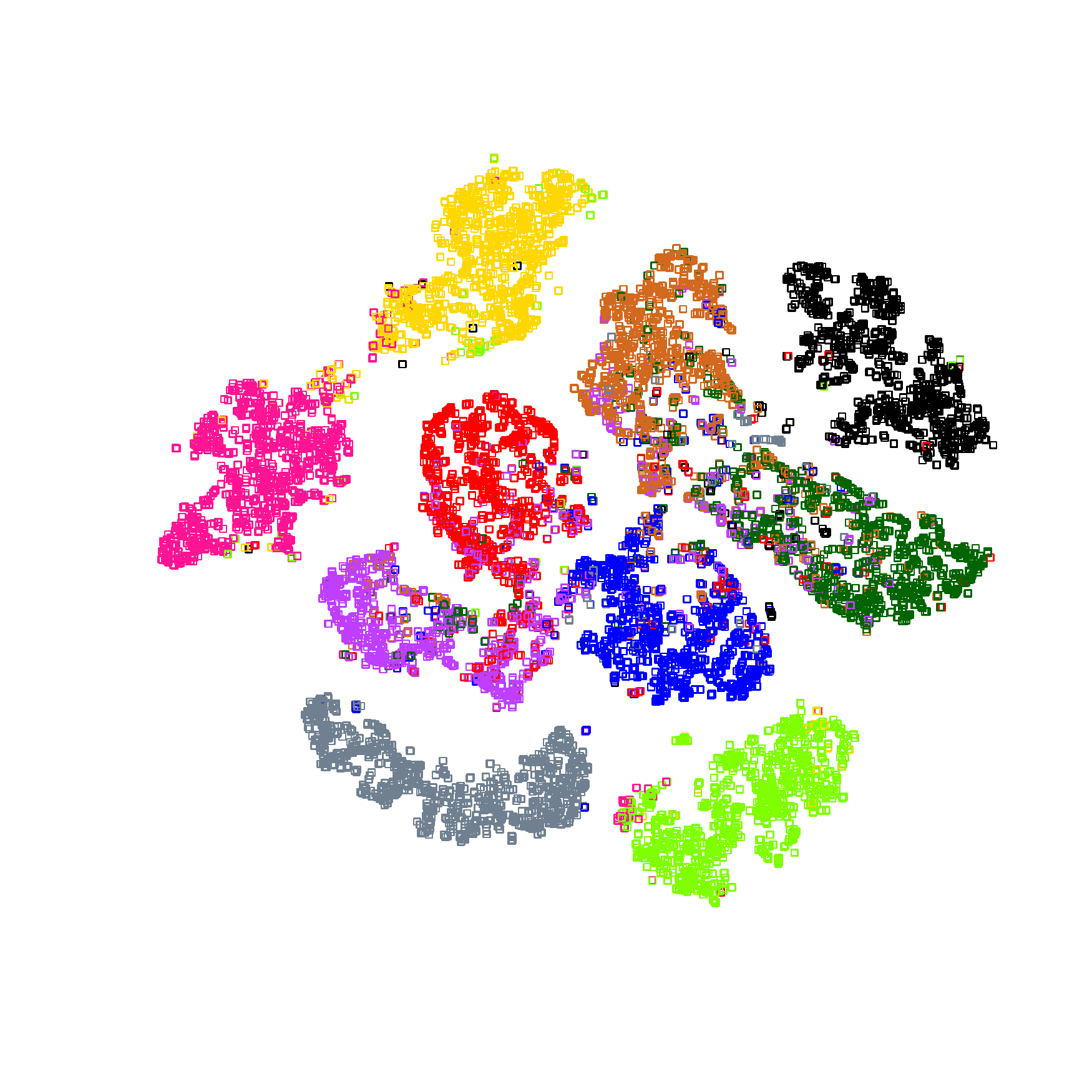}}
\caption{t-SNE plots of the various latent spaces acquired by the sparse rectifier network trained with backprop (top) or LRA-fdbk (bottom) on Fashion MNIST.}
\label{results:lra_tsne}
\end{figure*}

Note that with \emph{LRA-fdbk}, one can easily integrate what we will call \emph{soft} lateral competition instead of the hard vector quantization in LWTA. For example, one can use the softmax instead of the argmax operator for each competition block. This would mean the lateral competition block would be defined as:
\begin{align}
  z^i_j = \frac{\exp(z^i_j)}{\sum^n_{k = 1} \exp(z^i_k)} \mbox{.}
\end{align}
This soft block could then be treated as the probabilities of a mini categorical distribution and sampled accordingly, if a hard-decision is still required. Since \emph{LRA-fdbk} does not require the derivative of the lateral competition block function, one does not need to compute the expensive Jacobian associated with the pure softmax function.\footnote{When the softmax is used at the output layer in purely differentiable systems, one takes advantage of the analytically simplified derivative of the output with respect to pre-activities when using the categorical log likelihood loss function. When using the softmax \emph{inside} the network, this trick is no longer available to the practitioner.} For this experiment, we refer to the proposed soft LWTA as SLWTA.

In Table \ref{results:lra_nets}, we see that \emph{LRA-fdbk} is able to successfully train different activation functions. This includes the deep models that contain discrete-valued units. It is interesting to note, that while all networks trained with \emph{LRA-fdbk} generalize reasonably well, a network that performs best on one dataset does not necessarily perform the best on the other. For example, the best-performing network on MNIST was the network that used the signum function while the sparse rectifier network performed better on Fashion MNIST. To dig deeper into the discriminative ability of each layer of a network learned with \emph{LRA-fdbk} versus other algorithms such as BP and DFA, we extracted the hidden representations of the model learned under each when applied to the test-set of Fashion MNIST. Each multidimensional representation vector was then projected to 2D for visualization using t-SNE, Barnes-Hutt approximation \citep{van2013barnes}. The results of this visualization can be found Figure \ref{results:lra_tsne}. We see, first and foremost, that the model learned with \emph{LRA-fdbk} has acquired distributed representations that contain information useful for properly separating the data-points by class.

\subsection{Stochastic Networks}
We next investigate if LRA, or specifically \emph{LRA-fdbk}, can successfully handle networks with stochastic units, such as Bernoulli-distributed variables, an important class of non-differentiable activation functions. 
We compare LRA with DTP and back-propagation of errors. For back-propagation of errors, since a discrete sampling function is non-differentiable, we explore a variety of approximations that deal with binary stochastic units, which can be considered to be noisy modifications of the logistic sigmoid. To train networks with these kinds of units, we need an estimator, which can be classified under two categories--unbiased and biased. If the expected value of an estimate equals the expectation of the derivative it is trying to estimate, the estimator is unbiased. Otherwise, it is biased. We examine several such estimators, including the straight-through estimator  STE \citep{Bengio13,Chung16}, variations of the slope-annealing trick \citep{Chung16}, and reinforcement learning approaches \citep{Williams1992,Chung16} to training discrete-valued variables, i.e., REINFORCE. REINFORCE operates directly on the loss of the network and does not require back-propagated gradients while the class of STEs simply replace the derivative of the Bernoulli sampling operation with the identity function. The sigmoid-adjusted STE replaces the same derivative with that of the logistic sigmoid. In slope-annealing, we multiply the input value by a slope-value $m$, which is increased throughout training to make the sigmoidal derivative ultimately approach the step function. %, the true derivative. % AO maybe remove this description...ADD CITATIONS

The stochastic models trained for this experiment each contain two layers of 200 hidden units (which is the setting used by \cite{lee2015targetprop}) and parameters are trained over 500 epochs. The specific architecture is as follows:
\begin{align}
\mathbf{h}^1 &= W_1 \mathbf{x},\quad \mathbf{z}^1 = sigm( \mathbf{h}^1 ) \\
\mathbf{h}^2 &= W_2 S(\mathbf{z}^1),\quad \mathbf{z}^2 = sigm( \mathbf{h}^2 ) \\
\mathbf{h}^3 &= W_3 S(\mathbf{z}^2),\quad \mathbf{z}^3 = softmax( \mathbf{h}^3 )
\end{align}
where $sigm(v)$ is the logistic sigmoid, which parametrizes the probability $p$ of the layer of Bernoulli variables, and $S(p) \sim\mathbf{B}(1, p)$ is a stochastic operator that takes in a probability $p$ and returns either a zero or a one, e.g., a binary variable.

\begin{table*}[t]
\begin{center}
\caption{Generalization error (\%) of various stochastic binary networks. Errors are calculated using model posterior probabilities averaged over $M = 100$ samples.}
\label{results:stoch_nets}
\vspace{2mm}
\footnotesize
\begin{tabular}{|l||c||c|}
\hline
  & \textbf{MNIST} & \textbf{Fashion MNIST} \\
  \hline
  ST-BP, Pass-through STE & 85.409 & 59.040 \\
  BP, Sigmoid-Adjusted STE & 96.039 & 85.17 \\
  BP, Slope-Annealing & 97.180 & 84.790 \\
  BP, REINFORCE & 88.811 & 75.160 \\
  BP, REINFORCE (Variance-Adj.) & 94.110 & 81.360 \\
  DTP             & $\mathbf{98.460}$ & 87.320 \\
  LRA-\emph{fdbk} & 98.280 &  $\mathbf{89.210}$ \\
  \hline
\end{tabular}
\end{center}
\end{table*}

In Table \ref{results:stoch_nets}, observe that LRA is able to effectively train networks composed of stochastic binary units, competitive with DTP and outperforming the other estimators used in back-propagation. This is encouraging, since it is well-known that actual neurons communicate via spikes, and modeling this discrete signal as a Bernoulli variable brings us one step closer to incorporating neuro-biological ideas into artificial neural architectures.  We believe that using spike-like variables in a neural system offers a form of regularization much akin to that of drop-out \citep{srivastava2014dropout}. The key feature of using spike variables is that at test-time, we do not ``shut off'' this mechanism as is done in drop-out (where we calculate an expectation over all possible sub-models by multiplying the activities by the drop-out probability used in training).
One could easily use a stochastic model such as the one we train to also characterize its uncertainty at the posterior by simply estimating its variance in addition to the mean as we have done in these particular experiments.

%Experiment 1: 
%Track the bias gradient norms of VDCA versus backprop to show learning rate differences, especially focusing on the case of sigmoidal units.

%Experiment 4:
%Create t-SNE visualizations of 1) the representations learned by architectures under each algorithm, 2) the function-space of each algorithm. Do the algorithms lead to architectures acquiring different representations of the data?  Do the MLPs end up in different spots of function-space as we know happens with pre-trained vs non-pre-trained MLPs (see "Why Does Greedy Unsupervised Pretraining Work?" paper).

\begin{table*}[t]
\begin{center}
\caption{Finetuning experimental results on the Fashion MNIST dataset. 
%Models were first pre-trained for 100 epochs using either backprop or LRA. The model parameters at the very end of 100 epochs (and not necessarily the best model as would be found by early stopping) were then subsequently finetuned for another 100 epochs using only gentle backprop. Error is reported on training, validaiton, and test splits. 
For models after pretraining (``pretrained''), we report error at the very end of 100 epochs.
For the ``finetuned'' models, error is reported for parameters with best validation error.}
\label{results:feat_vs_opt}
\vspace{2mm}
\footnotesize
\begin{tabular}{|l||c|c|c|}
\hline
  & \textbf{Train} & \textbf{Valid} & \textbf{Test}\\
  \hline
  LRA, pretrained & 1.49 & 11.02 & 11.42 \\
  LRA, finetuned & 0.13 & 10.11 & 11.32 \\
  \hline
  Backprop, pretrained & 0.11 & 11.00 & 10.61 \\
  Backprop, finetuned & 0.00 & 10.00 & 10.19 \\
  \hline
\end{tabular}
\end{center}
\end{table*}

\subsection{Feature Extraction versus Optimization}
In our final experiment, we investigate if LRA works more like a feature extraction algorithm, which is more useful for pre-training, or as an optimization procedure. We conduct this experiment in two parts across two different researchers. The first part of the experiment entails training two MLPs with 3 hidden layers of 256 units, using hyperbolic tangent activation functions, one with \emph{LRA-fdbk} and one with back-propagation of errors. The parameters, e.g. synaptic weight matrices and bias vectors, are extracted from each model and communicated to a second researcher. The identifiers that indicate which network was trained by which algorithm are removed before communication. The second researcher is to fine-tune both networks using back-propagation of errors and stochastic gradient descent.

\begin{figure*}
\centering     %%% not \center
\includegraphics[width=70mm]{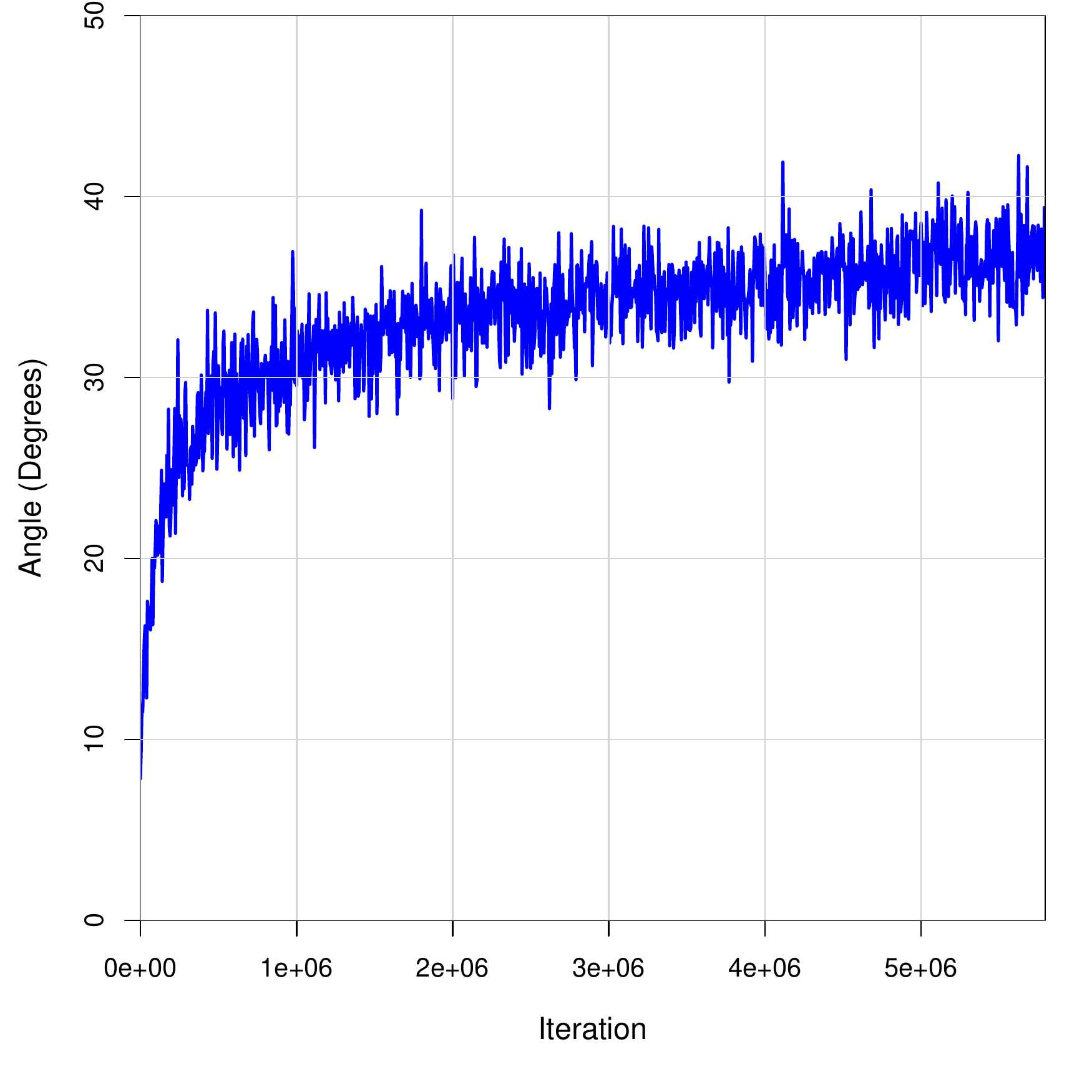}
\caption{Tracking the angle between updates computed by LRA with backprop as a function of iteration, or mini-batch. Every 1000th mini-batch the angle is measured.
}
\label{fig:angles}
\end{figure*}

To simulate pre-training for the first part of the experiment, parameters were optimized using the Adam adaptive learning rate (with a step-size of $0.0002$) and mini-batches of 50 samples. Weights for the models were initialized from a Gaussian distribution $\mathcal{N}(0,0.05)$ and error feedback weight matrices $E_\ell$ (used by LRA-fdbk) were initialized using a zero-mean Gaussian $\mathcal{N}(0,1)$. For fine-tuning, we trained with back-propagation and stochastic gradient descent with a step size of $0.001$ and mini-batches of 50 samples. Both phase models were trained over 100 epochs. 

% Think about this...is the validation correct? (AO)
In Table \ref{results:feat_vs_opt}, we report model performance on each split (training, validation, and test) at the end of pretraining and the best model (as measured on validation) found during fine-tuning.
We see that fine-tuning an LRA-trained network with back-propagation of errors appears to improve performance, especially on the training and validation sets. Due to the small training error (compared to validation and testing) and the similarities in errors between backprop and LRA, it appears that LRA still works like an optimization algorithm but one that is regularized (so it is slower to overfit). The source of this regularization might come from the use of the error feedback weights.
%This might indicate that LRA works more like a feature-extraction algorithm instead of an optimization algorithm, especially since we see that when fine-tuning a backprop-trained network with backprop only slightly improves scores on training (leading to zero error). 

%When considering the filter visualizations reported earlier (Figure \ref{fig:filter_lra2}) for LRA in comparison to those of backprop (Figure \ref{fig:filter_bp}), it does appear that LRA is capable of extracting useful/less noisy features sooner than backprop. 

To test this idea further, we conducted another experiment that measured the angle between the parameter updates computed by LRA (collectively denoted by $\Delta_{LRA}$)
%, $\Delta_{LRA}$ (which is the construct housing all of the different updates for all parameters $\Theta$), 
after processing a particular mini-batch and the angle of the update as would have been given by backprop, (collectively denoted by $\Delta_{BP}$). The angle is measured by first computing the cosine of the angle between the two different types of updates computed at any iteration $i$, or $cos(\Delta_{LRA},\Delta_{BP} = (\Delta_{LRA})^T \Delta_{BP} / ( ||\Delta_{LRA}||_2 ||\Delta_{BP}||_2)$, and ultimately converting to degrees. As indicated in Figure \ref{fig:angles}, we see that over the course of training (100 epochs), the angle between the updates found by the backprop and LRA are within the desired $90^{\circ}$  (meaning that the LRA update is a descent direction) and while it increases from an initial $10^{\circ}$, it ultimately stabilizes at about a divergence of $38^{\circ}$ on average. 
%The fact that LRA does follow steepest descent toward minima in error space might be serving as a form of implicit regularization, perhaps due to the noise/error in computing target representations. This strengthens our conclusion that LRA is working like a regularized optimization procedure.
%that LRA might simply be doing both jobs--optimization and feature extraction, indicating that it is more of a hybrid algorithm. This makes sense if we consider the fact that backprop is greedy, traveling down the direction of steepest descent, and the direction found by LRA would only be within 90 degrees of backprop.

\section{Related Work}
\label{sec:review}
% LRA --> approaches credit assignment by decomposition of problem, inspired by predictive coding...
% LRA --> is an example of local learning, but coordinated = relates to target prop /recirc
% LRA --> is a concrete implementation of the more general idea of discrepancy reduction
% can related DFA/FA to LRA a bit since fixed projection weights are part of its internal process
% can discuss other learning algorithms like EqProp and Var Walkback and decoupled neural interfaces

 % Dan comment: The closely related work: targetprop, the paper by lecun, the paper by carrera-perpinan need to be discussed in more detail - how they are related to lra (saying targetprop was designed for gaussians is not relevant).
% This discussion should be at or very close to the beginning.
% other stuff about predictive coding, VAE, and so on are very tangential and could probably be removed. Related work should not be philosophical.

As suggested by the algorithm's name, LRA approaches the credit assignment problem by explicitly formulating and optimizing the related problem of learning good latent representations, or abstractions, of the data. Specifically, LRA decomposes the problem into a series of linked local learning problems. This form of learning is what we will call ``coordinated local learning''. Like classical local (neurobiologically plausible) rules, such as Hebbian learning \citep{hebb1949organization}, we make use of information that is within close proximity of particular groups of neurons to compute updates to synaptic weights--guessed initial activation patterns and target activation patterns for any particular layer. However, unlike pure local learning, part of the local information LRA uses, i.e., the targets, are created through a process that is guided more globally by either explicit feedback weights or an iterative-inference pathway (implemented by the chain rule of calculus). The motivation behind this particular style of computation that defines LRA comes from the theory of predictive coding, part of which posits that local computations occur at multiple levels of the biological structure underlying the human brain \citep{grossberg1982does,rao1999predictive,huang2011predictive,clark2013whatever,panichello2013predictive}. This stands in contrast to back-propagation of errors, the workhorse algorithm behind modern neural networks, which crucially conducts credit assignment through the use of a global feedback pathway to carry back the error signals needed for computing updates \citep{ororbia2017learning}. This particular pathway creates problems such as exploding and vanishing gradients \citep{bengio1994learning,glorot2010understanding} and imposes severe restrictions on kinds of operations and modifications we can use--highly nonlinear mechanisms such as lateral neuronal competition and non-differentiable operators, such as discrete-valued stochastic activations, are difficult or even impossible to implement.

As we have shown earlier, LRA was created from the perspective of viewing back-propagation of errors from the perspective of target propagation \citep{Bengio14,lee2015targetprop}, of which recirculation \citep{hinton1988learning,o1996biologically} is a predecessor algorithm. In recirculation, a single hidden layer autoencoder uses the datum as the target value for reconstruction (affecting the decoder) and the initial encoded representation of the datum as the target for the encoder, which is computed after a second forward pass. 
Target propagation revolves around the concept of the function inverse--if we had access to the inverse of the network of forward propagations, we could compute which input values at the lower levels of the network would result in better values at the top that would please the global cost. In essence, we would use the inverse to propagate back along the network the target value and then update each layer to move closer to this target value. So long as we have access to the inverse of the functions used for each layer, we can use any non-linear activation, including those that are discrete-valued. Under simple conditions, when all the layer objectives are combined, target propagation could yield updates with the same sign as the updates obtained by back-propagation \citep{le1986learning}.

Like LRA, algorithms like target propagation and recirculation can be viewed as using higher-level objectives that seek better representations of data, governed by the principle of discrepancy reduction \citep{ororbia2017learning} which entails a two-step process for learning: 1) seek better representations of data, 2) minimize the mismatch between the model state and this better state. Furthermore, they represent a strong push towards using local, more biologically plausible, rules to learn neural systems. 

Local learning first made a small resurgence when training deeper networks first came into mainstream view in the form of layer-wise training of unsupervised models \citep{bengio_greedy_2007}, supervised models \citep{lee_deeply-supervised_2014}, and semi-supervised models/hybrid training \citep{ororbia_deep_hybrid_2015a,ororbia2015online}.  Although important in stimulating work towards improved learning and initialization of more complex neural models, the key problem with these layer-wise training approaches was that they were greedy--building a model from the bottom-up, freezing lower-level parameters as higher-level feature detectors were learnt. These approaches lacked the global coordination where upper-layer feature detectors direct lower-layer feature detectors as to what basic patterns they should be finding.\footnote{A lower-level feature detector might be able to find different aspects of structure in its input since multiple patterns might satisfy its layer-wise objective but this might not help the layers above find better higher-level patterns/abstractions.} Nonetheless, interesting local update rules could be used in the construction of these ``stacked'' models--back-propagation on the reconstruction cross entropy for autoencoders \citep{vincent2008extracting} and Contrastive Divergence for Boltzmann-based models \citep{hinton2002training,bengio_greedy_2007}. Another interesting approach, and one related to LRA in that it cares about sparsity of representations, is that of stacked sparse coding \citep{he2014unsupervised}, which greedily learns a composition of sparse coding sub-models \citep{olshausen1997sparse}.
%As mentioned earlier, LRA introduces a top-down or more globally-aware form of coordination through various forms of feedback.

The idea of learning locally is slowly becoming prominent in the training artificial neural networks, with other recent proposals including kickback \citep{balduzzi2015kickback}, which was derived specifically for regression problems. MAC/QP \citep{Miguel14} relaxes the hard constraint that the output of one layer equals the input to the next layer, adding a penalty term to the objective function when they are different. This allows the training of deep networks to be done locally and parallelized. Decoupled neural interfaces \citep{jaderberg2016decoupled} also operate locally, but take the approach of learning a predictive model of error gradients (and inputs) instead of trying to use local information to estimate an update to weights. As a result, this procedure allows layers of the underlying model to be trained independently. Other related approaches, which take a stochastic/probabilistic view of learning, include expectation propagation \citep{jylanki2014expectation}, the variational walkback algorithm \citep{goyal2017variational}, and equilibrium propagation \citep{scellier2017equilibrium}. Contrastive Hebbian learning (CHL) \citep{movellan1991contrastive,xie2003equivalence,o2001generalization} works similarly to Contrastive Divergence in that it is ultimately computing parameter updates using a positive phase and a negative phase, trying to make the negative phase (or the ``fantasies'') more similar to the positive phase (which is the state of the model clamped at the data).

% can maybe use
%This leads to the idea of \emph{higher-level objectives}, or local objectives that operate at various levels of latent space \citep{ororbia2017learning}. Variational autoencoders \citep{kingma2013auto}, and various extensions \cite{serban2017piecewise}, partially embody this idea, augmenting the usual global loss (reconstruction error) with a Kullback-Leibler Divergence, that measures the distance between a prior distribution and posterior distribution over some of its latent variables.

Another important idea that comes into play in LRA is that learning is possible with asymmetry, and even more interestingly, with random fixed feedback loops. This was shown in a learning algorithm called feedback alignment \citep{lillicrap2016random}, which essentially replaces the backward pass of back-propagation that involves the transpose of the feedforward weights with a random matrix of the same dimensions. Direct feedback alignment \citep{nokland2016direct} extends this idea further by directly connecting the output layer's pre-activation derivative to each layer's post-activation. These feedback alignment procedures directly resolve one of the main criticisms of back-propagation--the weight-transport problem \citep{grossberg_resonance_1987,liao2016important}--by showing that coherent learning is possible with asymmetric forward and backward pathways. LRA, however, uses the idea of random error feedback loops quite differently--use the error feedback connections to generate a better target representation to move towards instead of simply replacing the error gradient computations within back-propagation's global feedback pathway. %This is more in spirit to the rough idea proposed in \citep{ororbia2017learning}.

\section{Conclusion \& Future Work}
\label{conclusion}
We showed how back-propagation can be re-cast in the framework of target propagation and used the insights from this perspective to propose the Local Representation Alignment (LRA) algorithm. LRA is a training procedure that decomposes the credit assignment problem of artificial neural networks into smaller, local learning problems. Specifically, we introduced the notion of the computation subgraph--an object that encompasses two layers of processing elements and the underlying operations and parameter variables that connect them--and how to view a deep network as a linked set of such subgraphs. Motivated by fundamental ideas in representation learning, LRA structures every subgraph within the network to have a target, not just at the output layer, and adjusts the free parameters of the subgraph to move the output closer to the target. LRA, in contrast to previous approaches including target propagation, chooses targets that are in the possible representation of the associated layers and hence the layer's parameters can be updated more effectively (i.e. layers are not made to match a target that is impossible to achieve). The subgraph view also allowed us to introduce a short-circuit pathway, inspired by the idea of feedback alignment, which allows LRA to handle non-differentiable activation functions.

Unlike previously proposed algorithms, including back-propagation, target propagation, and variants of feedback alignment, LRA is far less sensitive to parameter initialization when training highly nonlinear networks. Furthermore, it 
%sidesteps the vanishing gradient problem when training deep networks independently of the nonlinearity used and 
can adaptively decide the depth of the credit assignment it needs to conduct, which can lead to savings in computation per step. In addition to being compatible with recent innovations such as batch normalization and drop-out, LRA is architecture-agnostic, so long as the global model can be decomposed into a series of linked subgraphs, where the output each subgraph can be viewed as a hidden representation to which a target can be assigned. For the case of feedforward networks, our experiments on MNIST and Fashion MNIST add strong empirical evidence to support the above claims.

LRA offers a pathway for users of neural networks to design the architecture for the problem at hand rather than for the traits and quirks of back-prop-based algorithms. This means that non-differentiable units and more biologically-motivated ones may be utilized.
Since discrepancy reduction \citep{ororbia2017learning} can be viewed as a special variant of LRA and was shown to be capable of successfully learning directed models of temporal data without back-propagation through time, future work will include training recurrent network models, on videos and text documents, with LRA. Furthermore, we intend to examine the algorithm's performance and behavior at a larger scale than that investigated in this paper.

\acks{We would like to acknowledge support for this project from the National Science Foundation (NSF grant \#1317560 ). }


\newpage

\appendix
\section*{Appendix A.}

\begin{figure*}
\centering     %%% not \center
\includegraphics[width=90mm]{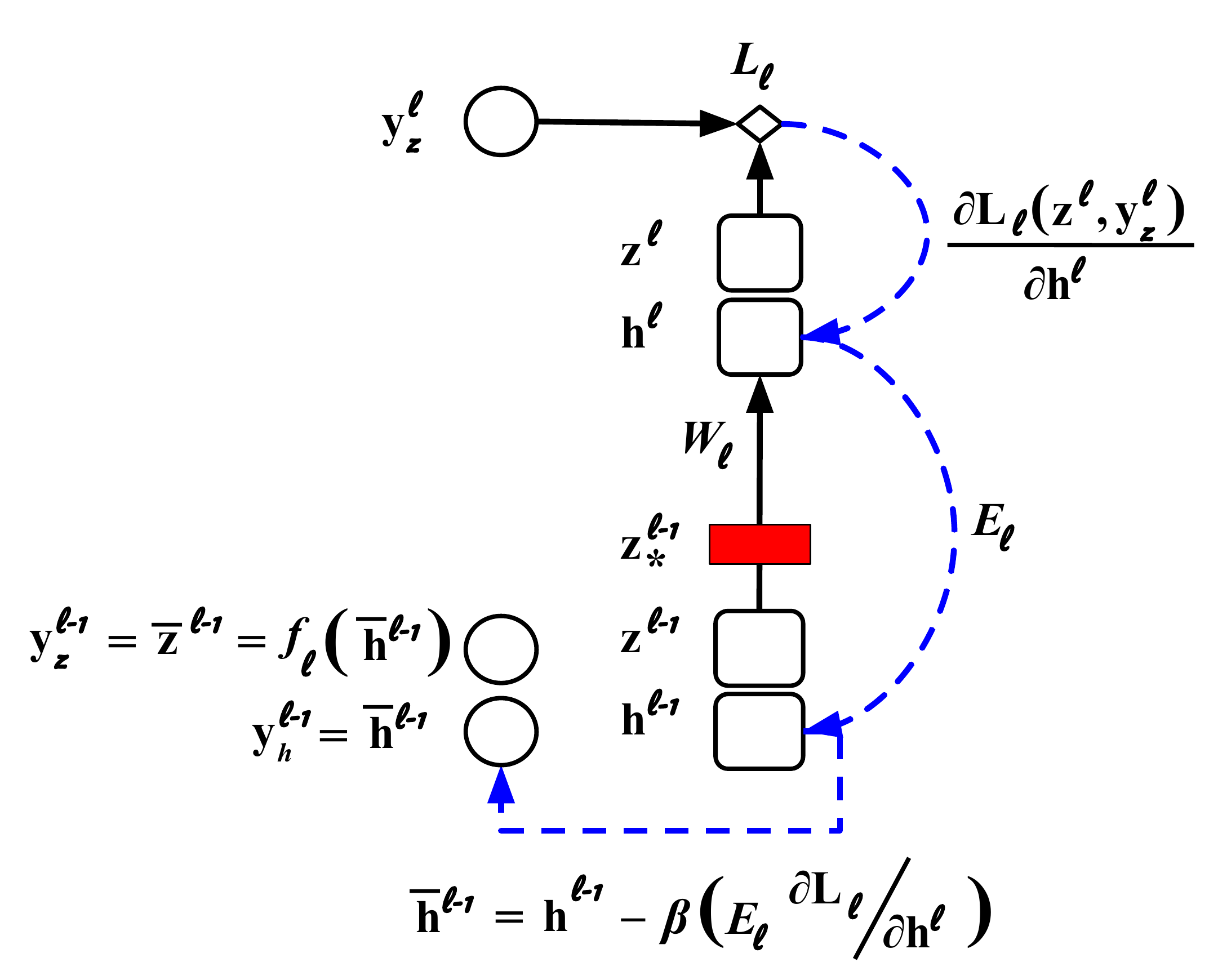}
\caption{The target calculation process for two variants of LRA applied to a differentiable multilayer perceptron (MLP) with a non-differentiable operator (shown in red).
}
\label{fig:lra_nondiff}
\end{figure*}

\begin{algorithm}%[tb]
   \caption{LRA-fdbk applied to an MLP with non-differentiable components.}
   \label{algo:lra_nondiff}
\begin{algorithmic}
   \STATE {\bfseries Input:} data $(\mathbf{x}, \mathbf{t})$, number steps \emph{K}, halting criterion $\epsilon$, step-size $\beta$, model parameters $\Theta = \{W_1, \mathbf{b}_1, \cdots, W_\ell, \mathbf{b}_\ell, \cdots, W_n, \mathbf{b}_n \}$, and norm constraints $\{ c_1, c_2 \}$
   \STATE {\bfseries Specify:} $\{f_1(\mathbf{h}), \cdots, f_\ell(\mathbf{h}), \cdots, f_n(\mathbf{h}) \}$ and $\{g_1(\mathbf{h}), \cdots, g_\ell(\mathbf{h}), \cdots, g_{n-1}(\mathbf{h}) \}$, \\$\{ \mathcal{L}_1(\mathbf{z}, \mathbf{y}), \cdots, \mathcal{L}_\ell(\mathbf{z}, \mathbf{y}), \cdots, \mathcal{L}_n(\mathbf{z}, \mathbf{y}) \}$, and, optionally, error weights $\{E_1, \cdots, E_\ell, \cdots, E_n \}$
   \STATE \COMMENT $\mathbf{y}^\ell_h$: what we would like input to the activation function at layer $\ell$ to be
   \STATE \COMMENT $\mathbf{y}^\ell_z$: what we would like output of activation function $f_\ell$ at layer $\ell$ to be
   \STATE \COMMENT $\mathbf{h}^\ell$: input to activation function at layer $\ell$ resulting from feed forward phase
      \STATE \COMMENT $\mathbf{z}^\ell$: output of activation function at layer $\ell$ resulting from feed forward phase
      \STATE \COMMENT $\mathbf{\bar{h}}^\ell$ corresponds to $\mathbf{h}^\ell$, $\mathbf{\bar{z}}^\ell$ corresponds to $\mathbf{z}^\ell$, and $\mathbf{\bar{z}}^\ell_{*}$ corresponds to $\mathbf{z}^\ell_{*}$
    \STATE \textbf{Function}\{ComputeUpdateDirection\}\{$(\mathbf{x}, \mathbf{t}),\Theta, K, c_1, c_2, \beta, \epsilon $\}
    \bindent
    \STATE $\mathbf{z}^0 = \mathbf{x} $ \COMMENT Run feedforward pass to get initial layer-wise statistics
	\FOR{$\ell=1$ {\bfseries to} $n$}
    	\bindent
        \STATE $W_\ell, \mathbf{b}_\ell, \leftarrow \Theta$ %// Extract relevant layerwise parameters
    	\STATE $\mathbf{h}^\ell \leftarrow W_\ell \mathbf{z}^{\ell-1} + \mathbf{b}_\ell $, $\mathbf{y}^\ell_h \leftarrow \mathbf{h}^\ell$, $\mathbf{\bar{h}}^\ell \leftarrow \mathbf{h}^\ell$
    	\STATE $\mathbf{z}^\ell \leftarrow f_\ell( \mathbf{h}^\ell )$, $\mathbf{y}^\ell_z \leftarrow \mathbf{z}^\ell$, $\mathbf{\bar{z}}^\ell \leftarrow \mathbf{z}^\ell$\
        \STATE $\color{blue} \mathbf{z}^\ell_{*} \leftarrow  g_\ell(\mathbf{z}^\ell) \mbox{ if } \ell \neq n \mbox{ and } \mathbf{z}_\ell \mbox{ otherwise}  $
        \eindent
    \ENDFOR
    \STATE $\mathbf{y}^n_z \leftarrow \mathbf{t}$ \COMMENT Override top-level target with correct target from training data
    \STATE $\ell = n$
    \WHILE{$\ell \geq 1$ and $\mathcal{L}_\ell( \mathbf{z}^{\ell} , \mathbf{y}^\ell_z) \geq \epsilon$ }
    	\bindent
          \STATE $W_\ell, \mathbf{b}_\ell, \leftarrow \Theta$

            \STATE \COMMENT Calculate parameter update direction for layer $\ell$, comparing initial guess to target
            \STATE \COMMENT Normalize$(\cdot,\cdot)$ is defined in Eq \ref{eqn:normalize}.
            \STATE $\nabla_{W_\ell} \leftarrow Normalize( \frac{\partial \mathcal{L}_\ell( \mathbf{z}^{\ell}, \mathbf{y}^\ell_z)}{\partial W_\ell }, c_1 )$, $\nabla_{\mathbf{b}_\ell} \leftarrow Normalize( \frac{\partial \mathcal{L}_\ell( \mathbf{z}^\ell , \mathbf{y}^\ell_z)}{\partial \mathbf{b}_\ell }, c_1 ) $
            \STATE \COMMENT Find target for layer $\ell-1$ (Note: could add early stopping criterion)
            \STATE \COMMENT Calculate pre-activation displacement
            \STATE $\Delta_{h^{\ell-1}} \leftarrow E_\ell \frac{\partial \mathcal{L}_\ell(\mathbf{z}^\ell, \mathbf{y}^\ell_z) }{\partial \mathbf{h}^{\ell}} = E_\ell \Big( \frac{\partial \mathcal{L}_\ell(\mathbf{z}^\ell, \mathbf{y}^\ell_z) }{\partial \mathbf{z}^\ell} \otimes \frac{\partial f_\ell(\mathbf{h}^\ell) }{\partial \mathbf{h}^\ell} \Big)$, $\Delta_{h^{\ell-1}} \leftarrow Normalize( \Delta_{h^{\ell-1}}, c_2 ) $ %\COMMENT See text for definition of \emph{Normalize}$(\cdot)$
            \STATE \COMMENT Recalculate neural activities of subgraph
            \STATE $\mathbf{\bar{h}}^{\ell-1} \leftarrow \mathbf{\bar{h}}^{\ell-1} - \beta \Delta_{h^{\ell-1}}, \quad \mathbf{\bar{z}}^{\ell-1} \leftarrow f_{\ell-1}( \mathbf{\bar{h}}^{\ell-1} )$ %,\quad $\color{blue} \mathbf{\bar{z}}^{\ell-1}_{*} \leftarrow g_{\ell-1}(\mathbf{\bar{z}}^{\ell-1})$
            %\STATE $\mathbf{\bar{h}}^\ell \leftarrow  W_\ell {\color{blue}\mathbf{\bar{z}}^{\ell-1}_{*} } + \mathbf{b}_\ell, \quad $ $\mathbf{\bar{z}}^\ell \leftarrow f_\ell( \mathbf{\bar{h}}^\ell )$ 
            \STATE $\mathbf{y}^{\ell-1}_z \leftarrow \mathbf{\bar{z}}^{\ell-1}$ \COMMENT Update variable holding target for subgraph below

        \STATE $\ell = \ell - 1$
        \eindent
        \ENDWHILE
	\STATE \textbf{Return} $\nabla_{\Theta} = \{ \nabla_{W_1}, \nabla_{\mathbf{b}_1}, \cdots, \nabla_{W_\ell}, \nabla_{\mathbf{b}_\ell}, \cdots, \nabla_{W_n}, \nabla_{\mathbf{b}_n} \}$
    \eindent
    \STATE \textbf{EndFunction}
    %\STATE % empty space to return
\end{algorithmic}
\end{algorithm}

In Figure \ref{fig:lra_nondiff}, we illustrate a computation subgraph that contains a non-differentiable operation, such as a discrete-valued nonlinearity or a sampling function. LRA avoids the need for approximating the derivative of the non-differentiable function by simply short-circuiting the pathway that gradients would naturally flow through. This is where the wiring of the error weights becomes even more useful--we can ``pocket'' exotic functions right on top of the subgraph's input post-activation. Another way to view this setup is to consider the input post-activation to be composed of two processing steps, an initial nonlinear transformation (such as through the hyperbolic tangent) followed by a discretization step (such as through a signum).

The full process of LRA, including target computation and parameter update calculation, is presented in Algorithm \ref{algo:lra_nondiff}. The primary differences between this non-differentiable variant and the original are emphasized by coloring the pseudocode lines relevant to the non-differentiable operations in blue. In particular, we see that we calculate the post-activation as two pieces--$\mathbf{z}^\ell$ and $\mathbf{z}^\ell_{*}$. $g_\ell(\cdot)$ is used to specifically indicate the discrete-valued activation function. It is also important to note that we do not apply discretization to the topmost post-activation, $\mathbf{z}^n$. Examples of non-differentiable operations can include the Heaviside step function or a Bernoulli sampling operator. However, investigating the integration of other functions where derivative computation is either expensive or impossible is now viable in this particular setup of LRA.

\section*{Appendix B.}
\label{extension:unfolding}
In this appendix we describe how Local Representation Alignment (LRA), specifically, \emph{LRA-fdbk}, can be applied to the training of a recurrent neural network (RNN). This variation of LRA aligns with the extension of back-propagation of errors to sequential neural models, or back-propagation through time \citep{werbos1988generalization}. 

Essentially, we apply LRA of errors but with one crucial exception--we must unfold the network $T$ steps back in time\footnote{Note that this action is the same as unrolling a mathematical recurrence relation, since the hidden state of an RNN is recursively computed.}, ideally from the end of the sample sequence back up to its beginning. Note that in practice, we break up our sequences into sub-sequences of length $K$, where $K < T$, and process time-varying data in chunks to maintain computational tractability. This creates a very deep feedforward network, with each input at each time step fed into the unfolded graph and the underlying parameters copied at each time step. 
%This parameter-sharing over time circumvents the need for a number of model parameters that increases with the length of each sequence, which would be highly impractical from a computer memory storage perspective.

We will focus on applying LRA to the situation where an RNN is fully unfolded over the length of an entire sample sequence. Assume a simple single-hidden layer Elman RNN with a linear output layer, defined as follows:
\begin{align}
\mathbf{h}^1_t = W \mathbf{x}_t + V \mathbf{z}^1_{t-1},\quad \mathbf{z}^1_t = \phi(\mathbf{h}^1_t),\quad \mathbf{z}^2_t = \mathbf{h}^2_t = U \mathbf{z}^1_t \label{eqn:rnn_simple}
\end{align}
where the parameters to learn are simply $\Theta = \{W,V,U\}$ (biases omitted). The task is next-step prediction, so at each time step, $\mathbf{y}_t = \mathbf{y}^2_{z,t} = \mathbf{x}_{t+1}$, where $t$ indexes a particular point in time.

Next, we define the variable $\mathbf{e}$ to be the first derivative of a given local loss, where $\mathbf{e}^{2}_t$ is the first partial derivative of the local loss with respect to the output units $\mathbf{z}^2_t$ at time $t$ and $\mathbf{e}^{1}_t$ is the first partial derivative of the local loss with respect to the hidden unit post-activation $\mathbf{z}^1_t$. In short, assuming a Gaussian loss for both output and hidden local losses $\mathcal{L}_2$ and $\mathcal{L}_1$ with a fixed variance of 1, this means that:
\begin{align}
\mathbf{e}^{2}_t &= \mathbf{e}^{2}_t(\mathbf{y}^{2}_{z,t},\mathbf{z}^{2}_t) = \frac{\partial \mathcal{L}_{2}(\mathbf{y}^{2}_{z,t},\mathbf{z}^{2}_t)}{\partial \mathbf{z}^{2}_t} = -(\mathbf{y}^{2}_{z,t} - \mathbf{z}^{2}_t) \\
\mathbf{e}^{1}_t &= \mathbf{e}^{1}_t(\mathbf{y}^{1}_{z,t},\mathbf{z}^{1}_t) = \frac{\partial \mathcal{L}_{1}(\mathbf{y}^{1}_{z,t},\mathbf{z}^{1}_t)}{\partial \mathbf{z}^{1}_t} = -(\mathbf{y}^{1}_{z,t} - \mathbf{z}^{1}_t) \mbox{.}
\end{align}
Finally, a single, extra set of recurrent parameters $E$ will transmit the error from the output units back to the hidden units. 
%Note that we now index each error unit and target with the symbol $t$, given that they operate at each time step of a sequence the RNN is currently processing.
%Let us assume Gaussian local losses (as in the MLP example developed earlier in this chapter) at both hidden and output layers.

To train an RNN over $T$ steps in time, we simply unfold the network as in back-propagation through time but copy the error units and error weights $T$ times as well. Finding the targets for the hidden layers of this unfolded RNN would then amount to:
\begin{align}
\mathbf{y}^{1}_{z,T} = \phi( \mathbf{h}^1_T - \beta (E \mathbf{e}^{2}_T) ), \cdots, \mathbf{y}^{1}_{z,t} = \phi( \mathbf{h}^1_t - \beta (E \mathbf{e}^{2}_t) ), \cdots,\mathbf{y}^{1}_{z,1} = \phi( \mathbf{h}^1_1 - \beta (E \mathbf{e}^{2}_1) )
\end{align}
noting that $\mathbf{e}^2_t$ can be readily computed since the target for the output units is the data point at the next time step $t+1$, in other words $\mathbf{y}^{1}_{z,t} = \mathbf{x}_{t+1}$. 
%Thus, in the context of an RNN, we observe that LRA is being applied to each time slice of the model, computing the difference between model output and target at each $t$ and then generating a target for each corresponding hidden state. Once the target $\mathbf{y}^{1}_{z,t}$ for any hidden state has been computed, we can calculate the hidden error units $\mathbf{e}^{1}_t$ and the appropriate local loss. 
Once targets for each $\mathbf{z}^2_t$ and $\mathbf{z}^1_t$ have been found, the updates to the parameters of this model are then calculated as follows:
\begin{align}
\Delta U = \sum^T_{t=1} \mathbf{e}^{2}_t (\mathbf{z}^1_t)^T,\quad  
\Delta V = \sum^T_{t=1} (\mathbf{e}^{1}_t \otimes \phi^{\prime}(\mathbf{h}^1_t) ) (\mathbf{z}^1_{t-1})^T,\quad 
\Delta W = \sum^T_{t=1} (\mathbf{e}^{1}_t \otimes \phi^{\prime}(\mathbf{h}^1_t) ) (\mathbf{x}_t)^T
\end{align}
where $\mathbf{z}^1_0 = 0$, or the null state. The update to $U$ is written in simplified form since the output nonlinearity's first derivative is the identity.

\vskip 0.2in
\bibliography{ref}

\begin{thebibliography}{61}
\providecommand{\natexlab}[1]{#1}
\providecommand{\url}[1]{\texttt{#1}}
\expandafter\ifx\csname urlstyle\endcsname\relax
  \providecommand{\doi}[1]{doi: #1}\else
  \providecommand{\doi}{doi: \begingroup \urlstyle{rm}\Url}\fi

\bibitem[Alias Parth~Goyal et~al.(2017)Alias Parth~Goyal, Ke, Ganguli, and
  Bengio]{goyal2017variational}
Anirudh~Goyal Alias Parth~Goyal, Nan Ke, Surya Ganguli, and Yoshua Bengio.
\newblock Variational walkback: Learning a transition operator as a stochastic
  recurrent net.
\newblock In I.~Guyon, U.~V. Luxburg, S.~Bengio, H.~Wallach, R.~Fergus,
  S.~Vishwanathan, and R.~Garnett, editors, \emph{Advances in Neural
  Information Processing Systems 30}, pages 4392--4402. Curran Associates,
  Inc., 2017.

\bibitem[Andersen et~al.(1969)Andersen, Gross, Lomo, and
  Sveen]{andersen1969participation}
P~Andersen, Gary~N Gross, T~Lomo, and Ola Sveen.
\newblock Participation of inhibitory and excitatory interneurones in the
  control of hippocampal cortical output.
\newblock In \emph{UCLA forum in medical sciences}, volume~11, page 415, 1969.

\bibitem[Balduzzi et~al.(2015)Balduzzi, Vanchinathan, and
  Buhmann]{balduzzi2015kickback}
David Balduzzi, Hastagiri Vanchinathan, and Joachim~M Buhmann.
\newblock Kickback cuts backprop's red-tape: Biologically plausible credit
  assignment in neural networks.
\newblock In \emph{AAAI}, pages 485--491, 2015.

\bibitem[Bengio(2014)]{Bengio14}
Yoshua Bengio.
\newblock How auto-encoders could provide credit assignment in deep networks
  via target propagation.
\newblock \emph{CoRR}, abs/1407.7906, 2014.
\newblock URL \url{http://arxiv.org/abs/1407.7906}.

\bibitem[Bengio et~al.(1994)Bengio, Simard, and Frasconi]{bengio1994learning}
Yoshua Bengio, Patrice Simard, and Paolo Frasconi.
\newblock Learning long-term dependencies with gradient descent is difficult.
\newblock \emph{IEEE transactions on neural networks}, 5\penalty0 (2):\penalty0
  157--166, 1994.

\bibitem[Bengio et~al.(2007)Bengio, Lamblin, Popovici, Larochelle,
  et~al.]{bengio_greedy_2007}
Yoshua Bengio, Pascal Lamblin, Dan Popovici, Hugo Larochelle, et~al.
\newblock Greedy layer-wise training of deep networks.
\newblock \emph{Advances in neural information processing systems},
  19:\penalty0 153, 2007.

\bibitem[Bengio et~al.(2013)Bengio, L{\'{e}}onard, and Courville]{Bengio13}
Yoshua Bengio, Nicholas L{\'{e}}onard, and Aaron~C. Courville.
\newblock Estimating or propagating gradients through stochastic neurons for
  conditional computation.
\newblock \emph{CoRR}, abs/1308.3432, 2013.
\newblock URL \url{http://arxiv.org/abs/1308.3432}.

\bibitem[Carreira{-}Perpi{\~{n}}{\'{a}}n and Wang(2012)]{Miguel14}
Miguel~{\'{A}}. Carreira{-}Perpi{\~{n}}{\'{a}}n and Weiran Wang.
\newblock Distributed optimization of deeply nested systems.
\newblock \emph{CoRR}, abs/1212.5921, 2012.
\newblock URL \url{http://arxiv.org/abs/1212.5921}.

\bibitem[Chung et~al.(2016)Chung, Ahn, and Bengio]{Chung16}
Junyoung Chung, Sungjin Ahn, and Yoshua Bengio.
\newblock Hierarchical multiscale recurrent neural networks.
\newblock \emph{CoRR}, abs/1609.01704, 2016.
\newblock URL \url{http://arxiv.org/abs/1609.01704}.

\bibitem[Clark(2013)]{clark2013whatever}
Andy Clark.
\newblock Whatever next? predictive brains, situated agents, and the future of
  cognitive science.
\newblock \emph{Behavioral and Brain Sciences}, 36\penalty0 (3):\penalty0
  181--204, 2013.

\bibitem[Eccles(2013)]{eccles2013cerebellum}
John~C Eccles.
\newblock \emph{The cerebellum as a neuronal machine}.
\newblock Springer Science \& Business Media, 2013.

\bibitem[Erhan et~al.(2010)Erhan, Courville, and
  Bengio]{erhan2010understanding}
Dumitru Erhan, Aaron Courville, and Yoshua Bengio.
\newblock Understanding representations learned in deep architectures.
\newblock \emph{Department d’Informatique et Recherche Operationnelle,
  University of Montreal, QC, Canada, Tech. Rep}, 1355, 2010.

\bibitem[Glorot and Bengio(2010)]{glorot2010understanding}
Xavier Glorot and Yoshua Bengio.
\newblock Understanding the difficulty of training deep feedforward neural
  networks.
\newblock In \emph{Proceedings of the Thirteenth International Conference on
  Artificial Intelligence and Statistics}, pages 249--256, 2010.

\bibitem[Glorot et~al.(2011)Glorot, Bordes, and Bengio]{glorot2011deep}
Xavier Glorot, Antoine Bordes, and Yoshua Bengio.
\newblock Deep sparse rectifier neural networks.
\newblock In \emph{Proceedings of the Fourteenth International Conference on
  Artificial Intelligence and Statistics}, pages 315--323, 2011.

\bibitem[Grossberg(1982)]{grossberg1982does}
Stephen Grossberg.
\newblock How does a brain build a cognitive code?
\newblock In \emph{Studies of mind and brain}, pages 1--52. Springer, 1982.

\bibitem[Grossberg(1987)]{grossberg_resonance_1987}
Stephen Grossberg.
\newblock Competitive learning: From interactive activation to adaptive
  resonance.
\newblock \emph{Cognitive Science}, 11\penalty0 (1):\penalty0 23 -- 63, 1987.
\newblock ISSN 0364-0213.
\newblock \doi{https://doi.org/10.1016/S0364-0213(87)80025-3}.
\newblock URL
  \url{http://www.sciencedirect.com/science/article/pii/S0364021387800253}.

\bibitem[Gulcehre et~al.(2016)Gulcehre, Moczulski, Denil, and
  Bengio]{gulcehre2016noisy}
Caglar Gulcehre, Marcin Moczulski, Misha Denil, and Yoshua Bengio.
\newblock Noisy activation functions.
\newblock In \emph{International Conference on Machine Learning}, pages
  3059--3068, 2016.

\bibitem[He et~al.(2015)He, Zhang, Ren, and Sun]{he2015delving}
Kaiming He, Xiangyu Zhang, Shaoqing Ren, and Jian Sun.
\newblock Delving deep into rectifiers: Surpassing human-level performance on
  imagenet classification.
\newblock In \emph{Proceedings of the 2015 IEEE International Conference on
  Computer Vision (ICCV)}, 2015.

\bibitem[He et~al.(2014)He, Kavukcuoglu, Wang, Szlam, and
  Qi]{he2014unsupervised}
Yunlong He, Koray Kavukcuoglu, Yun Wang, Arthur Szlam, and Yanjun Qi.
\newblock Unsupervised feature learning by deep sparse coding.
\newblock In \emph{Proceedings of the 2014 SIAM International Conference on
  Data Mining}, pages 902--910. SIAM, 2014.

\bibitem[Hebb(1949)]{hebb1949organization}
Donald~Olding Hebb.
\newblock The organization of behavior; a neuropsycholocigal theory.
\newblock \emph{A Wiley Book in Clinical Psychology.}, pages 62--78, 1949.

\bibitem[Hinton(2002)]{hinton2002training}
Geoffrey~E Hinton.
\newblock Training products of experts by minimizing contrastive divergence.
\newblock \emph{Neural computation}, 14\penalty0 (8):\penalty0 1771--1800,
  2002.

\bibitem[Hinton and McClelland(1988)]{hinton1988learning}
Geoffrey~E Hinton and James~L McClelland.
\newblock Learning representations by recirculation.
\newblock In \emph{Neural information processing systems}, pages 358--366,
  1988.

\bibitem[Huang and Rao(2011)]{huang2011predictive}
Yanping Huang and Rajesh~PN Rao.
\newblock Predictive coding.
\newblock \emph{Wiley Interdisciplinary Reviews: Cognitive Science}, 2\penalty0
  (5):\penalty0 580--593, 2011.

\bibitem[Ioffe and Szegedy(2015)]{ioffe2015batch}
Sergey Ioffe and Christian Szegedy.
\newblock Batch normalization: Accelerating deep network training by reducing
  internal covariate shift.
\newblock In \emph{International conference on machine learning}, pages
  448--456, 2015.

\bibitem[Jaderberg et~al.(2016)Jaderberg, Czarnecki, Osindero, Vinyals, Graves,
  and Kavukcuoglu]{jaderberg2016decoupled}
Max Jaderberg, Wojciech~Marian Czarnecki, Simon Osindero, Oriol Vinyals, Alex
  Graves, and Koray Kavukcuoglu.
\newblock Decoupled neural interfaces using synthetic gradients.
\newblock \emph{arXiv preprint arXiv:1608.05343}, 2016.

\bibitem[Jyl{\"a}nki et~al.(2014)Jyl{\"a}nki, Nummenmaa, and
  Vehtari]{jylanki2014expectation}
Pasi Jyl{\"a}nki, Aapo Nummenmaa, and Aki Vehtari.
\newblock Expectation propagation for neural networks with sparsity-promoting
  priors.
\newblock \emph{The Journal of Machine Learning Research}, 15\penalty0
  (1):\penalty0 1849--1901, 2014.

\bibitem[Kingma and Ba(2014)]{kingma2014adam}
Diederik~P Kingma and Jimmy Ba.
\newblock Adam: A method for stochastic optimization.
\newblock \emph{arXiv preprint arXiv:1412.6980}, 2014.

\bibitem[Le~Cun(1986)]{le1986learning}
Yann Le~Cun.
\newblock Learning process in an asymmetric threshold network.
\newblock In \emph{Disordered systems and biological organization}, pages
  233--240. Springer, 1986.

\bibitem[LeCun et~al.(1998{\natexlab{a}})LeCun, Bottou, Bengio, and
  Haffner]{lecun1998gradient}
Yann LeCun, L{\'e}on Bottou, Yoshua Bengio, and Patrick Haffner.
\newblock Gradient-based learning applied to document recognition.
\newblock \emph{Proceedings of the IEEE}, 86\penalty0 (11):\penalty0
  2278--2324, 1998{\natexlab{a}}.

\bibitem[LeCun et~al.(1998{\natexlab{b}})LeCun, Bottou, Orr, and
  M{\"u}ller]{lecun1998efficient}
Yann LeCun, L{\'e}on Bottou, Genevieve~B Orr, and Klaus-Robert M{\"u}ller.
\newblock Efficient backprop.
\newblock In \emph{Neural networks: Tricks of the trade}, pages 9--50.
  Springer, 1998{\natexlab{b}}.

\bibitem[Lee et~al.(2014)Lee, Xie, Gallagher, Zhang, and
  Tu]{lee_deeply-supervised_2014}
Chen-Yu Lee, Saining Xie, Patrick Gallagher, Zhengyou Zhang, and Zhuowen Tu.
\newblock {Deeply-Supervised Nets}.
\newblock \emph{{arXiv}:1409.5185 [cs, stat]}, 2014.

\bibitem[Lee et~al.(2015)Lee, Zhang, Fischer, and Bengio]{lee2015targetprop}
Dong-Hyun Lee, Saizheng Zhang, Asja Fischer, and Yoshua Bengio.
\newblock Difference target propagation.
\newblock In \emph{Proceedings of the 2015th European Conference on Machine
  Learning and Knowledge Discovery in Databases - Volume Part I}, ECMLPKDD'15,
  pages 498--515, Switzerland, 2015. Springer.
\newblock ISBN 978-3-319-23527-1.
\newblock \doi{10.1007/978-3-319-23528-8_31}.
\newblock URL \url{https://doi.org/10.1007/978-3-319-23528-8_31}.

\bibitem[Liao et~al.(2016)Liao, Leibo, and Poggio]{liao2016important}
Qianli Liao, Joel~Z Leibo, and Tomaso~A Poggio.
\newblock How important is weight symmetry in backpropagation?
\newblock In \emph{AAAI}, pages 1837--1844, 2016.

\bibitem[Lillicrap et~al.(2016)Lillicrap, Cownden, Tweed, and
  Akerman]{lillicrap2016random}
Timothy~P Lillicrap, Daniel Cownden, Douglas~B Tweed, and Colin~J Akerman.
\newblock Random synaptic feedback weights support error backpropagation for
  deep learning.
\newblock \emph{Nature communications}, 7:\penalty0 13276, 2016.

\bibitem[Mishkin and Matas(2015)]{MishkinM15}
Dmytro Mishkin and Jiri Matas.
\newblock All you need is a good init.
\newblock \emph{CoRR}, abs/1511.06422, 2015.
\newblock URL \url{http://arxiv.org/abs/1511.06422}.

\bibitem[Movellan(1991)]{movellan1991contrastive}
Javier~R Movellan.
\newblock Contrastive hebbian learning in the continuous hopfield model.
\newblock In \emph{Connectionist Models}, pages 10--17. Elsevier, 1991.

\bibitem[N{\o}kland(2016)]{nokland2016direct}
Arild N{\o}kland.
\newblock Direct feedback alignment provides learning in deep neural networks.
\newblock In \emph{Advances in Neural Information Processing Systems}, pages
  1037--1045, 2016.

\bibitem[Olshausen and Field(1997)]{olshausen1997sparse}
Bruno~A Olshausen and David~J Field.
\newblock Sparse coding with an overcomplete basis set: A strategy employed by
  v1?
\newblock \emph{Vision research}, 37\penalty0 (23):\penalty0 3311--3325, 1997.

\bibitem[O'Reilly(1996)]{o1996biologically}
Randall~C O'Reilly.
\newblock Biologically plausible error-driven learning using local activation
  differences: The generalized recirculation algorithm.
\newblock \emph{Neural computation}, 8\penalty0 (5):\penalty0 895--938, 1996.

\bibitem[O'reilly(2001)]{o2001generalization}
Randall~C O'reilly.
\newblock Generalization in interactive networks: The benefits of inhibitory
  competition and hebbian learning.
\newblock \emph{Neural computation}, 13\penalty0 (6):\penalty0 1199--1241,
  2001.

\bibitem[Ororbia~II et~al.(2015{\natexlab{a}})Ororbia~II, Giles, and
  Reitter]{ororbia2015online}
Alexander~G. Ororbia~II, C.~Lee Giles, and David Reitter.
\newblock Online semi-supervised learning with deep hybrid boltzmann machines
  and denoising autoencoders.
\newblock \emph{arXiv preprint arXiv:1511.06964}, 2015{\natexlab{a}}.

\bibitem[Ororbia~II et~al.(2015{\natexlab{b}})Ororbia~II, Reitter, Wu, and
  Giles]{ororbia_deep_hybrid_2015a}
Alexander~G. Ororbia~II, David Reitter, Jian Wu, and C.~Lee Giles.
\newblock Online learning of deep hybrid architectures for semi-supervised
  categorization.
\newblock In \emph{Machine {Learning} and {Knowledge} {Discovery} in
  {Databases} ({Proceedings}, {ECML} {PKDD} 2015)}, volume 9284 of
  \emph{Lecture {Notes} in {Computer} {Science}}, pages 516--532. Springer,
  Porto, Portugal, 2015{\natexlab{b}}.

\bibitem[Ororbia~II et~al.(2017{\natexlab{a}})Ororbia~II, Haffner, Reitter, and
  Giles]{ororbia2017learning}
Alexander~G. Ororbia~II, Patrick Haffner, David Reitter, and C~Lee Giles.
\newblock Learning to adapt by minimizing discrepancy.
\newblock \emph{arXiv preprint arXiv:1711.11542}, 2017{\natexlab{a}}.

\bibitem[Ororbia~II et~al.(2017{\natexlab{b}})Ororbia~II, Kifer, and
  Giles]{ororbia2017unifying}
Alexander~G Ororbia~II, Daniel Kifer, and C~Lee Giles.
\newblock Unifying adversarial training algorithms with data gradient
  regularization.
\newblock \emph{Neural computation}, 29\penalty0 (4):\penalty0 867--887,
  2017{\natexlab{b}}.

\bibitem[Panichello et~al.(2013)Panichello, Cheung, and
  Bar]{panichello2013predictive}
Matthew Panichello, Olivia Cheung, and Moshe Bar.
\newblock Predictive feedback and conscious visual experience.
\newblock \emph{Frontiers in Psychology}, 3:\penalty0 620, 2013.
\newblock ISSN 1664-1078.
\newblock \doi{10.3389/fpsyg.2012.00620}.
\newblock URL
  \url{https://www.frontiersin.org/article/10.3389/fpsyg.2012.00620}.

\bibitem[Pascanu et~al.(2013)Pascanu, Mikolov, and
  Bengio]{pascanu2013difficulty}
Razvan Pascanu, Tomas Mikolov, and Yoshua Bengio.
\newblock On the difficulty of training recurrent neural networks.
\newblock In \emph{International Conference on Machine Learning}, pages
  1310--1318, 2013.

\bibitem[Rao and Ballard(1999)]{rao1999predictive}
Rajesh~PN Rao and Dana~H Ballard.
\newblock Predictive coding in the visual cortex: a functional interpretation
  of some extra-classical receptive-field effects.
\newblock \emph{Nature neuroscience}, 2\penalty0 (1), 1999.

\bibitem[Romero et~al.(2014)Romero, Ballas, Kahou, Chassang, Gatta, and
  Bengio]{romero2014fitnets}
Adriana Romero, Nicolas Ballas, Samira~Ebrahimi Kahou, Antoine Chassang, Carlo
  Gatta, and Yoshua Bengio.
\newblock Fitnets: Hints for thin deep nets.
\newblock \emph{arXiv preprint arXiv:1412.6550}, 2014.

\bibitem[Rumelhart et~al.(1988)Rumelhart, Hinton, and
  Williams]{rumelhart1988backprop}
David~E. Rumelhart, Geoffrey~E. Hinton, and Ronald~J. Williams.
\newblock Neurocomputing: Foundations of research.
\newblock chapter Learning Representations by Back-propagating Errors, pages
  696--699. MIT Press, Cambridge, MA, USA, 1988.
\newblock ISBN 0-262-01097-6.
\newblock URL \url{http://dl.acm.org/citation.cfm?id=65669.104451}.

\bibitem[Scellier and Bengio(2017)]{scellier2017equilibrium}
Benjamin Scellier and Yoshua Bengio.
\newblock Equilibrium propagation: Bridging the gap between energy-based models
  and backpropagation.
\newblock \emph{Frontiers in computational neuroscience}, 11:\penalty0 24,
  2017.

\bibitem[Srivastava et~al.(2014)Srivastava, Hinton, Krizhevsky, Sutskever, and
  Salakhutdinov]{srivastava2014dropout}
Nitish Srivastava, Geoffrey Hinton, Alex Krizhevsky, Ilya Sutskever, and Ruslan
  Salakhutdinov.
\newblock Dropout: A simple way to prevent neural networks from overfitting.
\newblock \emph{The Journal of Machine Learning Research}, 15\penalty0
  (1):\penalty0 1929--1958, 2014.

\bibitem[Srivastava et~al.(2013)Srivastava, Masci, Kazerounian, Gomez, and
  Schmidhuber]{srivastava2013compete}
Rupesh~K Srivastava, Jonathan Masci, Sohrob Kazerounian, Faustino Gomez, and
  J\"{u}rgen Schmidhuber.
\newblock Compete to compute.
\newblock In C.~J.~C. Burges, L.~Bottou, M.~Welling, Z.~Ghahramani, and K.~Q.
  Weinberger, editors, \emph{Advances in Neural Information Processing Systems
  26}, pages 2310--2318. Curran Associates, Inc., 2013.
\newblock URL \url{http://papers.nips.cc/paper/5059-compete-to-compute.pdf}.

\bibitem[Stefanis(1969)]{stefanis1969interneuronal}
Costas Stefanis.
\newblock Interneuronal mechanisms in the cortex.
\newblock In \emph{UCLA forum in medical sciences}, volume~11, page 497, 1969.

\bibitem[Sussillo(2014)]{Sussillo14}
David Sussillo.
\newblock Random walks: Training very deep nonlinear feed-forward networks with
  smart initialization.
\newblock \emph{CoRR}, abs/1412.6558, 2014.
\newblock URL \url{http://arxiv.org/abs/1412.6558}.

\bibitem[Szegedy et~al.(2013)Szegedy, Zaremba, Sutskever, Bruna, Erhan,
  Goodfellow, and Fergus]{szegedy2013intriguing}
Christian Szegedy, Wojciech Zaremba, Ilya Sutskever, Joan Bruna, Dumitru Erhan,
  Ian Goodfellow, and Rob Fergus.
\newblock Intriguing properties of neural networks.
\newblock \emph{arXiv preprint arXiv:1312.6199}, 2013.

\bibitem[Van Der~Maaten(2013)]{van2013barnes}
Laurens Van Der~Maaten.
\newblock Barnes-hut-sne.
\newblock \emph{arXiv preprint arXiv:1301.3342}, 2013.

\bibitem[Vincent et~al.(2008)Vincent, Larochelle, Bengio, and
  Manzagol]{vincent2008extracting}
Pascal Vincent, Hugo Larochelle, Yoshua Bengio, and Pierre-Antoine Manzagol.
\newblock Extracting and composing robust features with denoising autoencoders.
\newblock In \emph{Proceedings of the 25th international conference on Machine
  learning}, pages 1096--1103. ACM, 2008.

\bibitem[Werbos(1988)]{werbos1988generalization}
Paul~J Werbos.
\newblock Generalization of backpropagation with application to a recurrent gas
  market model.
\newblock \emph{Neural networks}, 1\penalty0 (4):\penalty0 339--356, 1988.

\bibitem[Williams(1992)]{Williams1992}
Ronald~J. Williams.
\newblock Simple statistical gradient-following algorithms for connectionist
  reinforcement learning.
\newblock \emph{Machine Learning}, 8\penalty0 (3):\penalty0 229--256, May 1992.
\newblock ISSN 1573-0565.
\newblock \doi{10.1007/BF00992696}.
\newblock URL \url{https://doi.org/10.1007/BF00992696}.

\bibitem[Xiao et~al.(2017)Xiao, Rasul, and Vollgraf]{xiao2017fashion}
Han Xiao, Kashif Rasul, and Roland Vollgraf.
\newblock Fashion-mnist: a novel image dataset for benchmarking machine
  learning algorithms.
\newblock \emph{arXiv preprint arXiv:1708.07747}, 2017.

\bibitem[Xie and Seung(2003)]{xie2003equivalence}
Xiaohui Xie and H~Sebastian Seung.
\newblock Equivalence of backpropagation and contrastive hebbian learning in a
  layered network.
\newblock \emph{Neural computation}, 15\penalty0 (2):\penalty0 441--454, 2003.

\end{thebibliography}

\end{document}